\documentclass[10pt,twocolumn,letterpaper]{article}

\usepackage{iccv}
\usepackage{times}
\usepackage{epsfig}
\usepackage{graphicx}

\usepackage{amsmath,amsfonts,bm}

\def\eqref#1{equation~\ref{#1}}

\def\1{\bm{1}}

\newcommand{\valid}{\mathcal{D_{\mathrm{valid}}}}

\def\vz{{\bm{z}}}

\DeclareMathAlphabet{\mathsfit}{\encodingdefault}{\sfdefault}{m}{sl}
\SetMathAlphabet{\mathsfit}{bold}{\encodingdefault}{\sfdefault}{bx}{n}

\def\sO{{\mathbb{O}}}

\newcommand{\pder}[2]{\frac{\partial #1}{\partial #2}}

\newcommand{\Ls}{\mathcal{L}}

\newcommand{\softmax}{\mathrm{softmax}}

\DeclareMathOperator*{\argmin}{arg\,min}

\usepackage{amsmath}
\usepackage{amssymb}
\usepackage{multirow}
\usepackage{subcaption}
\usepackage{booktabs} 
\usepackage{comment}
\usepackage{color}
\newtheorem{theorem}{Theorem}[section]

\newtheorem{proof}{Proof}[section]

\usepackage[pagebackref=true,breaklinks=true,letterpaper=true,colorlinks,bookmarks=false]{hyperref}

\iccvfinalcopy 

\def\iccvPaperID{15} 

\ificcvfinal\fi

\begin{document}

\title{Single-DARTS: Towards Stable Architecture Search}

\author{Pengfei Hou\thanks{Equal Contribution}\\

{\tt\small houpengfei2020@126.com}
\and
Ying Jin$^*$\\
Tsinghua University\\
{\tt\small sherryying003@gmail.com}

\and
Yukang Chen\\
The Chinese University of Hong Kong\\
{\tt\small yukangchen@cse.cuhk.edu.hk}

}

\maketitle
\ificcvfinal\thispagestyle{empty}\fi

\begin{abstract}
Differentiable architecture search (DARTS) marks a milestone in Neural Architecture Search (NAS), boasting simplicity and small search costs. However, DARTS still suffers from frequent performance collapse, which happens when some operations, such as skip connections, zeroes and poolings, dominate the architecture. In this paper, we are the first to point out that the phenomenon is attributed to \textbf{bi-level} optimization. 
 
We propose \textbf{Single-DARTS} which merely uses single-level optimization, updating network weights and architecture parameters simultaneously with the same data batch. Even single-level optimization has been previously attempted, no literature provides a systematic explanation on this essential point. 
Replacing the bi-level optimization, Single-DARTS obviously alleviates performance collapse as well as enhances the stability of architecture search. 
Experiment results show that Single-DARTS achieves state-of-the-art performance on mainstream search spaces. For instance, on NAS-Benchmark-201, the searched architectures are nearly optimal ones. We also validate that the single-level optimization framework is much more stable than the bi-level one. We hope that this simple yet effective method will give some insights on differential architecture search. The code is available at https://github.com/PencilAndBike/Single-DARTS.git.
\end{abstract}

\section{Introduction}
Neural architecture search (NAS) has helped to find more excellent architectures than manual design. Generally, NAS is formulated as a bi-level optimization problem\cite{baker2016designing}:
\begin{equation}\label{equ:nas1}
\begin{split}
    \alpha^* &= \argmin_{\alpha \in \bm{A}} \mathcal{L}_{val}(\alpha, w_\alpha^*) \\
    s.t.\ w_\alpha^* &= \argmin_{w_{\alpha}} \mathcal{L}_{train}(\alpha, w)\\
\end{split}
\end{equation}

\begin{figure}[t]
  \centering
  \includegraphics[width=0.8\linewidth]{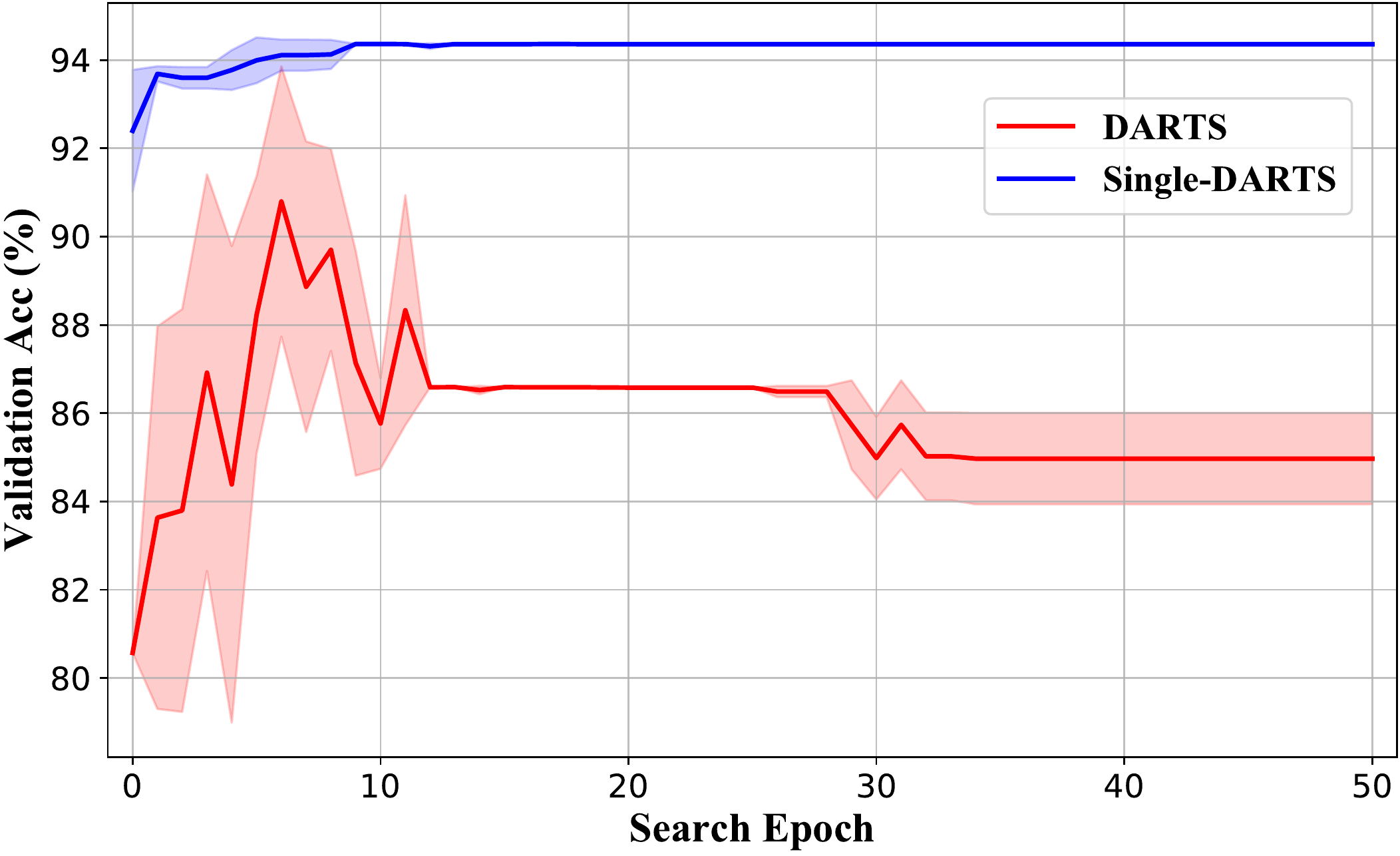}
  \caption{Search process comparison between the original DARTS and Single-DARTS.}
  \label{fig:intro1}
\end{figure}

\noindent
where $\alpha$ denotes architecture, $\bm{A}$ denotes architecture search space, $w_\alpha$ denotes the network weights with the architecture parameter $\alpha$, $\mathcal{L}_{train}$ and $\mathcal{L}_{val}$ denote the optimization loss on the training and validation dataset. Due to the inner optimization on $\mathcal{L}_{train} $ where each architecture should be well-trained respectively, it costs huge computational resources to search the optimal architecture. To avoid training each architecture from scratch, weight-sharing methods \cite{liu2018darts} are proposed, which constructs a super network where all architectures share the same weights. DARTS relaxes the search space to be continuous and approximates $w^*_\alpha$ by adapting $w$ with only a single training step, instead of solving the inner optimization \ref{equ:nas1} completely. The approximation scheme is:\\
	$\nabla_\alpha{ \mathcal{L}_{val}(\alpha, w^*_{\alpha})}
	\approx \nabla_\alpha  \mathcal{L}_{val}(\alpha, w-\xi\nabla_w{ \mathcal{L}_{train}(\alpha, w)})$.
\\
 It saves computations and finds competitive networks. 
 
However, many papers \cite{xu2019pc, liu2018progressive, zela2019understanding, li2019stacnas, chu2019fair, dong2020bench} have reported the performance collapse in  DARTS, where it easily converges to non-learnable operations, such as skip-connection, pooling, and zero, which devastates the accuracy of the searched model. As shown in Figure \ref{fig:intro1}, at the early stage of searching, DARTS performs very unsteadily, and after one point when it converges to non-learnable operation, its performance drops. They also propose some effective mechanisms or regularization to improve DARTS. 

In this paper, we dig deeper into this situation. First, we propose that the performance collapse in DARTS is \textbf{irreversible}. 
Moreover, we delve into the irreversible performance collapse in DARTS from the perspective of \textbf{optimization}, where little attention has been paid to in previous works. In the bi-level optimization in DARTS, network weights and architecture parameters are updated alternatively with different batches of data (train and validation data). We derive that in such framework, learnable operations act as adding noise on input and fitting input worse than non-learnable operations like skip-connection and pooling. As a result, the probabilities of non-learnable operations increase faster than learnable ones, which enables non-learnable operations to dominate the search phase.

To this end, we propose \textbf{Single-DARTS}, a simple yet effective single-level optimization framework, which updates network weights and architecture parameters simultaneously with the identical batch of data. Through gradient analysis, we present that with such an optimization mechanism, Single-DARTS can obviously alleviate the performance collapse. Experiment results prove that Single-DARTS shows strong performance both in accuracy and stability. In Figure \ref{fig:intro1}, Single-DARTS performs much more steadily than DARTS as well as reaches higher accuracy. On NAS-Benchmark-201 \cite{dong2020bench}, Single-DARTS searches nearly the optimal architecture with a variance of 0, showing excellent stability. 


In general, our contributions are as follows:
\begin{itemize}

\item We analyze the irreversible performance collapse in DARTS from the perspective of optimization. We provide theoretical insights by gradient analysis, by which we derive that it is the bi-level optimization framework in DARTS that causes severe and irreversible performance collapse. 

\item We propose Single-DARTS with a single-level optimization framework where the network weights and architecture parameters are updated simultaneously. 

\item Our simple yet effective method, Single-DARTS, reaches the state-of-the-art performance in mainstream search spaces. It also show strong stability.
\end{itemize}

\section{Related Works}
\textbf{Neural Architecture Search.} Neural architecture search (NAS) is an automatic method to design neural architecture instead of human design. Early NAS methods adopt reinforcement learning (RL) or evolutionary strategy \cite{zoph2016neural, baker2016designing, bello2017neural, real2017large, real2019regularized, zoph2018learning} to search among thousands of individually trained networks, which costs huge computation sources. Recent works focus on efficient weight-sharing methods, which falls into two categories: one-shot approaches \cite{brock2017smash, bender2018understanding, akimoto2019adaptive, cai2019once, guo2019single, stamoulis2019single, pham2018efficient} and gradient-based approaches \cite{shin2018differentiable, liu2018darts, chen2019progressive, cai2018proxylessnas, hundt2019sharpdarts, chu2019fair, xu2019pc, liang2019darts}, achieve state-of-the-art results on a series of tasks \cite{chen2019detnas, ghiasi2019fpn, liu2019auto, xu2020autosegnet, fu2020autogan, nascimento2020finding} in various search spaces. They construct a super network/graph which shares weights with all sub-network/graphs. The former commonly does heuristic search and evaluates sampled architectures in the super network to get the optimal architecture. The latter relaxes the search space to be continuous and introduces differential and learnable architecture parameters. In this paper, we mainly discuss gradient-based approaches.\\

\textbf{Differentiable Architecture Search.} Gradient-based approaches are commonly formulated as an approximation of the bi-level optimization which updates network weights and updates architecture parameters in the training and validation dataset alternatively\cite{liu2018darts}. It has searched competitive architectures, however, many papers pointed out DARTS-based methods don't work stably, suffering from severe performance collapse. NAS-Benchmark-201 \cite{dong2020bench} points out that DARTS performs badly in this search space.  Its performance drops quickly during the search procedure. Towards this serious problem, researchers try to give explanations and solve it. \cite{bender2018understanding} shows the relationship between DARTS searched architectures' performance and the domain eigenvalue of $\nabla_{\alpha}^2 \mathcal{L}_{valid}$. FairDARTS \cite{chu2019fair} shows that DARTS is easy to converge to skip-connect operation due to its unfair advantage in exclusive competition. It gives each operation an independent weight to replace the exclusive competition and introduces zero-one loss to decrease the unfair advantage. However, in our point bi-level optimization is the key reason for the collapse and single-level optimization could also get satisfied result under exclusive competition. Previous works\cite{liu2018progressive, xu2019pc, hong2020dropnas, li2019stacnas} show that DARTS favors non-parametric operations. DropNAS \cite{hong2020dropnas} observes the co-adaption problem and Matthew effect that light operations are trained maturely earlier. Different from these works, we analyze DARTS from the perspective of optimization.

\textbf{Single-level Optimization.} Although the original paper \cite{liu2018darts} shows that single-level optimization performs worse than bi-level optimization and they indicate that single-level would cause overfitting, but recent works show the opposite results \cite{li2019stacnas, bi2020gold, hong2020dropnas}. They adopt single/one-level optimization to replace bi-level optimization in their methods. Combining single-level optimization with their proposed methods, they can find more accurate network architectures. 
StacNAS\cite{li2019stacnas} advocates that the over-fitting phenomenon is caused by the network's depth gap between the search phase and the evaluation phase. 
GoldNAS\cite{bi2020gold} indicates that super-network parameters are trained much more effectively than architecture parameters and adds data augmentation in the search phase. DropNAS
\cite{hong2020dropnas} combines operations dropout with single-level optimization to solve the co-adaptation problem it proposes. 
However, on one side, though based on single-level optimization, these works focus on other techniques to get satisfactory results. Our method just uses single-level optimization alone to search steadily high-performance architectures. On the other side, these works prove that single-level optimization outperforms bi-level optimization empirically, without detailed explanations or proof. On the contrary, in this paper, we give theoretical insights about bi-level optimization from the perspective of optimization, on the basis of which we propose a single-level method, Single-DARTS. 

\section{Irreversible Collapse in DARTS}
In this section, we first review the Differentiable Architecture Search (DARTS) framework. Then we propose that the performance collapse in DARTS is irreversible. 

\paragraph{Preliminary of DARTS}
Differentiable Architecture Search (DARTS) is a milestone in neural architecture search. Its search space includes stacks of several cells, where each cell is a directed acyclic graph. Each cell consists of sequential nodes where node $i$ represents the latent feature map $x^i$. The edge from node $i$ to node $j$ represents a connection, and each connection is one operation from candidate operations $\sO$: convolution, pooling, zero, skip-connect, etc. Let $o^{i, j}$ denote the operation from node $i$ to $j$. The output of each intermediate node is computed based on all of its predecessors: $x_{j} = \sum_{i<j}o^{i,j}(x_i)$
In DARTS, it transforms the problem of searching the best architecture to searching the best operation on each connection. 
DARTS makes the search space continuous by relaxing the categorical choice of a particular operation to a weighted sum over all possible operations. Let $\alpha^{i, j}_k$ denote architecture parameter with operation $o_k$ between node $i$ and node $j$, then the output of this connection is 
\begin{equation}
\label{equ:darts_connection}
    \bar{o}_{i,j}(x_i) = \sum_{k} \frac{exp(\alpha^{i,j}_k)}{\sum_{k'}exp(\alpha^{i,j}_{k'})} o_k(x_i)
\end{equation}
where it adopts softmax function to have weights over different operations. The output of one intermediate node is computed based on all of its predecessors: $x_j = \sum_{i<j}\bar{o}_{i,j}(x_i)$

Therefore, DARTS needs to optimize the network weights $w$ and the architecture parameters $\alpha$. DARTS resolves this problem by a bi-level optimization framework, which updates network weights on the train set and architecture parameters on the validation set: 
$\nabla_\alpha{ \mathcal{L}_{val}(\alpha, w^*_{\alpha})}
\approx \nabla_\alpha  \mathcal{L}_{val}(\alpha, w-\xi\nabla_w{ \mathcal{L}_{train}(\alpha, w)})$.

\paragraph{Irreversible Performance Collapse}

Previous works have unveiled that it is easy for DARTS to converge to non-learnable operations like skip-connect, zero, pooling, etc. As a result, the performance collapse will happen. In this paper, we delve deeper into this situation.
As shown in Appendix Figure A-1, we train DARTS in the two spaces and plot the probability curves of one edge. At one point, the probability of non-learnable operation starts to surpass the learnable one. However, for this edge, the ideal choices are actually learnable operations, and non-learnable operations will bring about disastrous influence to model accuracy. Thus, the divergence between them becomes increasingly larger, enabling non-learnable operations to dominate the architecture search. Such divergence also indicates that the performance collapse is irreversible. 


\section{Methodology}
In this section, we firstly understand DARTS from the perspective of information gain. Then we theoretically analyze the bi-level optimization and explain the collapse in DARTS. At last, we present our solution, Single-DARTS. 

\subsection{Understanding DARTS}
According to Eq.\ref{equ:darts_connection}, we consider a general formulation for the connection setting in DARTS:
\begin{align*}
f(\sum_i p_i x_i), where \sum p_i = 1, 0 \leq p_i \leq 1
\end{align*}
where ${x_i}$ is a set of input vectors and its corresponding loss function about one-hot vector $y$ is:
\begin{align}
\label{equ:general_loss}    
\Ls = -y^T\ln (\softmax(f(\sum_i p_i x_i)))
\end{align}
We denote 
$\bar{x}$ as $\sum p_i x_i$.
If model fits $y$ with $x_i$ better than $x_j$, it means the cross entropy between $y$ and $f(x_i;\theta)$ is smaller than $f(x_j;\theta)$. For loss function \ref{equ:general_loss}, $p_i$ should be increased faster than $p_j$ for $x_i$ gets more information gain than $x_j$. Under gradient descend methods, the gradient of $p_i$ is smaller than $p_j$. Heuristically, \textit{$\pder{\Ls}{\bar{x}}^T (x_i)$ reflect the advantage of input $x_i$}. We give the formal description as follows:
for one hot random vector $y$, function $f(x;\theta)$ and loss $\Ls = -y^T\ln(f(\sum p_k x_k;\theta)$, if  $-y^T\ln(f(x_i;\theta)) \leq -y^T\ln(f(x_j;\theta))$, then we have $\pder{\Ls}{p_i} \leq \pder{\Ls}{p_j}$.
Limited to mathematical tools, we consider a simplified situation where $f(x;\theta) = x$ to give some insight and we conduct empirical experiments to verify the following theorems in deep neural networks. 
\begin{theorem}
For input vectors \{$x_k$\}, $\Ls=-y^T \ln (\softmax(\sum_k p_k x_k))$ where y is one-hot vector, $\sum_k p_k=1$ and $0 \leq p_k \leq 1$, $\bar{x}$ as $\sum_k p_k x_k$, if KL-Divergence$(p(\bar{x})||p(x_k))\approx 0$ $\forall$ $x_k$ and
\begin{equation}
-y^T\ln (\softmax(x_i)) \leq -y^T \ln (\softmax(x_j))
\end{equation}
then
\begin{equation}
\pder{\Ls}{p_i} \lessapprox \pder{\Ls}{p_j}
\end{equation}
\label{theo:understand_darts}
\end{theorem}

\begin{proof}
We have $\pder{\Ls}{p_i} = (\softmax(x_i) - y)^T x_i$, let $t$ denote $\softmax(\bar{x})$. Suppose $y^0 = 1$ and $y^l = 0$ where $l > 0$, we have 
\begin{gather*}
    -y^T \ln \softmax(x_i) = -\ln \frac{\exp{x_i^0}}{\sum_l \exp{x_i^l}}
\end{gather*}
We have
\begin{gather*}
    -\ln \exp{x_i^0} + \ln \sum_l \exp{x_i^l} \leq -\ln \exp{x_j^0} + \ln \sum_l \exp{x_j^l} \\
    x_i^0 - x_j^0 \geq \ln \frac{\sum_l \exp{x_i^l}}{\sum_l \exp{x_j^l}} 
\end{gather*}
And 
\begin{gather*}
    \pder{\Ls}{p_i} = (\softmax(\bar{x})-y)^Tx_i 
    = \sum_l t^l x_i^l - x_i^0 \\
    \pder{\Ls}{p_i} - \pder{\Ls}{p_j} = (\sum_l t^l x_i^l -x_i^0 ) - (\sum_l t^l x_j^l - x_j^0) \\
    = \sum_l t^l (x_i^l - x_j^l) - (x_i^0 - x_j^0) 
\end{gather*}
To prove $\pder{\Ls}{p_i} \leq \pder{\Ls}{p_j}$, we just need to prove $\sum_l t^l(x_i^l-x_j^l) \leq \ln \frac{\sum_l \exp{x_i^l}}{\sum_l \exp{x_j^l}}$, then we have $\pder{\Ls}{p_i} - \pder{\Ls}{p_j} \leq \ln \frac{\sum_l \exp{x_i^l}}{\sum_l \exp{x_j^l}} - (x_i^0 - x_j^0) \leq 0$. \\
Let $q_j$ denote $\softmax(x_j)$,
according to Jensen's inequality that for convex function $h(x)$, $h(\sum_k p_k x_k) \leq \sum h_k f(x_k)$ and $-\ln(x)$ is a convex function, we have
\begin{equation*}
\begin{split}
    \ln \frac{\sum_l \exp{x_i^l}}{\sum_l \exp{x_j^l}} &= \ln \sum_l \frac{\exp{x_j^l}}{\sum_{l'}\exp{x_j^{l'}}}\exp(x_i^l-x_j^l) \\
    &\geq \sum_l q_j^l\ln(\exp(x_i^l-x_j^l)) \\
    &= \sum_l q_j^l(x_i^l - x_j^l) 
\end{split}
\end{equation*}
Since KL-Divergence$(p(\bar{x})||p(x_k))\approx 0$, thus $\softmax(\bar{x}) \approx \softmax(x_k)$, $\sum_l q_j^l(x_i^l - x_j^l) \approx \sum_l t^l (x_i^l - x_j^l)$. So $\sum_l t^l(x_i^l-x_j^l) \leq \ln \frac{\sum_l \exp{x_i^l}}{\sum_l \exp{x_j^l}}$. 
\quad
\end{proof}

Note that in Theorem \ref{theo:understand_darts} we assume KL-Divergence$(p(\bar{x})||p_k(x)) \approx 0$. In fact, with the help of Batch Normalization (BN) in neural networks, the distribution of intermediate batch normalized feature maps can be approximated as normal Gaussian distribution. Moreover, $\bar{x}$ is the weighted sum of $x_i$ so it can be approximated as normal Gaussian distribution as well. Therefore, our assumption is valid. 

\subsection{Drawback in Bi-level Framework}
Consider a more concrete formulation that ${x_i}$ are transformed features from the same input $x$, $x_i = o_i(x)$. Specifically, $o_i(x)$ could be linearly expanded as $o_i(x) = W_i x$. 
\begin{equation}\label{form:grad_pi}
\pder{\Ls}{p_i} = \mathbb{E}[\pder{\Ls}{\bar{x}}^T (W_i x); X, y]
\end{equation}
\begin{equation}
\label{form:grad_wi}
\pder{\Ls}{W_i} = \mathbb{E}[p_i \pder{\Ls}{\bar{x}} x^T; X, y]
\end{equation}
To be emphasized, Eq.\ref{form:grad_wi} only exist for learnable operations for non-learnbale operations \textbf{do not} update their weights. According to last section, if $W_i x$ bring more information gain than $W_j x$, then $\pder{\Ls}{p_i}$ should be smaller than $\pder{\Ls}{p_j}$. 
For DARTS, $p_i$ is updated on the validation dataset. On the iteration $t$ and on the validation dataset $\valid = \{X_{val}, y_{val}\}$:
\begin{equation*}
\pder{\Ls}{p_i^t} = \frac{1}{N}\sum_{j=1}^N \pder{\Ls(x^j_{val})}{\bar{x}}^T (W_i^t x_{val}^j)
\end{equation*}
$W_i^t$ is updated on the training dataset $\mathcal{D_{\mathrm{train}}} = \{X_{train}, y_{train}\}$:
\begin{equation*}
\begin{split}
\pder{\Ls}{W_i^t} &= \frac{p_i}{M}\sum_{k=1}^M \pder{\Ls(x^k_{train})}{\bar{x}} {x_{train}^k}^T \\
W_i^t &= W_i^{t-1} - \eta\pder{\Ls}{W_i^{t-1}} \\
	 &= W_i^{t-1} - \eta\frac{p_i^{t-1}}{M}\sum_{k=1}^N \pder{\Ls(x^k_{train})}{\bar{x}} {x_{train}^k}^T \\
\end{split}
\end{equation*}
As a result, the gradient of $p_i$ is:
\begin{equation*}
\begin{split}
&\pder{\Ls}{p_i^t} 
= \frac{1}{N}\sum_{j=1}^N \pder{\Ls(x^j_{val})}{\bar{x}}^T (W_i^{t-1} x_{val}^j) \\
&-  \frac{\eta p_i^{t-1}}{NM}\sum_{j=1}^N \sum_{k=1}^M (\pder{\Ls(x^j_{val})}{\bar{x}}^T \pder{\Ls(x^k_{train})}{\bar{x}}) ({x_{train}^k}^T x_{val}^j)\\
\end{split}
\label{gradmult}
\end{equation*}
For bi-level optimization $W$ and $\alpha$ are computed on the different batches, since samples are different,
$\pder{\Ls(x^j_{val})}{\bar{x}}$ is independent of $\pder{\Ls(x^k_{train})}{\bar{x}}$, and ${x_{train}^k}$ is independent of $x_{val}^j$. Thus 
\begin{equation}
\label{equ:corr}
\sum_{j=1}^N \sum_{k=1}^M (\pder{\Ls(x^j_{val})}{\bar{x}}^T \pder{\Ls(x^k_{train})}{\bar{x}}) ({x_{train}^k}^T x_{val}^j) \approx 0
\end{equation}
Furthermore, during the early stage of architecture search
\begin{align}
\label{nearly}
\pder{\Ls}{p_i^t} &\approx \frac{1}{N}\sum_{j=1}^N \pder{\Ls(x^j_{val})}{\bar{x}}^T (W_i^{t-1} x_{val}^j) \\
                 &= \frac{1}{N} \sum_{j=1}^N \pder{\Ls(x^j_{val})}{\bar{x}}^T (W_i^{t-2} x_{val}^j) \\
                 &-  \frac{\eta p_i^{t-2}}{NM}\sum_{j=1}^N \sum_{k=1}^M (\pder{\Ls(x^j_{val})}{\bar{x}}^T \pder{\Ls(x^k_{train})}{\bar{x}}) ({x_{train}^k}^T x_{val}^j) \\
                  &\approx\frac{1}{N}\sum_{j=1}^N \pder{\Ls(x^j_{val})}{\bar{x}}^T (W_i^{t-2} x_{val}^j) \\
                &\approx .... \approx\frac{1}{N}\sum_{j=1}^N \pder{\Ls(x^j_{val})}{\bar{x}}^T (W_i^0 x_{val}^j) \\
\end{align}
Let $i$ denote learnable operaions' indicator while $j$ as the non-learnable's. Learnable operations (convolution) are initialized randomly thus they are actually adding noise on input $x$. So during the early stage, they bring less information gain than non-learnable operations. \ref{nearly} does not establish when iterations are adequate, however, due to softmax as activation function, the gap between $\pder{\Ls}{p_j^t}$ and $\pder{\Ls}{p_i^t}$ could be expanded during early stage. 


\paragraph{Effect of activation function} For $\alpha_i$ and $\alpha_j$, if $\alpha_j \geq \alpha_i$, the gradient of $\alpha_{i}$ is more likely to be smaller than $\alpha_{j}$ under softmax activation and gradient descent. It means the gap between $\alpha_i$ and $\alpha_j$ would be expanded. Formally, we have
\begin{theorem}
\label{theo:softmax_grad}
For function $\Ls=g(\sum_i \frac{exp(\alpha_i)}{\sum_j exp(\alpha_j)} x_i)$, let $p_i = \frac{exp(\alpha_i)}{\sum_j exp(\alpha_j)}$, $\bar{x}=\sum_i p_i x_i$, if $\pder{\Ls}{\bar{x}}^Tx_j < \pder{\Ls}{\bar{x}}^Tx_i$ and $\alpha_{j} \geq \alpha_i$, then $\pder{\Ls}{\alpha^j} \leq \pder{\Ls}{\alpha_i}$
\end{theorem}
The proof is included in Appendix. On the early stage, $\pder{\Ls}{\bar{x}}^T x_j \leq \pder{\Ls}{\bar{x}^T}x_i$ for learnable operation $i$ and non-learnable operation $j$. 
Theorem \ref{theo:softmax_grad} indicates that in DARTS, on the early training stage,  once one non-learnable operation is superior to others, the probability of this operation will become increasingly larger, and after some iterations, it will be nearly 1, which leads to the irreversible performance collapse. Moreover, we could get the convergence speed as follows.
\begin{theorem}\label{lemma1_softmax} 
Given $n$ operations, learning rate $\eta$, $i^*=\argmin_i \pder{\Ls}{\bar{x}}^Tx_i$, we define the margin between operations as $\delta=\min_{i\neq i^*}\pder{\Ls}{\bar{x}}^T(x_{i} - x_{i^*}) \geq 0$, if $\alpha_{i^*} \geq \alpha_i$, then $\forall$ $\epsilon > 0$, it achieves $p_{i^*} > 1 - \epsilon$ under gradient descent in iterations 
\begin{equation}
 t \leq \frac{n \ln((1-\epsilon)n )}{\eta \delta}
\end{equation}
\end{theorem}
The proof is included in Appendix. It shows that fewer candidate choices, a bigger margin and larger learning rates would cause performance collapse earlier. It is consistent with the empirical results that when the learning rate is lower, DARTS will perform better on NAS-Benchmark-201 if the total epoch is fixed to 50.

\subsection{Solution: Single-DARTS}
The analysis above indicates that it is the bi-level optimization in DARTS that plays a crucial role in performance collapse. In this paper, we propose Single-DARTS, which optimizes the network weights and architecture parameters by the same batch of data. 
\begin{equation}\label{equ:single-level}
\begin{split}
    \alpha^t, w^t &-= \eta \nabla_{\alpha, w} \mathcal{L}_{train} (\alpha^{t-1}, w^{t-1})
\end{split}
\end{equation}
Under this framework, we have
\begin{align*}
\pder{\Ls}{p_i^t} 
&= \frac{1}{N}\sum_{j=1}^N \pder{\Ls(x^j_{val})}{\bar{x}}^T (W_i^{t-1} x_{val}^j) \\
&- \frac{\eta p_i^{t-1}}{N^2}\sum_{j=1}^N (\pder{\Ls(x^j_{train})}{\bar{x}}^T \pder{\Ls(x^{j}_{train})}{\bar{x}}) ({x_{train}^j}^T x_{train}^{j})\\
&< \frac{1}{N}\sum_{j=1}^N \pder{\Ls(x^j_{val})}{\bar{x}}^T (W_i^{t-1} x_{val}^j) \\
&< ... < \frac{1}{N}\sum_{j=1}^N \pder{\Ls(x^j_{val})}{\bar{x}}^T (W_i^0 x_{val}^j)\\
\end{align*}
It means for learnable operations,  $\pder{\Ls}{p_i}$ in the single-level framework is smaller than in bi-level framework at the same iteration and $\pder{\Ls}{p_i}$ is decreased as the training process goes. After several iterations, learnable operations (convolution) are really learned but not add noise on the input, Thus they bring more information gain than non-learnable operations. Therefore, $\pder{\Ls}{p_i}$ of learnable operations can possibly be smaller than non-learnable operations. Under gradient descent, $p_i$ would consequently be increased faster than non-learnable operations and finally exceed them.

Furthermore, we use sigmoid to replace softmax as activation function on $\alpha$. For sigmoid, $\pder{\Ls}{\alpha_i} = p_i(1-p_i)\pder{\Ls}{\bar{x}}^Tx_i$. $p_i(1-p_i)$ is a nonmonotone function and it could be nearly zero when $p_i$ is near 1, thus even $p_i \geq p_j$ and $\pder{\Ls}{\bar{x}}^Tx_i \leq \pder{\Ls}{\bar{x}}^Tx_j$, it's possible that $\pder{\Ls}{\alpha_i} \geq \pder{\Ls}{\alpha_j}$. Therefore, from theoretical analysis, Single-DARTS can alleviate the irreversible performance collapse in DARTS. 

\section{Experiments}
\label{section:exp}
In this section, we conduct experiments to validate our theoretical insights and evaluate the performance of our proposed method, Single-DARTS. We choose two mainstream search spaces, NAS-Benchmark-201 and DARTS space.


\subsection{Results}
\subsubsection{NAS-Benchmark-201}
\label{nas201}


In NAS-Benchmark-201, we train the supermodel for 50 epochs with batch size of 64. We set SGD optimizer with learning rate as 0.005, weight decay as $3\times10^{-4}$, momentum as 0.9 and the cosine scheduler for network weights. And we set Adam optimizer with a fixed learning rate of $3\times10^{-4}$ and a momentum of (0.5, 0.999) for architecture parameters. We run each method 5 times and report the average results. 

Table \ref{tbl:nas201} shows the experiment results on NAS-Benchmark-201. NAS-Benchmark-201 is a challenging search space, where DARTS performs poorly, showing very low accuracy consistently. On the contrary, by replacing the bi-level optimization with single-level optimization, Single-DARTS outperforms all the previous methods. Single-DARTS shows very strong stability and its results are very close to the optimal results. 
So Single-DARTS is superior to the previous method both in accuracy and stability. 

\begin{table*}[t]
\caption{Comparison of different NAS algorithms on NAS-Bench-201.}
\begin{center}
\resizebox{0.8\linewidth}{!}{
\begin{tabular}{l|l|l|l|l|l|l}
\toprule
\multicolumn{1}{c}{\multirow{2}{*}{Method}} & \multicolumn{2}{|c|}{CIFAR-10}                               & \multicolumn{2}{c|}{CIFAR-100}                              & \multicolumn{2}{c}{ImageNet-16-120}                        \\ \cline{2-7} 
\multicolumn{1}{c|}{}                        & \multicolumn{1}{c|}{validation} & \multicolumn{1}{c|}{test} & \multicolumn{1}{c|}{validation} & \multicolumn{1}{c|}{test} & \multicolumn{1}{c|}{validation} & \multicolumn{1}{c}{test} \\ \hline 
Optimal   &91.61      & 94.37      &73.49        &73.51         &46.77   &47.31  \\  \hline \hlineÅ
RSPS\cite{li2020random}                                          & 84.16$_{\pm1.69}$                      & 87.66$_{\pm1.69}$                & 59.00$_{\pm4.60}$                      & 58.33$_{\pm4.34}$                & 31.56$_{\pm3.28}$                      & 31.14$_{\pm3.88}$                \\ 
DARTS\cite{liu2018darts}                                         & 39.77$_{\pm0}$                      & 54.30$_{\pm0 }$               & 15.03$_{\pm0}$                      & 15.61$_{\pm0}$                & 16.43$_{\pm0  }$                    & 16.32$_{\pm0  }$              \\ 
GDAS\cite{dong2019searching}                                          & 90.00$_{\pm0.21 }$                     & 93.51$_{\pm0.13}$                & 71.14$_{\pm0.27 }$                     & 70.61$_{\pm0.26}$              & 41.70$_{\pm1.26}$                      & 41.84$_{\pm0.90}$                \\ 
SETN \cite{dong2019one}                                         & 82.25$_{\pm5.17}$                      & 86.19$_{\pm4.63}$                & 56.86$_{\pm7.59}$                      & 56.87$_{\pm7.77}$                & 32.54$_{\pm3.63}$                      & 31.90$_{\pm4.07}$                \\ 
ENAS \cite{pham2018efficient}                                         & 39.77$_{\pm0}$                      & 54.30$_{\pm0}$                & 15.03$_{\pm0}$                      & 15.61$_{\pm0}$                & 16.43$_{\pm0}$                      & 16.32$_{\pm0}$                \\ 
CDARTS \cite{yu2020cyclic}    &91.13$_{\pm0.44}$   &94.02$_{\pm0.31}$     &72.12$_{\pm1.23}$  &71.92 $_{\pm1.30}$    &45.09$_{\pm0.61}$  & 45.51$_{\pm0.72}$ \\ 
DARTS-  \cite{chu2020darts} &91.03$_{\pm0.44}$   &93.80$_{\pm0.40}$    &71.36$_{\pm1.51}$   &71.53$_{\pm1.51}$   &44.87$_{\pm1.46}$    & 45.12$_{\pm0.82}$ \\ \hline
Single-DARTS   &{\bf91.55$_{\pm0}$}   & {\bf 94.36$_{\pm0}$}   &{\bf 73.49$_{\pm0}$}  &{\bf 73.51$_{\pm0}$}   & {\bf 46.37$_{\pm0}$}  &{\bf 46.34$_{\pm0}$} \\  
\bottomrule 
\end{tabular}}
\end{center}
\label{tbl:nas201}
\end{table*}

\subsubsection{DARTS}

\begin{table}[h!]
  \caption{
  Comparison with SOTA architectures on ImageNet in DARTS space. $\dagger$: the average results of searched architectures (not the average results of retraining searched the best architecture) and  $*$: the best result. $\star$: scaling channels of the best such that its FLOPs bellow 600M. 
  $\diamond$: adding extra training strategy on PDARTS setting.
  }
  \label{tab:darts-imagenet}
  \centering
  \begin{tabular}{lccc}
  	\toprule
    \textbf{Architecture} &  \textbf{FLOPs/M}
 & \textbf{Params/M} & \textbf{Top-1/\%.}   \\
    \midrule
    NASNet-A \cite{zoph2018learning} & 564 & 5.3 & 74.0 \\
    AmoebaNet-C \cite{real2019regularized} &570 & 6.4 & 75.7 \\
    PDARTS \cite{chen2019progressive} & 557 & 4.9 &75.6 \\
    PC-DARTS \cite{xu2019pc} & 597 & 5.3 & 75.8  \\ 
    DARTS \cite{liu2018darts} & 574 & 4.7 & 73.3  \\
    DARTS+$\diamond$ \cite{liang2019darts} & 591 & 5.1 & 76.3  \\
    CDARTS$\dagger$ \cite{yu2020cyclic} &732 & 6.1 & 76.3  \\
    CDARTS* \cite{yu2020cyclic} &704 & 6.3 & 76.6  \\
    CDARTS$\star$ \cite{yu2020cyclic} &571 & 5.4 & 75.9 \\
    DropNAS$\diamond$ \cite{hong2020dropnas} & 597 & 5.4 & 76.6 \\
    \midrule
    \textbf{Single-DARTS}* &714 & 6.60 & {\bf 77.0}  \\
    \textbf{Single-DARTS}$\dagger$ &707 &6.55 & 76.7 \\
    \textbf{Single-DARTS}$\star$ & 599 & 5.3 & 76.0 \\
     \bottomrule
  \end{tabular}
  \label{tab:darts-imagnet}
\end{table}

\begin{table}[h!]
  \caption{
  Comparisons on CIFAR10 in DARTS space. 
  }
  \centering
  \begin{tabular}{lccc}
  	\toprule
    \textbf{Architecture} & \textbf{Params} & \textbf{Top-1.}   \\
  &  (M)  &  (\%) \\
    \midrule
    AmoebaNet-B \cite{real2019regularized} & 2.8 & 97.45  \\
    PDARTS \cite{chen2019progressive} & 3.4 & 97.50 \\
    PC-DARTS \cite{xu2019pc} & 3.6 & 97.43$_{\pm0.07}$  \\ 
    DARTS \cite{liu2018darts} & 3.3 & 97.24$_{\pm0.09}$  \\
    DARTS+ $\diamond$ \cite{liang2019darts} & 4.3 & 97.63$_{\pm0.13}$ \\
    CDARTS \cite{yu2020cyclic} & 3.8 & 97.52$_{\pm0.04}$  \\
    DropNAS \cite{hong2020dropnas} & 4.1 & 97.42$_{\pm0.14}$\\
    \midrule
    \textbf{Single-DARTS} & 3.3 & 97.54 \\
     \bottomrule
    
  \end{tabular}
  \label{tab:darts-cifar}
\end{table}

\begin{figure*}[t]
\begin{center}
\includegraphics[width=0.3\linewidth]{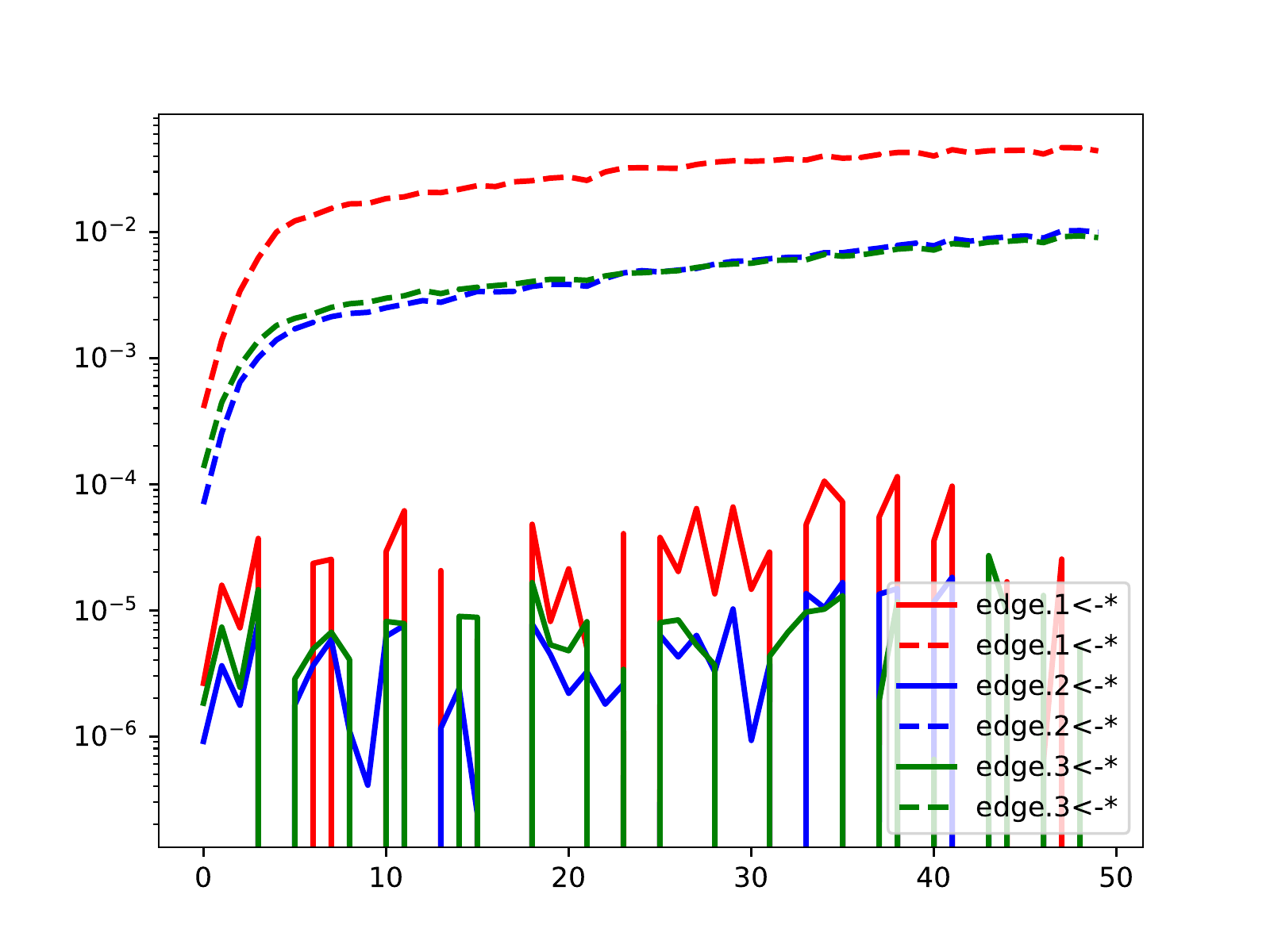}
\includegraphics[width=0.3\linewidth]{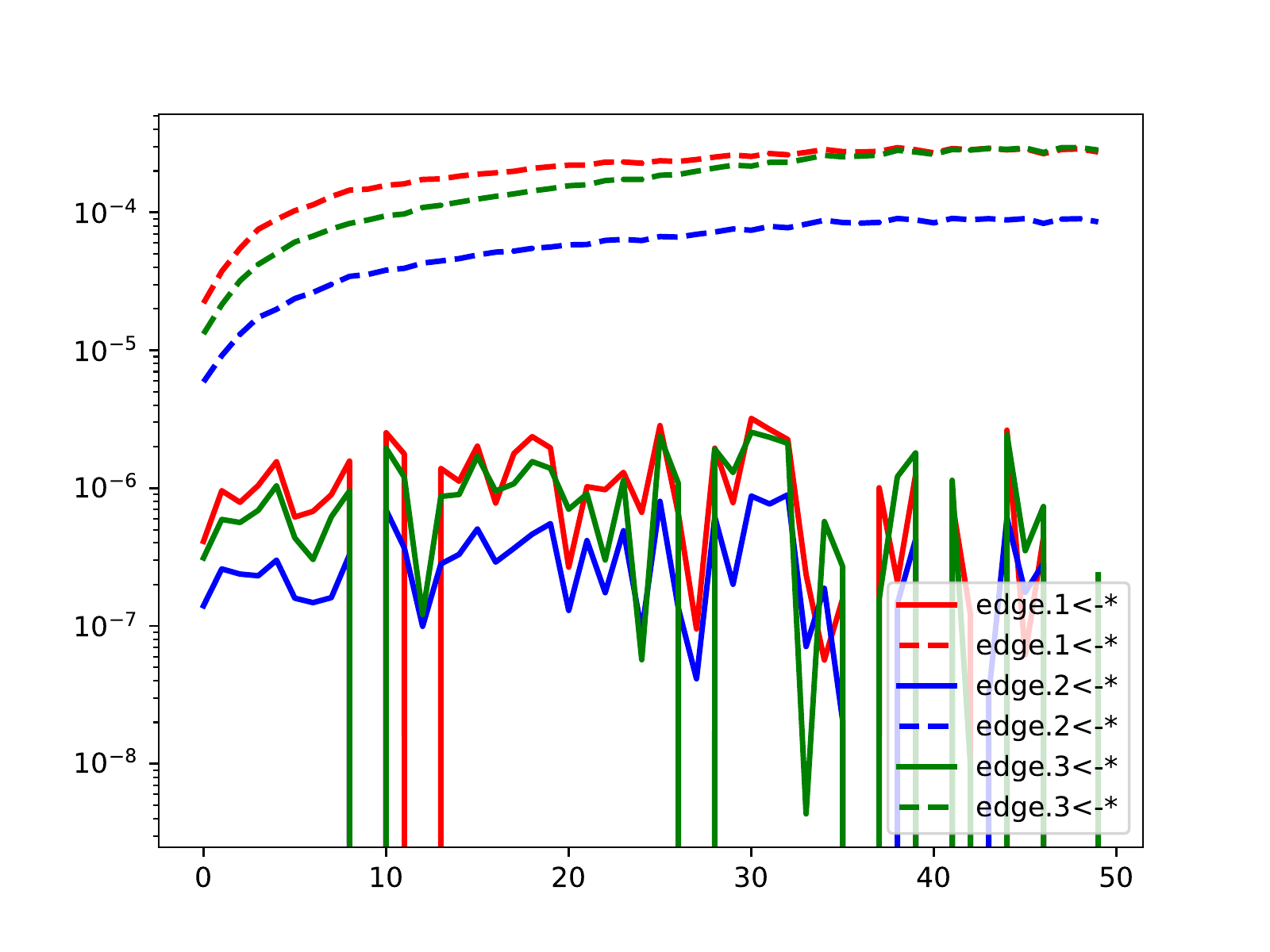}
\includegraphics[width=0.3\linewidth]{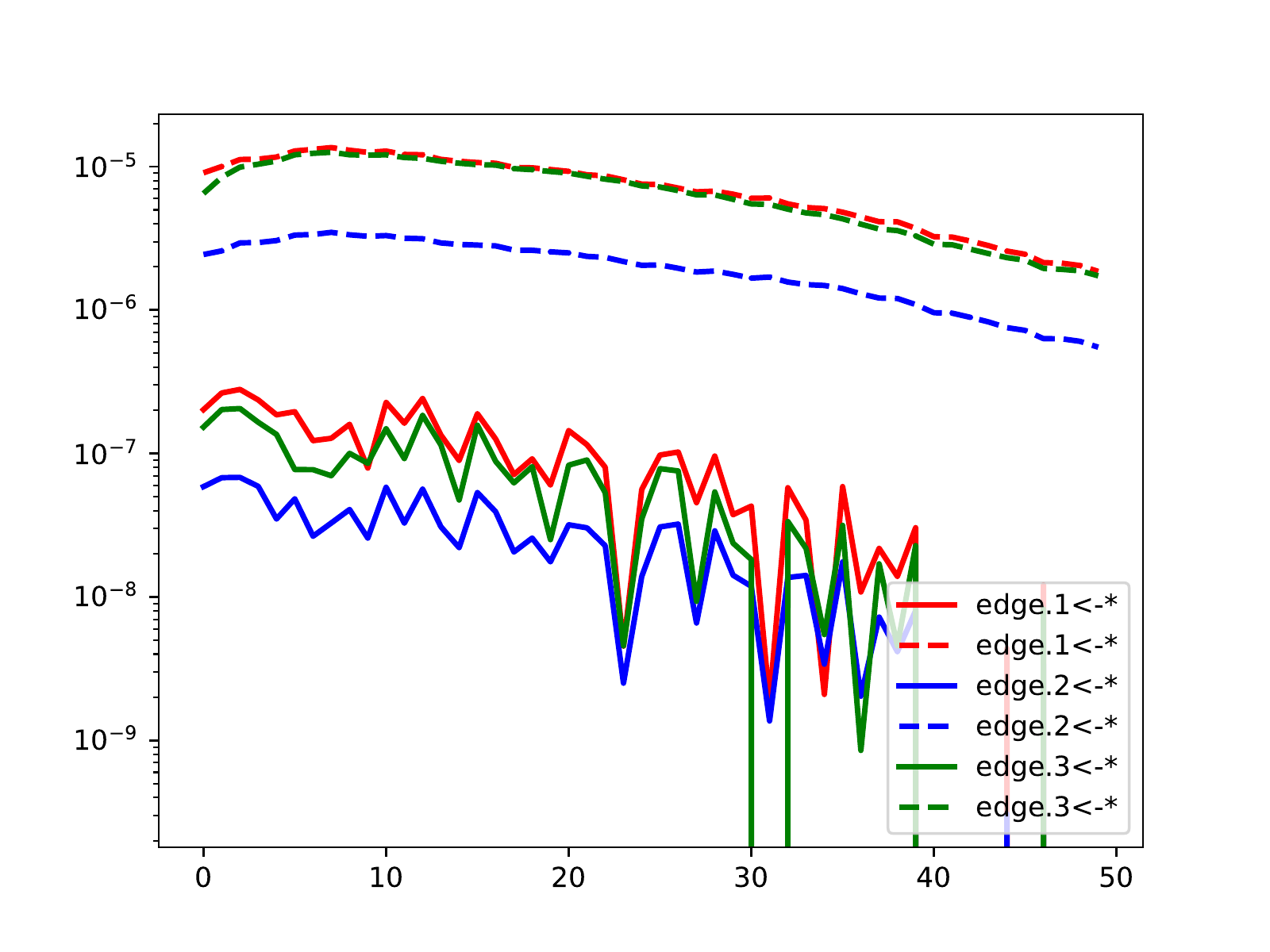}
\caption{Comparison of gradient correlation \ref{equ:corr} between bi-level and single-level optimization on CIFAR10 dataset in NAS-Benchmark-201 search space. It shows the gradient correlation on different nodes of the first, the middle, and the last cells. Solid lines represent bi-level and dashed lines represent single-level. Gradient correlation in bi-level optimization is nearly zero and far smaller than in single-level optimization.} 
\label{fig:corr1}
\end{center}
\end{figure*}
We directly do the search process on the ImageNet-1k dataset. Specifically, input images are downsampled third times by convolution layers at the beginning of the supermodel to reduce spatial resolution which follows \cite{unnas}. In the search phase, we train the supermodel for 50 epochs. We set batch size as 360, learning rate as 0.025, weight decay as $3\times10^{-4}$, and the cosine scheduler for network weights. And we use Adam with a fixed learning rate of $3\times10^{-4}$ and a momentum of (0.5, 0.999) for architecture parameters, and weight decay is set to 0 to avoid extra gradients on zero operations. 
For sigmoid, $\alpha$ is initialized as $-\ln(7)$ such that $sigmoid(\alpha)=0.125$ (from our empirical experiments setting $\alpha=0$ could be worse than using softmax, more comparisons are shown in \ref{tab: lr-act-opt-analysis} and Appendix). The search phase costs 4 GPUs for about 28 hours on NVIDIA GeForce RTX 2080ti. In the retraining phase, we adopt the training strategy as previous works \cite{hundt2019sharpdarts} to train the searched architecture from scratch, without any additional module. The whole process lasts 250 epochs, using SGD optimizer with a momentum of 0.9, a weight decay of $3\times10^{-5}$. It's to be mentioned that due to memory limit of 2080ti, we scale both batch size and learning rate by 0.75 (1024 to 768, 0.5 to 0.375). And we reproduce previous methods \cite{hundt2019sharpdarts, xu2019pc, yu2020cyclic} under this setting to ensure fairness. Additional enhancements are adopted including label smoothing and an auxiliary loss tower during training as in PDARTS. Learning rate warm-up is applied for the first 5 epochs and decayed down to zero linearly. The retrain phase costs 8 GPUs for about 3.5 days on RTX 2080ti.

As shown in Table \ref{tab:darts-imagenet}, Single-DARTS search the state-of-the-art architecture with 77.0\% top1 accuracy on ImageNet-1K. There is no constraint on FLOPs during the search process and larger models naturally have better performance. 
We report the result of previous methods in their original paper. When training the searched architectures from scratch, we follow the strategy of PDARTS, which re-trains the searched architecture for 250 epochs. In contrast, DropNAS trains for 600 epochs and DARTS+ for 800 epochs.  AutoAugment, mixup and SE module are also applied in DropNAS.

We transfer the searched architecture to CIFAR10. To keep mobile setting, we shorten the number of layers from 20 to 14 to make parameters under 3.5M. The training setting is identical to PDARTS. The architecture's initial channels are 36. The whole process lasts 600 epochs with batch size of 128, using SGD optimizer with a momentum of 0.9, a weight decay of $3\times10^{-4}$. Cutout regularization of length 16, drop-path of probability 0.3 and auxiliary towers of weight 0.4 are applied. The initial learning rate is 0.025, which is decayed to 0 following the cosine rule. As shown in Table \ref{tab:darts-cifar}, Single-DARTS search the comparable result with 97.54\% accuracy on CIFAR10. It's to be mentioned that DARTS+ trains for 2,000 epochs.

\subsection{Analysis}
\subsubsection{Gradient Correlation}
We define $\sum_{j=1}^N \sum_{k=1}^M (\pder{\Ls(x^j_{val})}{\bar{x}}^T \pder{\Ls(x^k_{train})}{\bar{x}}) ({x_{train}^k}^T x_{val}^j)$
as the gradient correlation in bi-level optimization, accordingly the one in single-level optimization is 
$\sum_{j=1}^N \sum_{k=1}^M (\pder{\Ls(x^j_{train})}{\bar{x}}^T \pder{\Ls(x^k_{train})}{\bar{x}}) ({x_{train}^k}^T x_{train}^j)$.
In our analysis, we point out that it is the gradient correlation that plays a crucial role in the performance collapse in DARTS. 
Here we visualize the gradient correlation in bi-level optimization and single-level optimization. Figure \ref{fig:corr1} compares the gradients correlation \ref{equ:corr} in DARTS and Single-DARTS. For DARTS, the gradients correlation is nearly zero. On the contrary, the gradient correlation in Single-DARTS is much bigger than the one in DARTS, which is consistent with our theoretical analysis.

\subsubsection{Operator Gradients}
\label{subsec: op-grad}

We compare $\pder{\Ls}{p_i}$ in the last cell between DARTS and Single-DARTS in Figure \ref{fig:grad_pi}. For DARTS, $\pder{\Ls}{p_i}$ of non-learnable operations are smaller than learnable operations and the gap increases as the training process goes. Thus non-learnable operations' architecture parameters would be much larger than the learnable ones. On the contrary, in Single-DARTS, $\pder{\Ls}{p_i}$ of learnable operations compete against non-learnable operations and finally learnable operations surpass non-learnable ones. More comparison results of different cells are shown in Appendix.

\begin{figure}[t]
\centering
 \begin{subfigure}[b]{0.48\linewidth}
 \centering
\includegraphics[width=\textwidth]{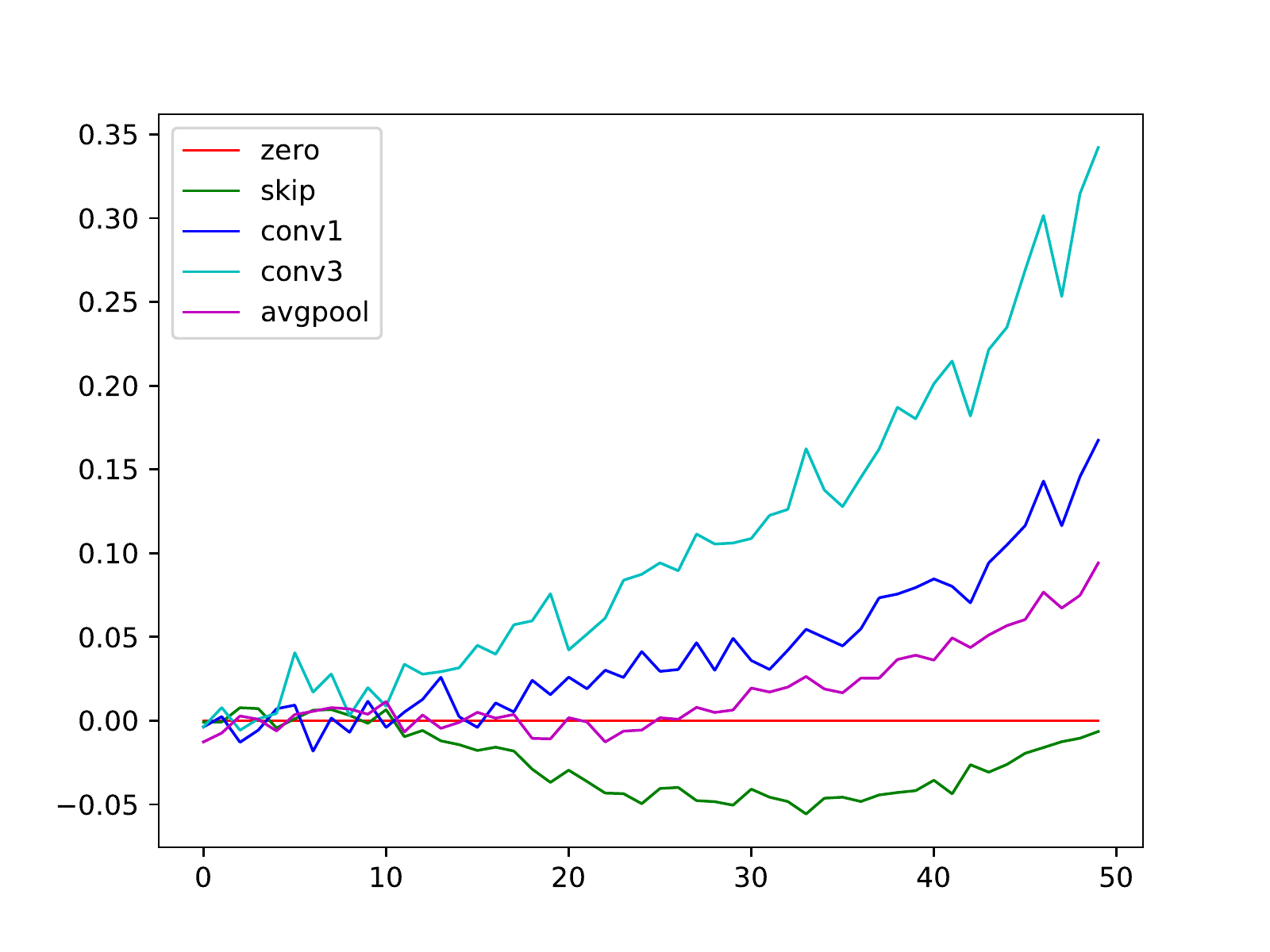}
 \caption{DARTS}
\end{subfigure}
\begin{subfigure}[b]{0.48\linewidth}
\centering
\includegraphics[width=\textwidth]{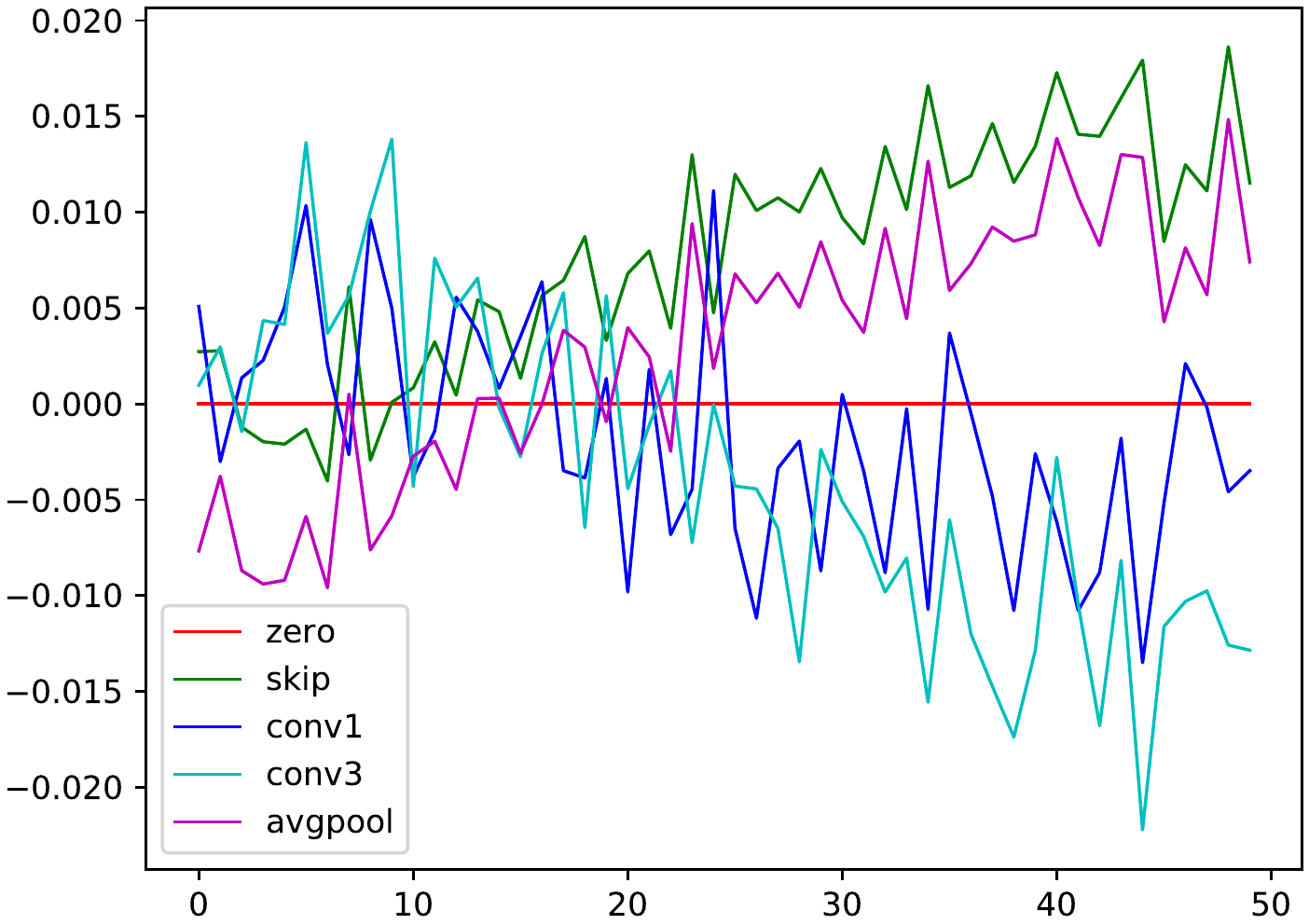}
\caption{Single-DARTS}
\end{subfigure}

\caption{Comparison of gradients of $p_i$ in the edges 3$\leftarrow$2 in the 15th cell in DARTS and Single-DARTS. 
} 
\label{fig:grad_pi}
\end{figure}

\subsection{Ablation Study} 

\paragraph{Bi-level V.S. Single-level} In this paper, Single-DARTS optimizes network weights and architecture parameters with the same batch of data. However, there are also some different optimization methods, such as using two different batch of data from the train set. Here we conduct the ablation study on it and the results are shown in Table \ref{tab:batch}. The training phase lasts 50 epochs. $w$ and $\alpha$ are optimized with the SGD optimizer. The learning rate is set to $0.005$, which is the optimal value for these methods. We run each experiment 5 times and report the average values. We can see from the results that if we sample two different batch data from the same subset, the performance collapse is alleviated to some extent. Thus, using same batch of data is optimal, which is equivalent to single-level optimization in ours. 


\begin{table}[t]
\caption{Comparison of different optimization strategies with Softmax, whether to optimize on same subset or batch. 
}
\label{tab:batch}
\begin{center}
\resizebox{\linewidth}{!}{
\begin{tabular}{c|c|l|l|l}
\toprule
 \small{Subset} & \small{Batch} & \small{CIFAR10} & \small{CIFAR100} & \small{ImgNet$_{16-120}$} \\ \hline
                 &                &61.88$_{\pm10.72}$ &25.13$_{\pm13.47}$ &17.98$_{\pm2.35}$             \\ \hline
$\checkmark$                  &                & 84.34$_{\pm0.02}$       & 54.92$_{\pm0.06}$        & 25.97$_{\pm0.49}$              \\ \hline
$\checkmark$                  & $\checkmark$                &\textbf{94.27}$_{\pm0.13}$ &\textbf{73.01}$_{\pm0.71}$ &\textbf{46.10}$_{\pm0.34 }$             \\ \bottomrule

\end{tabular}}
\end{center}
\end{table}

\paragraph{Different Learning Rates.} We evaluate the performance of bi-level (DARTS) and single-level (Single-DARTS) optimization under different learning rates. Here, we use SGD as the optimizer and change the learning rate. Table \ref{tab: lr-act-opt-analysis} shows that when the learning rate changes, the performance of DARTS varies dramatically, where the gap can be more than $50\%$. On the contrary, Single-DARTS shows much stronger stability and always achieves accuracies well above $93\%$.

\paragraph{Different Activation Functions} In Single-DARTS, we adopt a non-competitive activation function, Sigmoid, instead of the original Softmax function. Here we show the performance of DARTS and Single-DARTS under different activation functions in Table \ref{tab: lr-act-opt-analysis}. 
We can see that using non-competitive activation functions is effective, which is consistent with the empirical results in previous works and the theoretical analysis in our paper. Moreover, Single-DARTS performs much more steadily than DARTS under different activation functions, no matter competitive or not. 

\paragraph{Different Optimizers} We evaluate the performance of DARTS and Single-DARTS under different optimizers. For SGD optimizer, we set momentum as 0.9, learning rate as 0.005, and weight decay as $3\times10^{-4}$. For Adam optimizer, we use a fixed learning rate of $3\times10^{-4}$ and a momentum of (0.5, 0.999). From Table \ref{tab: lr-act-opt-analysis}, using Adam as the optimizer can admittedly push the performance of DARTS to a higher level. On the contrary, Single-DARTS performs well no matter what optimizer it uses. 

\begin{table}[t]
\caption{Accuracy under different learning rates (lr), activation functions (act), and optimizers (opt). * init $\alpha$ as 0.}
\label{tab:lr_act}
\begin{center}
\resizebox{\linewidth}{!}{
\begin{tabular}{l|l|l|l}
\toprule
Method  & lr/act/opt & CIFAR10 & CIFAR100  \\ \hline
\multirow{3}{*}{DARTS}  &0.001 &91.54$_{\pm1.26}$ &68.21$_{\pm1.15}$  \\ \cline{2-4}
  &0.005 &61.88$_{\pm10.72}$ &25.13$_{\pm13.47}$  \\ \cline{2-4}
  &0.025 &39.77$_{\pm0}$ &54.30$_{\pm0}$  \\ \hline
\multirow{3}{*}{Single-DARTS}  &0.001 &\textbf{94.36}$_{\pm0}$ &\textbf{73.51}$_{\pm0}$ \\ \cline{2-4}
Single-DARTS  &0.005 &\textbf{94.36}$_{\pm0}$ &\textbf{73.51}$_{\pm0}$ \\ \cline{2-4}
  &0.025 &\textbf{93.10}$_{\pm0}$ &\textbf{69.24}$_{\pm0}$ \\ \hline \hline
  \multirow{2}{*}{DARTS}  & Softmax &61.88$_{\pm10.72}$ &25.13$_{\pm13.47}$  \\ \cline{2-4}
  & Sigmoid &80.57$_{\pm0}$ &47.93$_{\pm0}$  \\ \hline
\multirow{3}{*}{Single-DARTS}  & Softmax & \textbf{94.27}$_{\pm0.13}$ &\textbf{73.01}$_{\pm0.71}$ \\ \cline{2-4}
  & Sigmoid & \textbf{94.36}$_{\pm0}$ &\textbf{73.51}$_{\pm0}$ \\ \cline{2-4}
  & Sigmoid*  & \textbf{93.76}$_{\pm0}$ &\textbf{70.71}$_{\pm0}$ \\ \hline \hline

  \multirow{2}{*}{DARTS}  & SGD &61.88$_{\pm10.72}$ &25.13$_{\pm13.47}$  \\ \cline{2-4}
  & Adam &80.57$_{\pm0}$ &47.93$_{\pm0}$  \\ \hline
\multirow{2}{*}{Single-DARTS} & SGD & \textbf{94.36}$_{\pm0}$ &\textbf{73.51}$_{\pm0}$ \\ \cline{2-4}
 & Adam & \textbf{94.36}$_{\pm0}$ &\textbf{73.51}$_{\pm0}$ \\\bottomrule
\end{tabular}}
\end{center}
\label{tab: lr-act-opt-analysis}
\end{table}

\section{Conclusion}
In this paper, we unveil that the performance collapse in DARTS is irreversible. Then we analyze this phenomenon from the perspective of optimization. Through gradient analysis, we state out that the bi-level framework plays a crucial role in performance collapse. On the basis of our theoretical insights, we propose a simple yet effective method, which uses single-level optimization. Ours outperforms previous methods both in accuracy and stability. 
\\

{\small
\bibliographystyle{ieee_fullname}
\bibliography{single_level}
}

\clearpage
\appendix
\section{Irreversible Trend of Non-parametric Operations}
 The probability curves of architecture parameters ($\softmax (\alpha)$) during the whole search of DARTS on CIFAR-10 in (a) NAS-201 and (b) DARTS Search space. There is an irreversible trend that non-learnable operations ($e.g.$, skip, pool) surpass learnable operations. 

\begin{figure}[h]
  \centering
  \includegraphics[width=\linewidth]{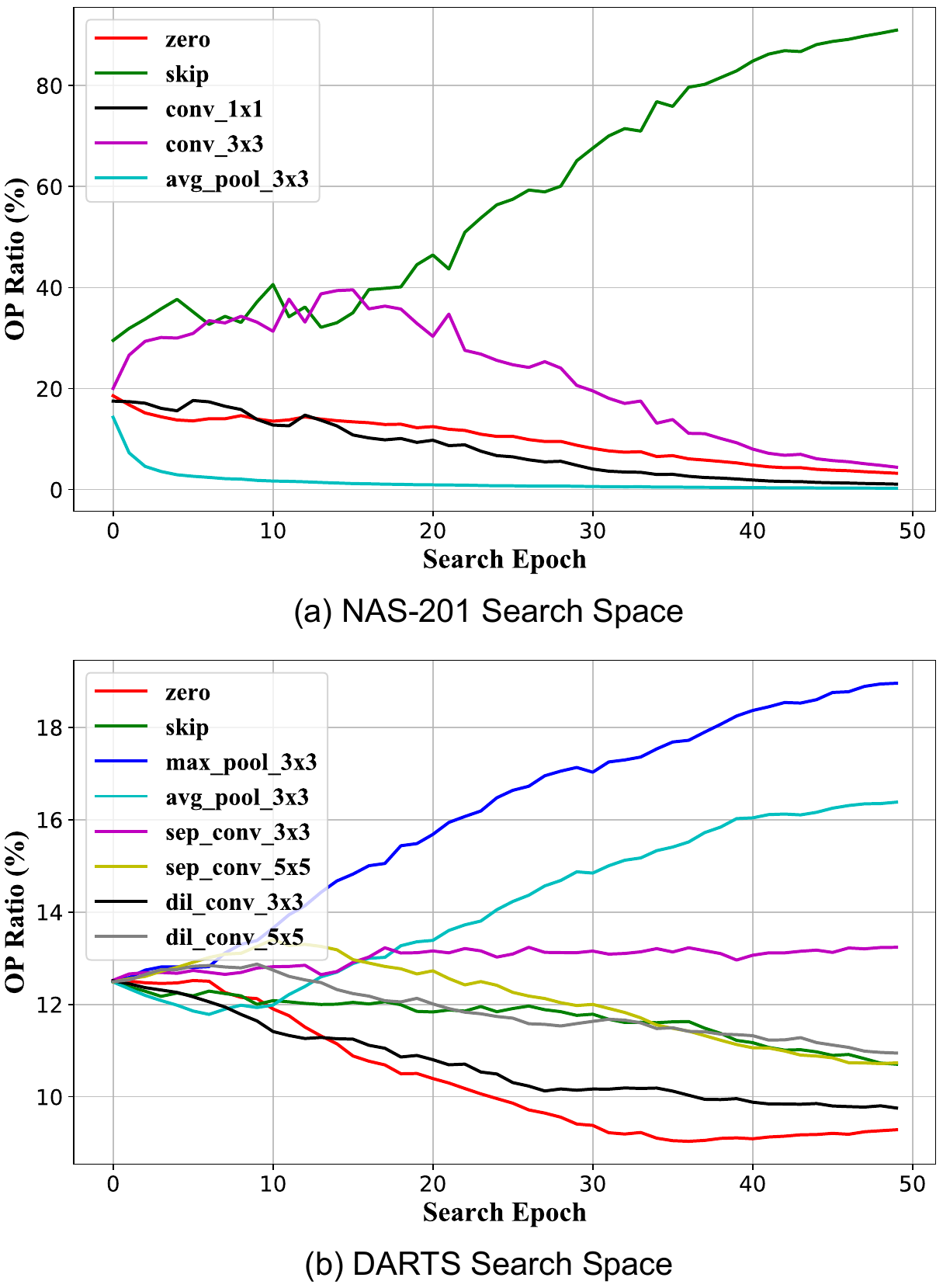}
  \caption{Irreversible trend of non-parametric operations in different space}
  \label{fig:intro}
\end{figure}

\section{Search space details}
NAS-Bench-201 \cite{dong2020bench} builds a cell-based search space, where one cell could be seen as a directed acyclic graph consisting of 4 nodes and 6 edges. Each network is stacked by 15 cells. Each edge represents an operation selected from (1) zero, (2) skip connection, (3) 1$\times$1 convolution, (4) 3$\times$3 convolution, and (5) 3$\times$3 average pooling. The search space has 15,625 neural cell candidates in total. And all the candidates are given training accuracy/valid accuracy/test accuracy on three datasets: (1) CIFAR-10: In NAS-Bench-201, it separates the train set in original CIFAR-10 into two parts, one as the train set and the other as the validation set, each set contains 25K images with 10 classes, (2) CIFAR-100: For CIFAR-100, it separates the original test set to get the validation and test set, each set has 5K images, and (3) ImageNet-16-120. It downsamples ImageNet to $16\times16$ pixels and selects 120 classes. Totally, it involves 151.7K training images, 3K validation images, and 3K test images. We set Adam as architecture parameters' optimizer and SGD as network weights'.

DARTS \cite{dong2020bench} is also a cell-based search space. Each cell contains 6 nodes and each node has to select 2 edges to connect with the previous 2 nodes. Each edge has 8 operations: $3\times3$ and $5\times5$ separable convolution, $3\times3$ and $5\times5$ dilated separable convolution, $3\times3$ max-pooling, $3\times3$ average-pooling, skip-connect (identity), and zero (none). The stacked networks have normal cells and reduction cells. It contains  $10^{18}$ candidates, which is quite large.\\

\section{Proof of theorem 4.2}
\begin{theorem}
For function $\Ls=g(\sum_i \frac{exp(\alpha_i)}{\sum_j exp(\alpha_j)} x_i)$, let $p_i = \frac{exp(\alpha_i)}{\sum_j exp(\alpha_j)}$, $\bar{x}=\sum_i p_i x_i$, if $\pder{\Ls}{\bar{x}}^Tx_j < \pder{\Ls}{\bar{x}}^Tx_i$ and $\alpha_{j} \geq \alpha_i$, then $\pder{\Ls}{\alpha^j} \leq \pder{\Ls}{\alpha_i}$
\end{theorem}
\begin{proof}
It's easily to see that $p_i \geq p_j$ if $\alpha_i \geq \alpha_j$ and $p_{i^*} \geq \frac{1}{n}$ for $\sum_i p_i = 1$.
Consider
\begin{equation*}
\begin{split}
\pder{\Ls}{\alpha_i} &= \pder{\Ls}{\bar{x}}^T(p_i(1-p_i)x_i - p_i\sum_{k\neq i}p_k x_k) \\
                    &= \pder{\Ls}{\bar{x}}^T \sum_{k\neq i} p_i p_k (x_i - x_k) \\
                    &= \pder{\Ls}{\bar{x}}^T \sum_{k} p_i p_k (x_i - x_k) \\
                    &= p_i\sum_k p_k \pder{\Ls}{\bar{x}}^T(x_i - x_k) \\
\end{split}
\end{equation*}
For $\pder{\Ls}{\bar{x}}^Tx_j \leq \pder{\Ls}{\bar{x}}^Tx_i$,
\begin{equation*}
\begin{split}
    &\pder{\Ls}{\alpha_i} - \pder{\Ls}{\alpha_j} \\
    &= p_i\sum_k p_k \pder{\Ls}{\bar{x}}^T(x_i - x_k) - p_j\sum_k p_k \pder{\Ls}{\bar{x}}^T(x_j - x_k) \\
    &= (p_i - p_j)\sum_k p_k \pder{\Ls}{\bar{x}}(x_i - x_k) + p_j(\sum_k p_k \pder{\Ls}{\bar{x}} (x_i - x_k) \\
    &-\sum_k p_k\pder{\Ls}{\bar{x}}(x_j-x_k)) \\
    &= (p_i - p_j)\sum_k p_k \pder{\Ls}{\bar{x}}(x_i - x_k) + p_j\sum_k p_k\pder{\Ls}{\bar{x}}(x_i - x_j)
\end{split}    
\end{equation*}

As a result
\begin{equation}
\begin{split}
&\frac{\partial l}{\partial\alpha_{i^*}} - \frac{\partial l}{\partial\alpha_{i}}\\
&= p_{i^*} \sum_{j} p_j\frac{\partial l}{\partial\vz} (\vz_{i^*}-\vz_j) - p_i \sum_{j} p_j\frac{\partial l}{\partial\vz} (\vz_i-\vz_j) \\
&= (p_{i^*} - p_i) \sum_{j} p_j\frac{\partial l}{\partial\vz} (\vz_{i^*}-\vz_j) + p_i (\sum_{j} p_j\frac{\partial l}{\partial\vz} (\vz_{i^*}-\vz_j) \\
&- \sum_{j} p_j\frac{\partial l}{\partial\vz} (\vz_i-\vz_j)) \\
&= (p_{i^*} - p_i) \sum_{j} p_j\frac{\partial l}{\partial\vz} (\vz_{i^*}-\vz_j) + p_i\sum_{j} p_j\frac{\partial l}{\partial\vz} (\vz_{i^*}-\vz_i) \\
&\geq (p_{i^*} - p_i) \delta + p_i \delta \\
&= p_{i^*} \delta \\
&\geq \frac{\delta}{n} \\
\end{split}
\end{equation}
Under gradient ascent, on update step $t$ we have, for any $i \neq i^*$,
\begin{equation}
\begin{split}
\alpha_{i^*}^t - \alpha_{i}^t &= \alpha_{i^*}^{t-1} - \alpha_{i}^{t-1} + \eta(\frac{dl_{t-1}}{d\alpha_{i^*}^{t-1}} - \frac{dl_{t-1}}{d\alpha_i^{t-1}}) \\
&\geq \alpha_{i^*}^{t-1} - \alpha_{i}^{t-1} + \eta p_{i^*}^t\delta_t \\
& \geq \alpha_{i^*}^{0} - \alpha_{i}^{0} + \eta \sum_t p_{i^*}^t\delta_t \\
& \geq \frac{\eta t \delta}{n}  \\
\end{split}
\end{equation}
When $t=\frac{n \ln((1-\epsilon)n )}{\eta \delta}$, we have \\
\begin{equation*}
	\alpha_{i^*} - \alpha_i \geq \ln((1-\epsilon)n)
\end{equation*}
Thus
\begin{equation}
 p_{i^*} = \frac{1}{\sum_i \exp({\alpha_i - \alpha_{i^*}})} \geq \frac{1}{\sum_i \exp(-\ln((1-\epsilon)n))} = 1 - \epsilon \\
\end{equation}
Under gradient descent, for $i^*=\argmin_i \frac{\partial l}{\partial\vz}\vz_i$, we have the same conclusion.
\end{proof}

\section{Searched results in DARTS}
We use Single-DARTS to search directly on the ImageNet-1K dataset in DARTS space. Initializing $\alpha_i$ as $-\ln(7)$ ( $\softmax(\alpha_i)=0.125$) will improve the performance. 
The training setting follows PDARTS, without any additional tricks. In addition, for Single-DARTS, using half of the dataset also gains promising results.

\begin{table}[h]
  \caption{$*$ denotes $\alpha$ is initialized as $-\ln(7)$. 'Data' means using full or half of data to search.}
  \label{table:darts-imagenet}
  \centering
  \resizebox{\linewidth}{!}{
  \begin{tabular}{lccccc}
  	\toprule
   \textbf{Activation} & \textbf{Data} & \textbf{Seed}  & \textbf{FLOPs} & \textbf{Params} & \textbf{Top-1.}   \\
  & (M)&  (M)  &  (\%) \\
    \midrule
    softmax & full & 0 & 714.72 & 6.58 & 76.28  \\
    softmax & full & 1 & 712.92 & 6.56 & 76.71  \\ 
    softmax & full & 2 & 722.25 &6.61 & 76.27  \\
    sigmoid & full & 0 & 738.21 & 6.69 & 76.12 \\
    sigmoid & full & 1 & 738.21 & 6.69 & 76.58 \\
    sigmoid & full & 2 & 721.35 & 6.60 & 76.29 \\
    sigmoid* & full & 0 & 707.89 & 6.50 & 76.96  \\
    sigmoid* & full & 1  & 721.35 &6.60 &  77.0 \\
    sigmoid* & full & 2 & 700.61 & 6.40 & 76.95 \\
    softmax & half & 0 & 712.01 & 6.55 & 76.51  \\
    softmax & half & 1 & 692.18 & 6.36 & 76.31	  \\ 
    softmax & half & 2 & 709.04 & 6.45 & 76.51  \\
    sigmoid* & half & 0 & 720.44 & 6.59 & 76.67  \\
    sigmoid* & half & 1 & 692.18 & 6.36 &  76.78 \\
    sigmoid* & half & 2 & 707.89 & 6.50 & 76.54 \\
    \bottomrule
  \end{tabular}}
\end{table}

\section{More comparison of gradients of $p_i$ in 5.2.2}
\begin{figure*}[h]
\centering
 \begin{subfigure}[b]{0.3\linewidth}
 \centering
\includegraphics[width=\textwidth]{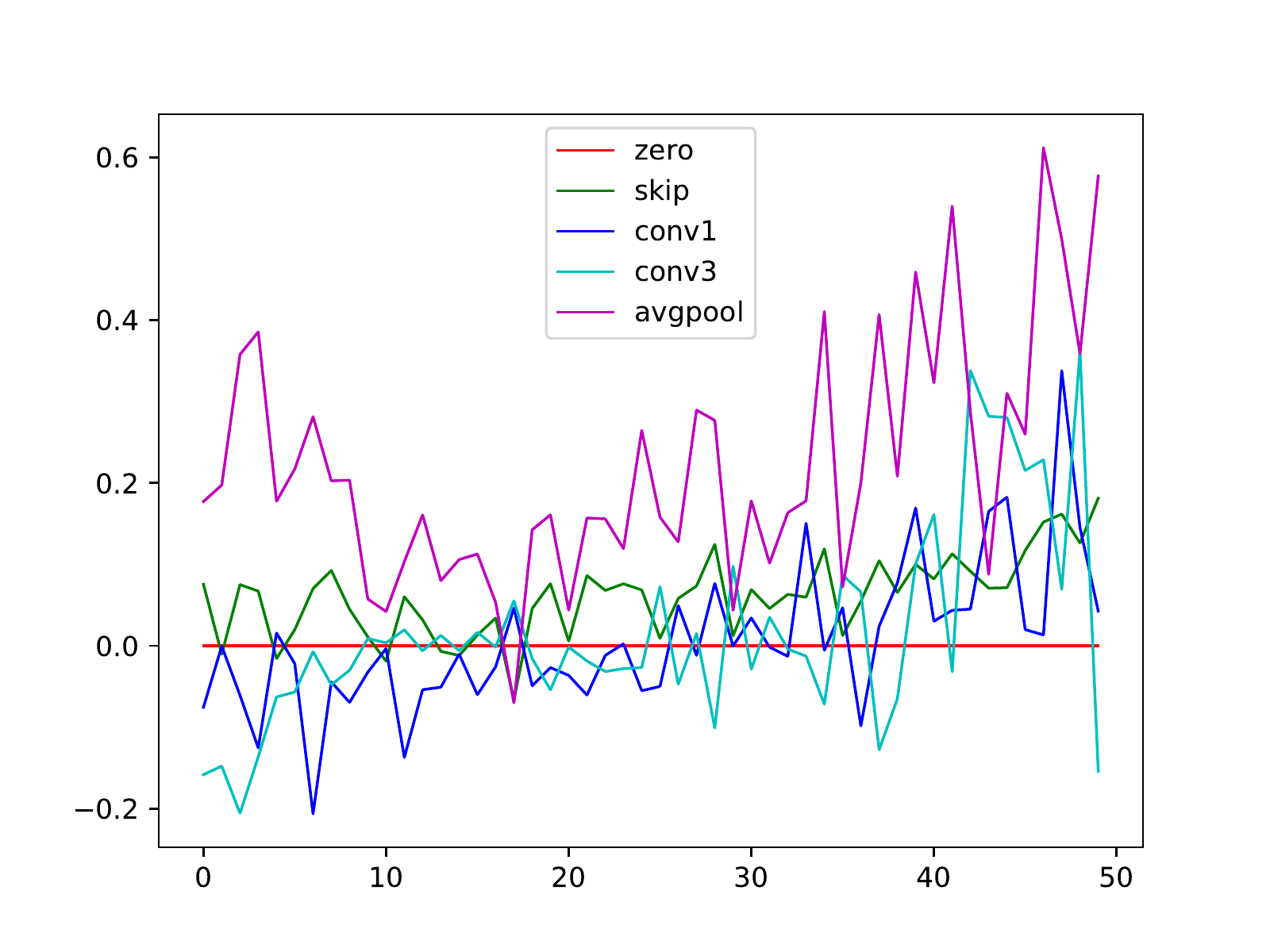}
 \caption{edge.1$\leftarrow$0}
\end{subfigure}
\hfill
 \begin{subfigure}[b]{0.3\linewidth}
 \centering
\includegraphics[width=\textwidth]{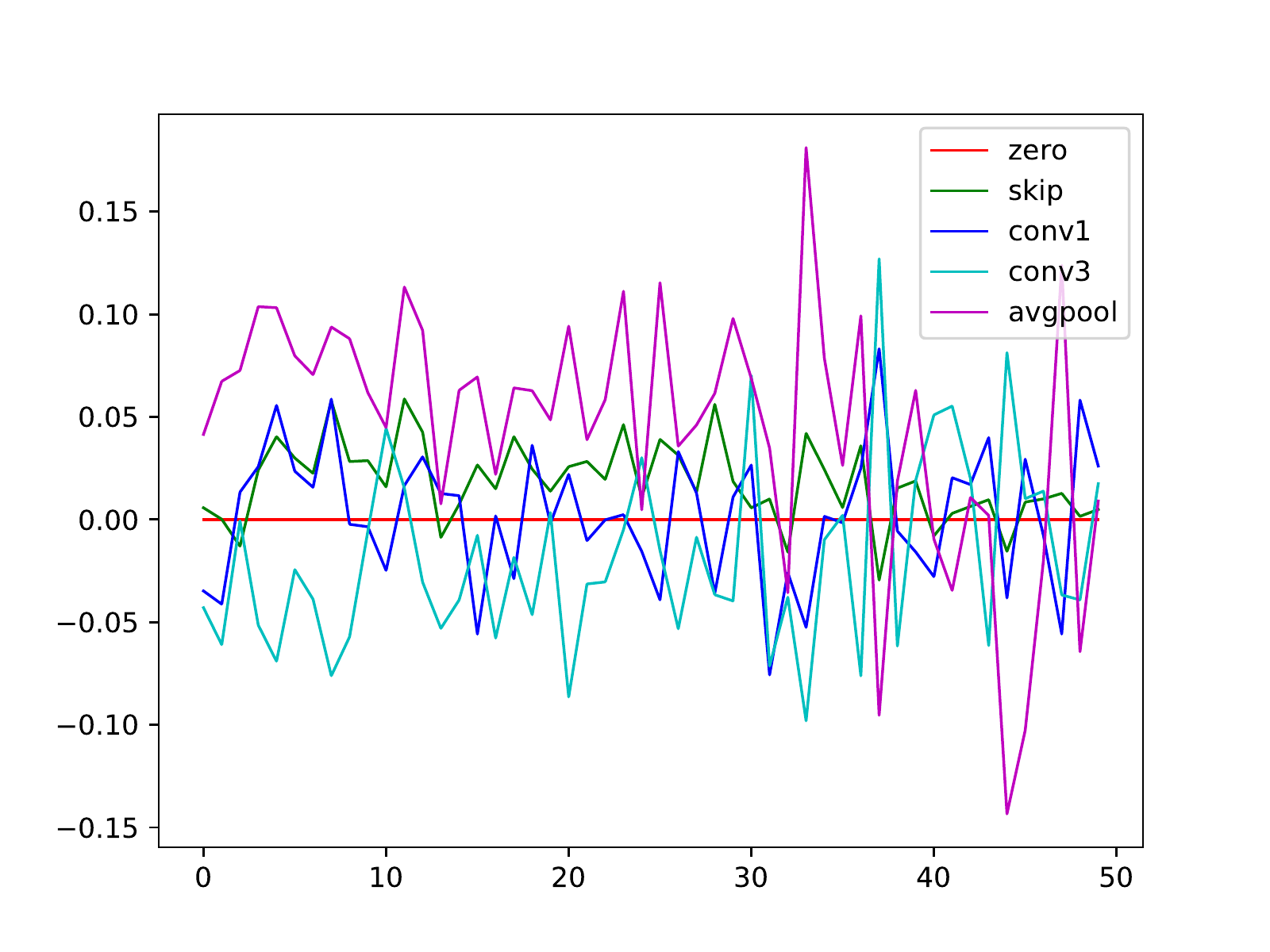}
 \caption{edge.2$\leftarrow$0}
\end{subfigure}
\hfill
 \begin{subfigure}[b]{0.3\linewidth}
 \centering
\includegraphics[width=\textwidth]{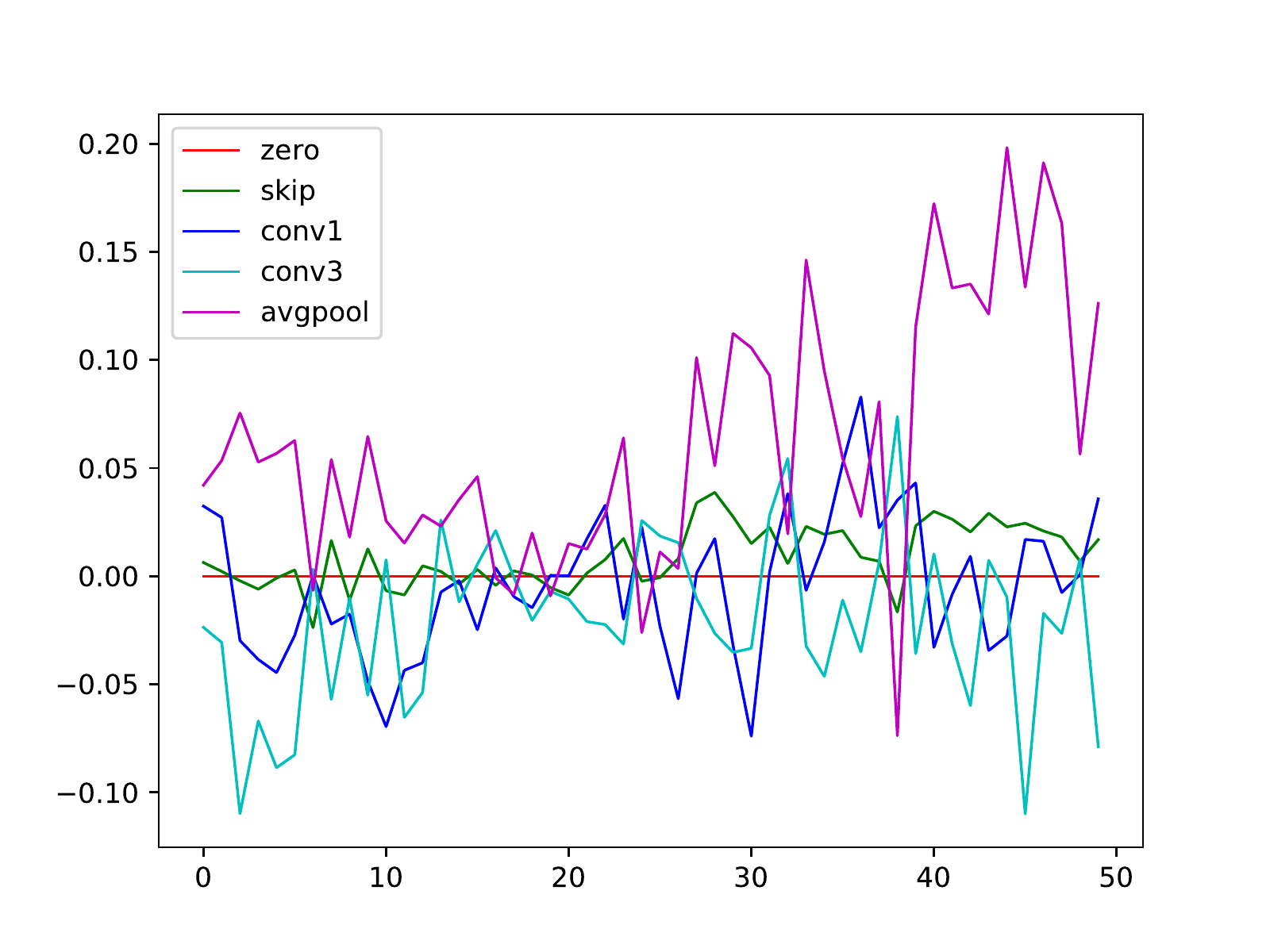}
 \caption{edge.2$\leftarrow$1}
\end{subfigure}
\hfill
\quad
 \begin{subfigure}[b]{0.3\linewidth}
 \centering
\includegraphics[width=\textwidth]{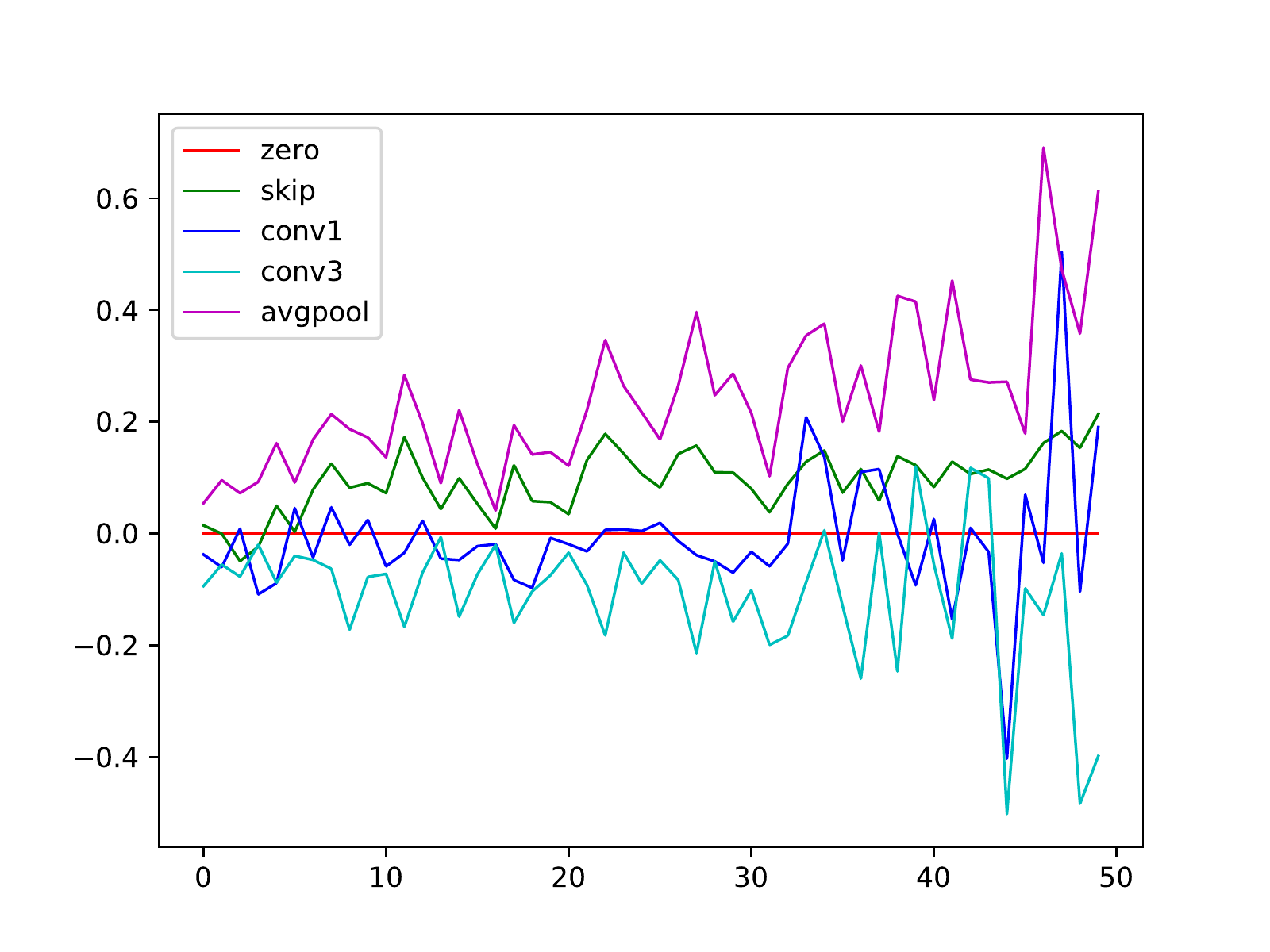}
 \caption{edge.3$\leftarrow$0}
\end{subfigure}
\hfill
 \begin{subfigure}[b]{0.3\linewidth}
 \centering
\includegraphics[width=\textwidth]{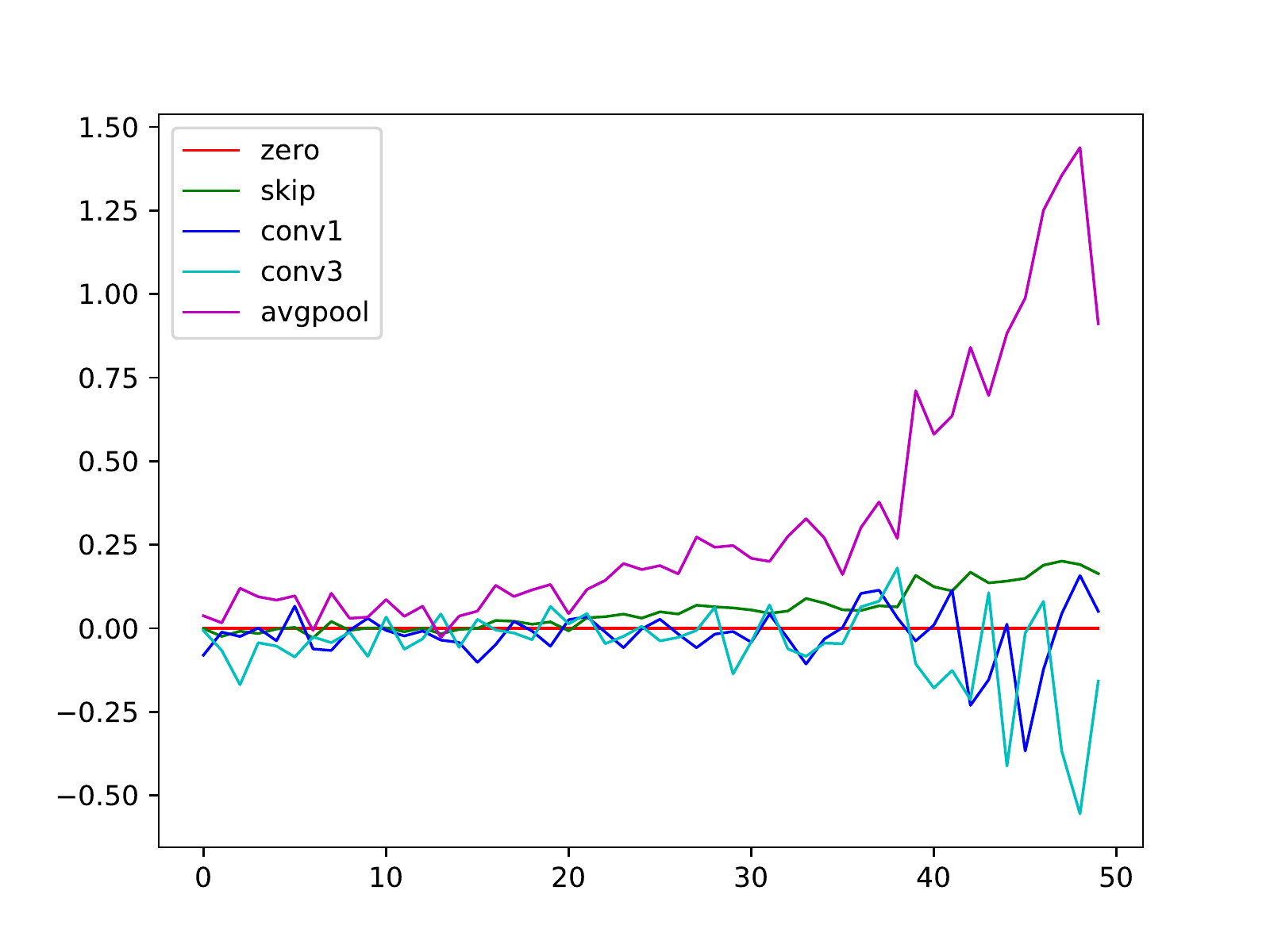}
 \caption{edge.3$\leftarrow$1}
\end{subfigure}
\hfill
 \begin{subfigure}[b]{0.3\linewidth}
 \centering
\includegraphics[width=\textwidth]{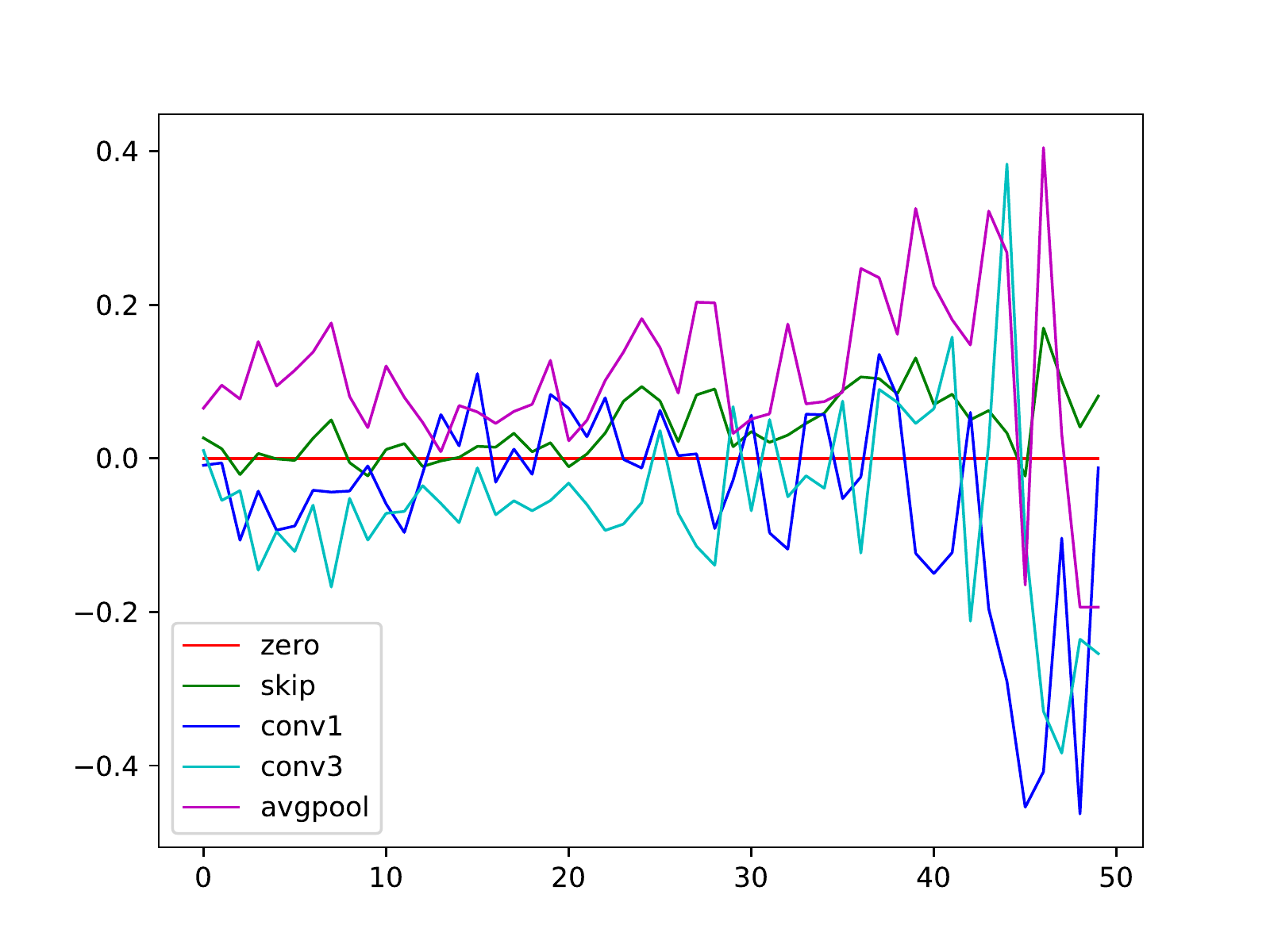}
 \caption{edge.3$\leftarrow$2}
\end{subfigure}
\hfill
\caption{DARTS, the 0th cell.}
\end{figure*}

\begin{figure*}[h]
 \begin{subfigure}[b]{0.3\linewidth}
 \centering
\includegraphics[width=\textwidth]{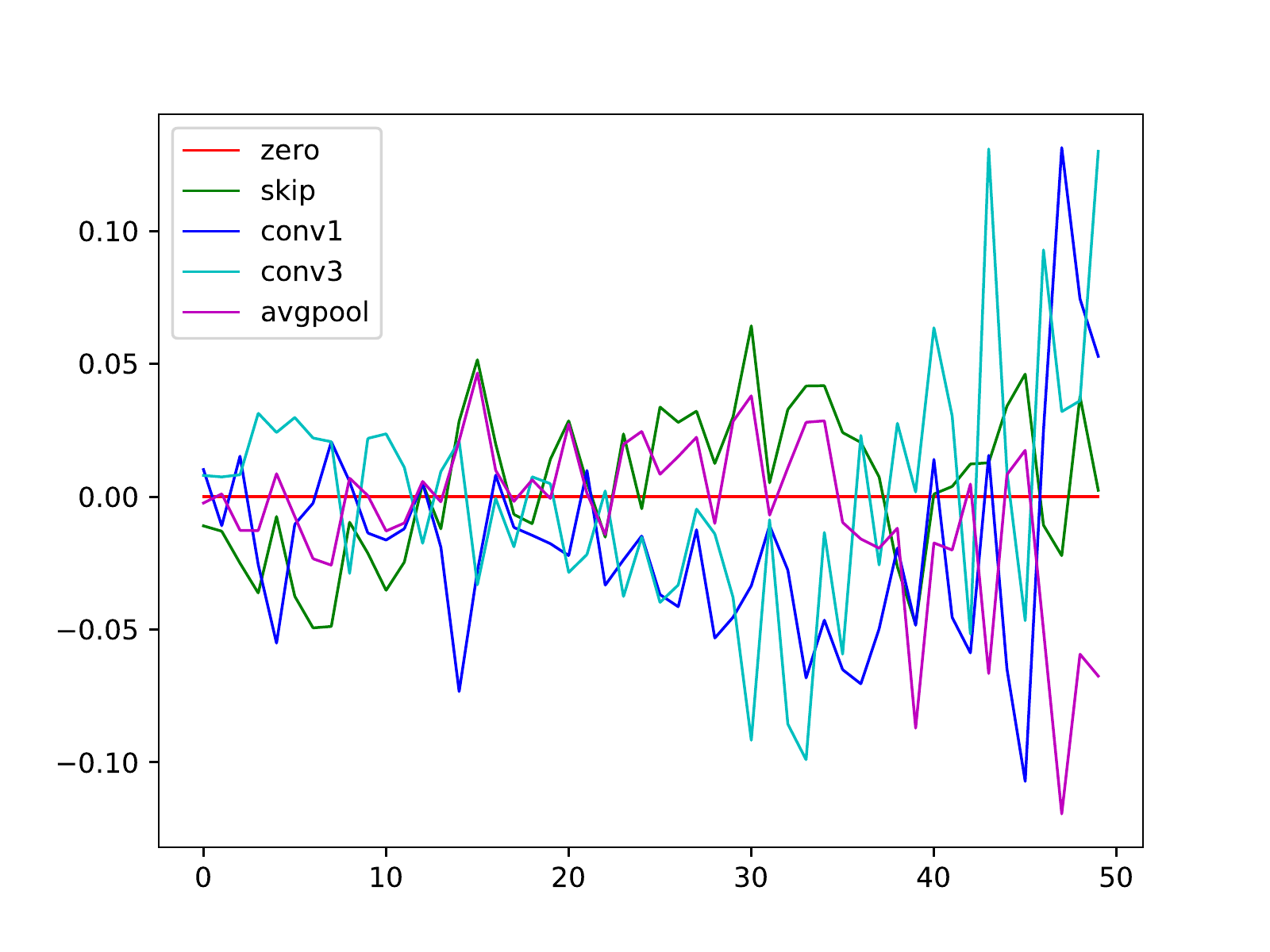}
 \caption{edge.1$\leftarrow$0}
\end{subfigure}
\hfill
 \begin{subfigure}[b]{0.3\linewidth}
 \centering
\includegraphics[width=\textwidth]{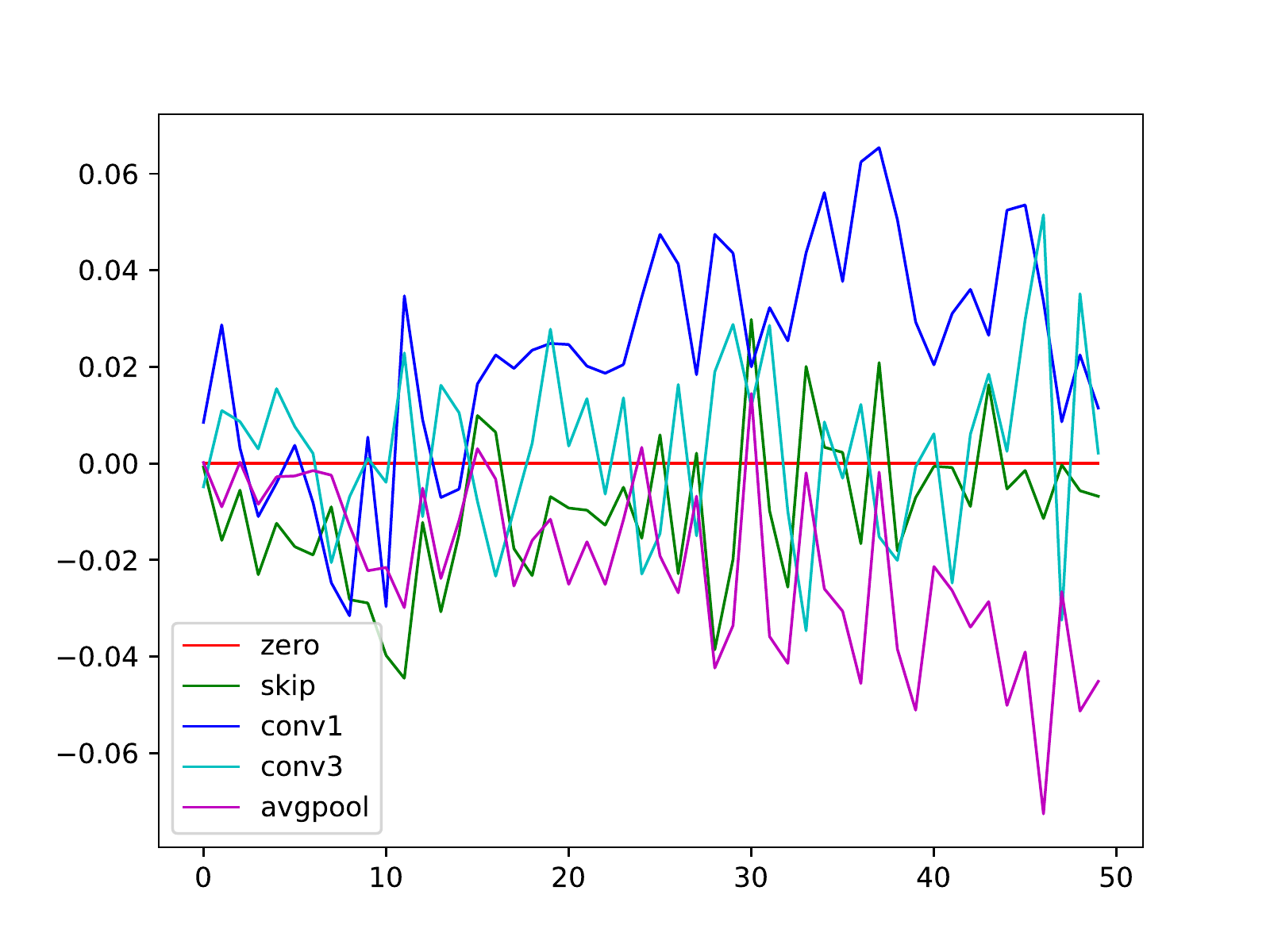}
 \caption{edge.2$\leftarrow$0}
\end{subfigure}
\hfill
 \begin{subfigure}[b]{0.3\linewidth}
 \centering
\includegraphics[width=\textwidth]{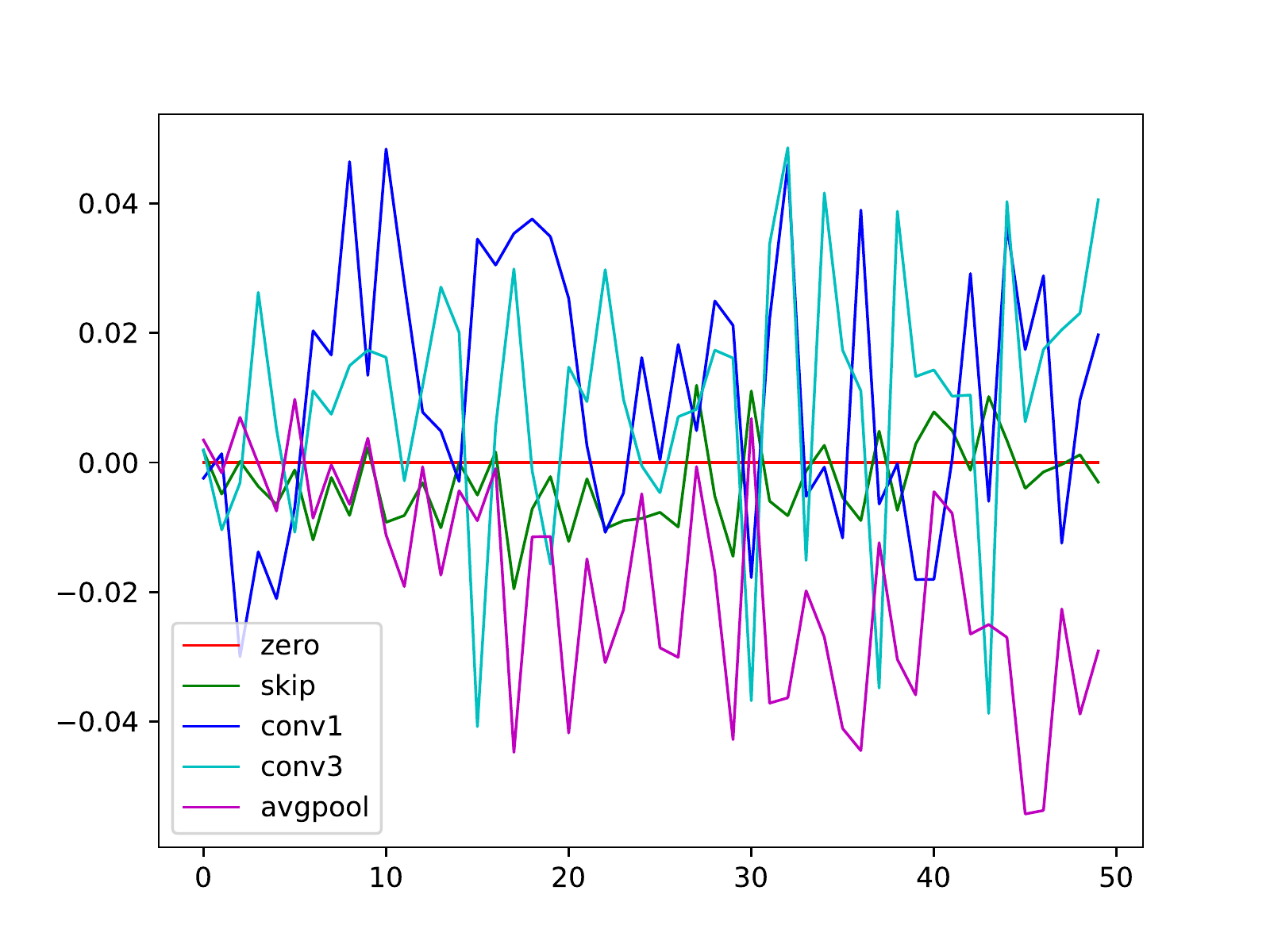}
 \caption{edge.2$\leftarrow$1}
\end{subfigure}
\hfill
\quad
 \begin{subfigure}[b]{0.3\linewidth}
 \centering
\includegraphics[width=\textwidth]{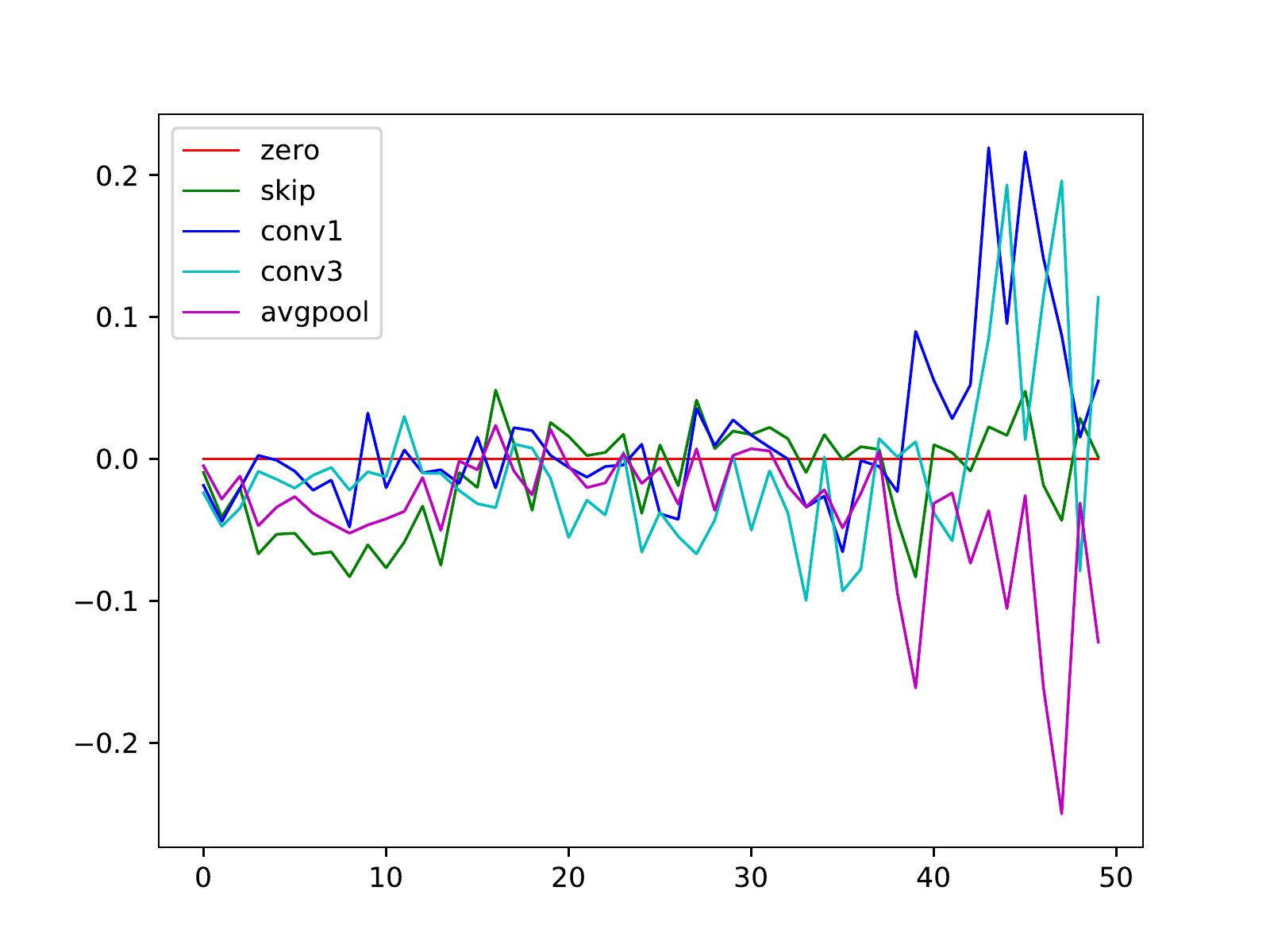}
 \caption{edge.3$\leftarrow$0}
\end{subfigure}
\hfill
 \begin{subfigure}[b]{0.3\linewidth}
 \centering
\includegraphics[width=\textwidth]{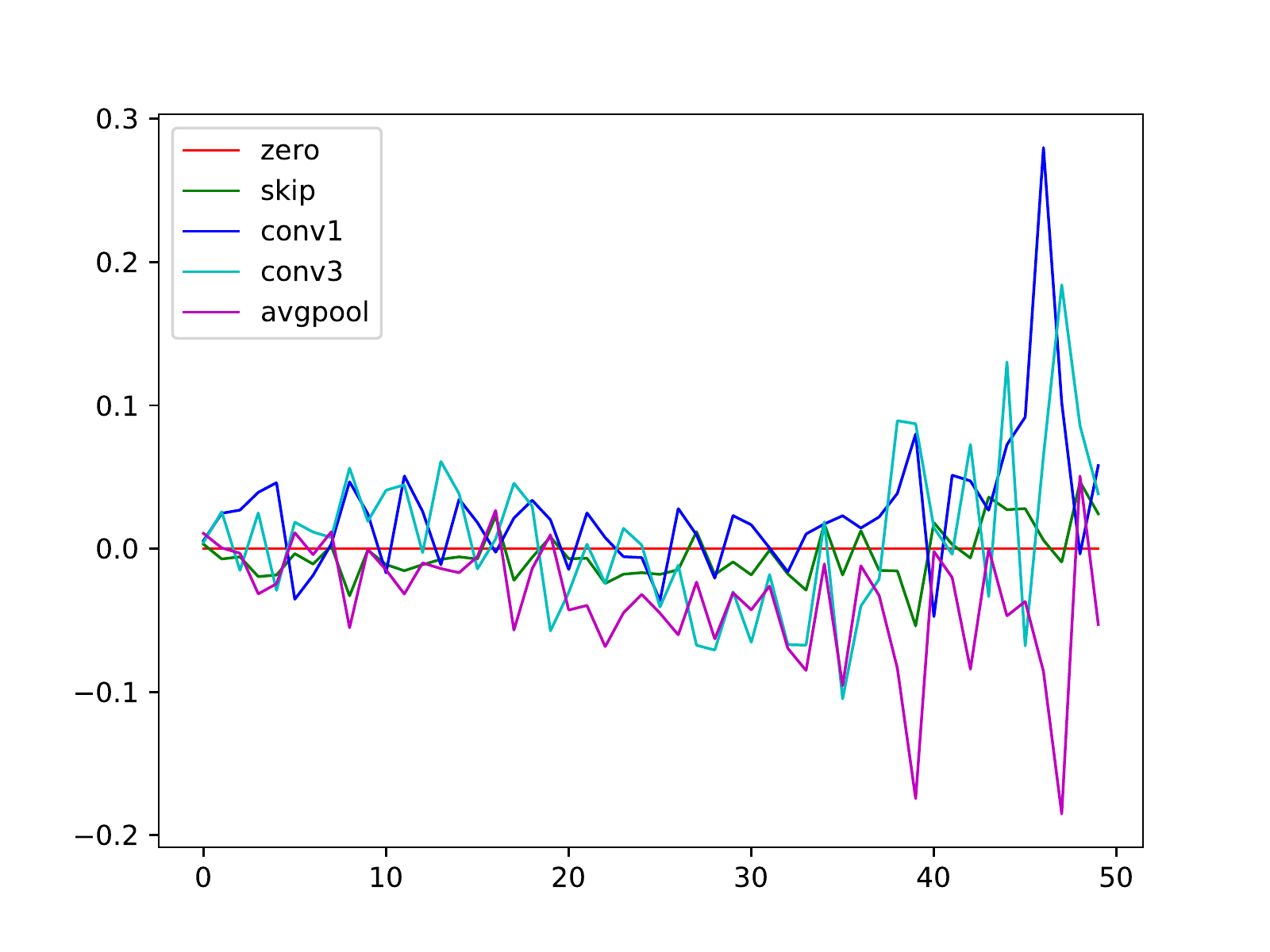}
 \caption{edge.3$\leftarrow$1}
\end{subfigure}
\hfill
 \begin{subfigure}[b]{0.3\linewidth}
 \centering
\includegraphics[width=\textwidth]{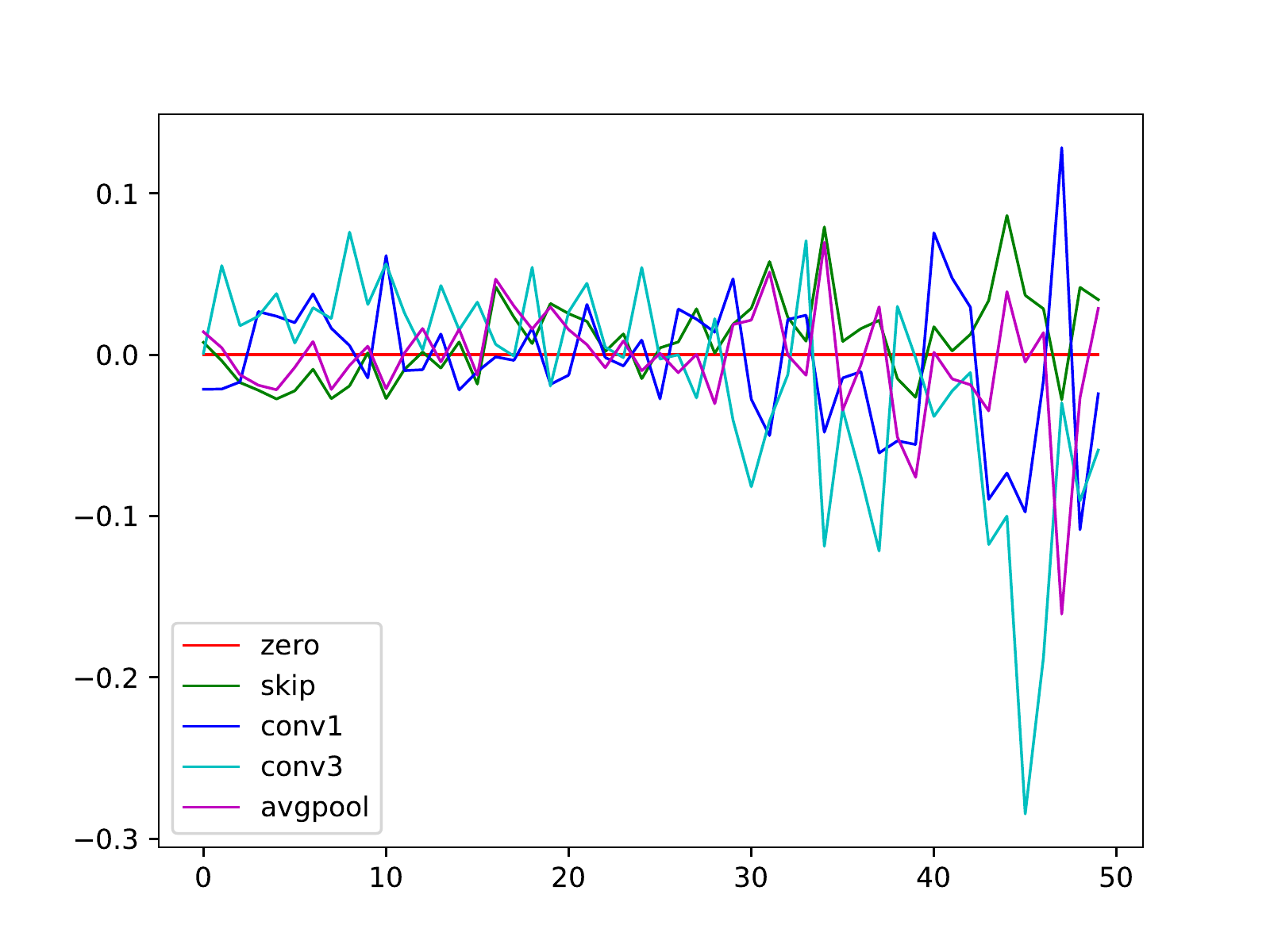}
 \caption{edge.3$\leftarrow$2}
\end{subfigure}
\hfill
\caption{DARTS, the 8th cell.}
\end{figure*}

\begin{figure*}[h]
 \begin{subfigure}[b]{0.3\linewidth}
 \centering
\includegraphics[width=\textwidth]{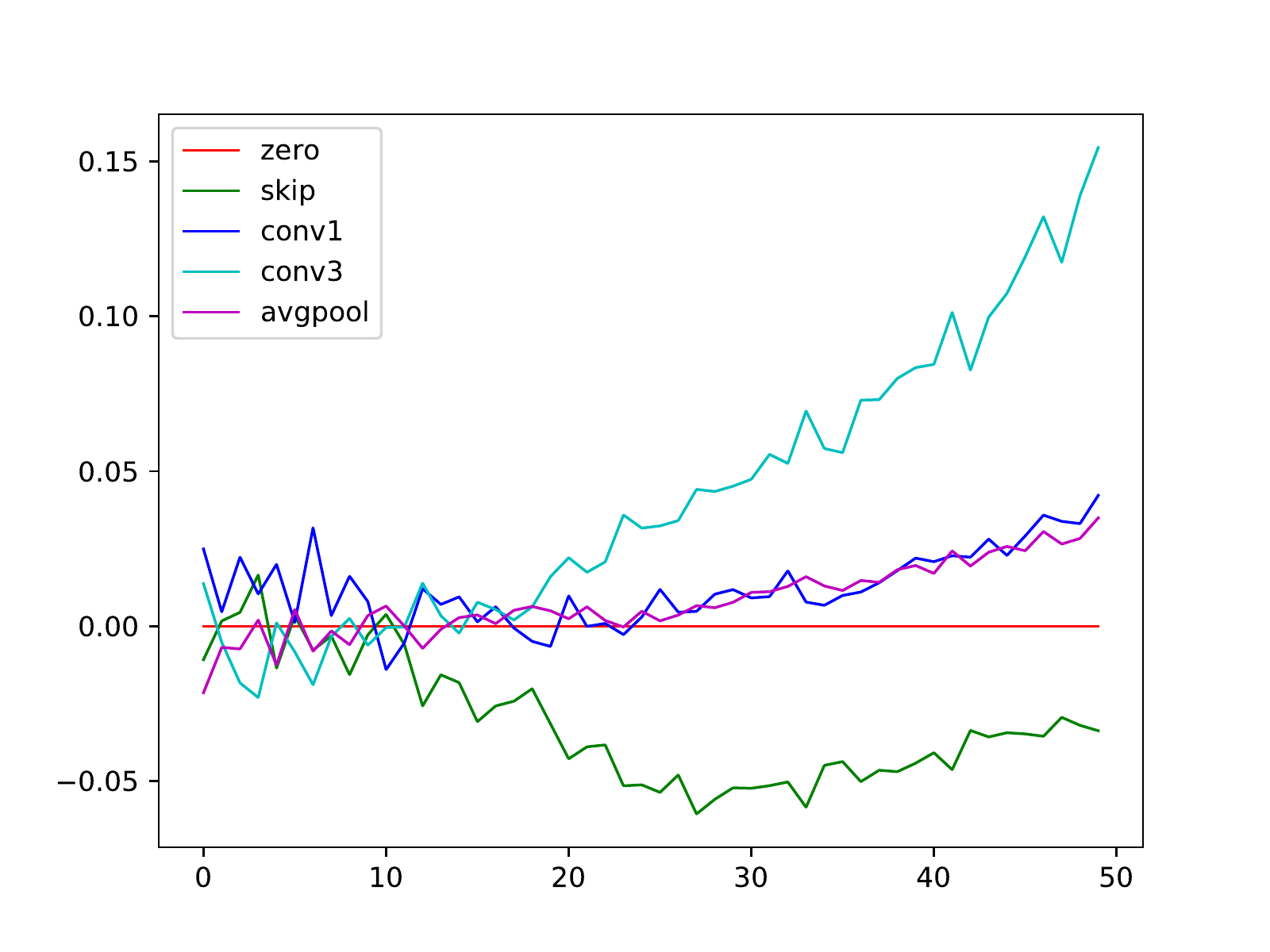}
 \caption{edge.1$\leftarrow$0}
\end{subfigure}
\hfill
 \begin{subfigure}[b]{0.3\linewidth}
 \centering
\includegraphics[width=\textwidth]{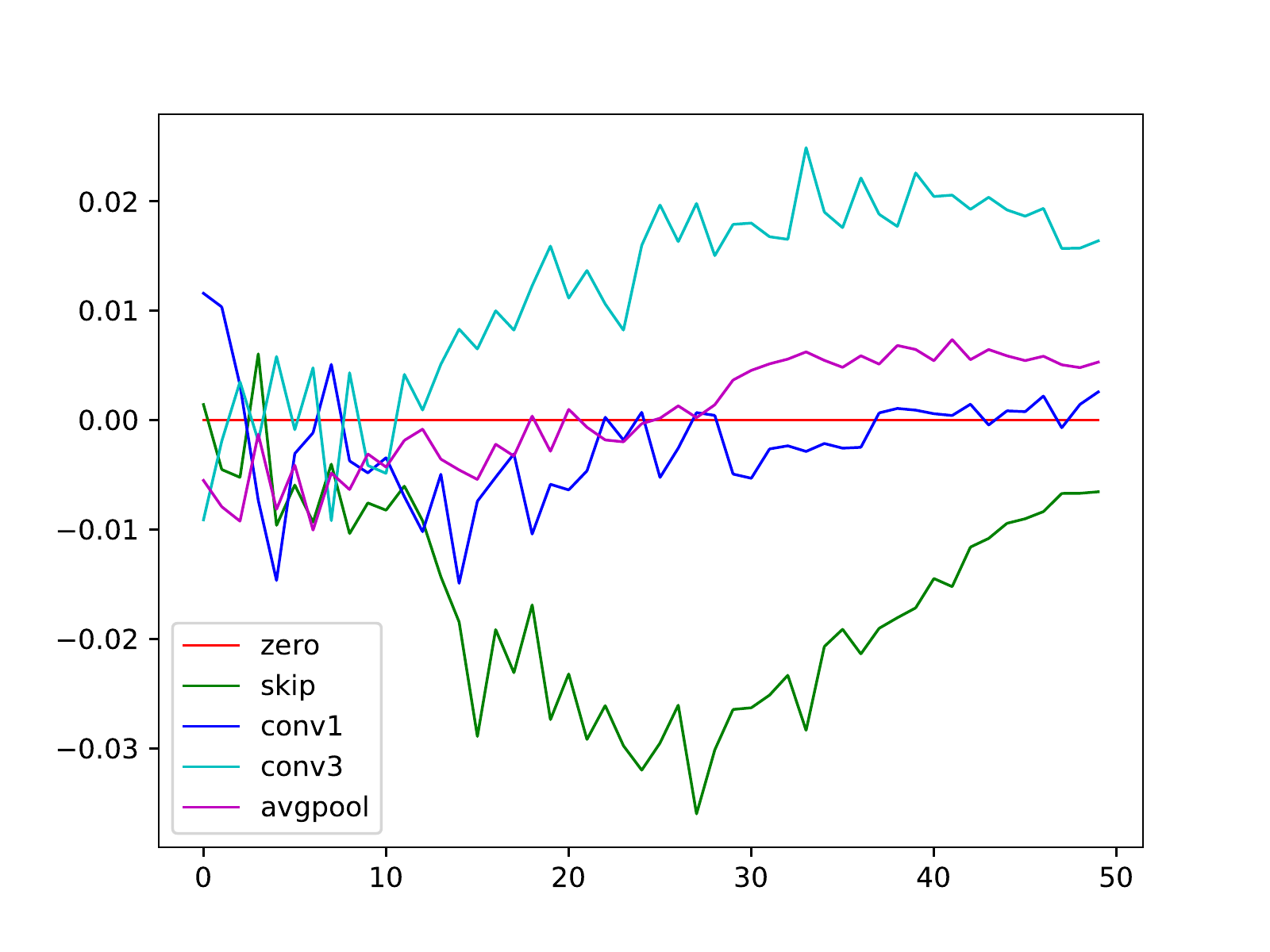}
 \caption{edge.2$\leftarrow$0}
\end{subfigure}
\hfill
 \begin{subfigure}[b]{0.3\linewidth}
 \centering
\includegraphics[width=\textwidth]{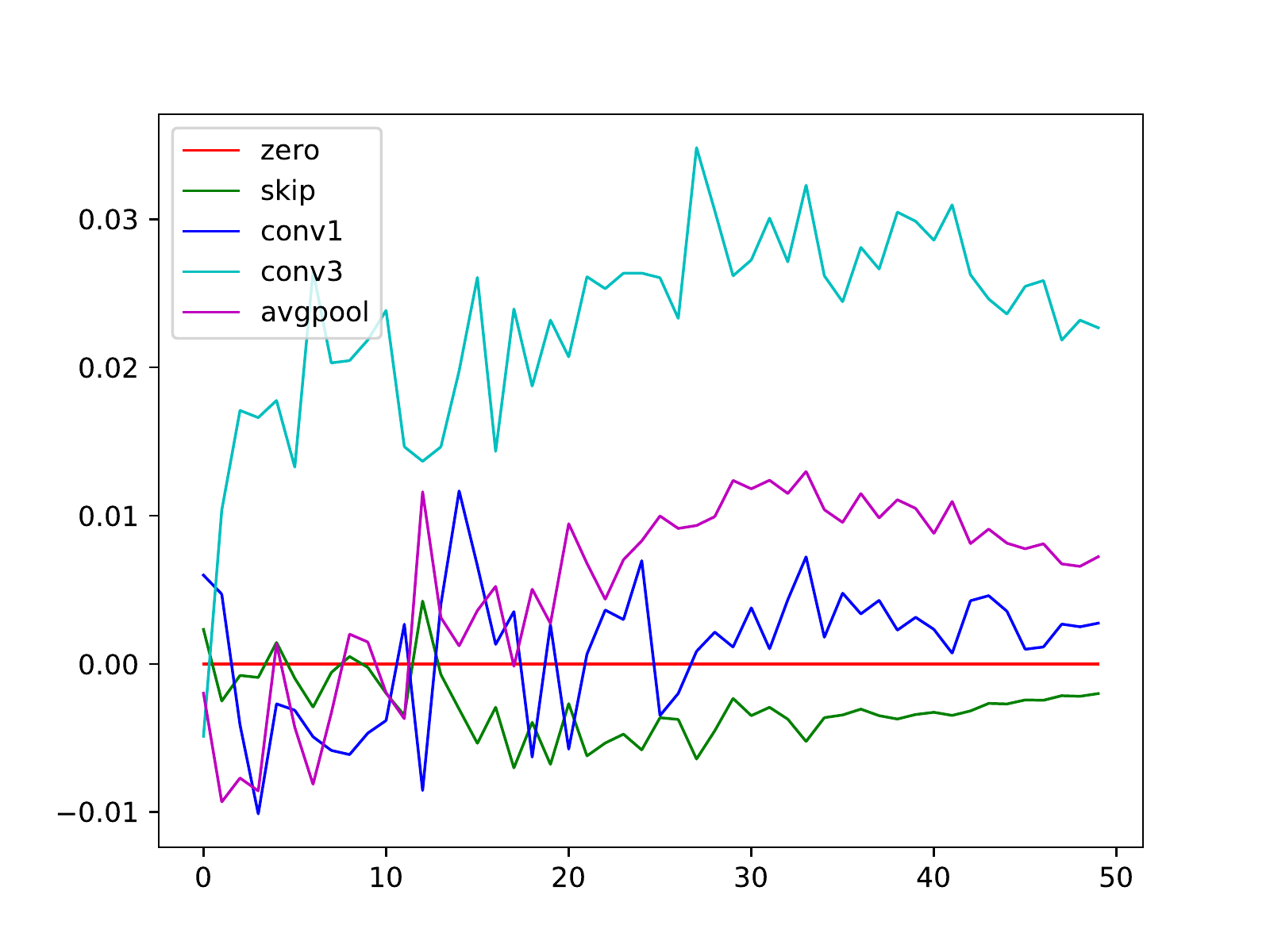}
 \caption{edge.2$\leftarrow$1}
\end{subfigure}
\hfill
\quad
\begin{subfigure}[b]{0.3\linewidth}
 \centering
\includegraphics[width=\textwidth]{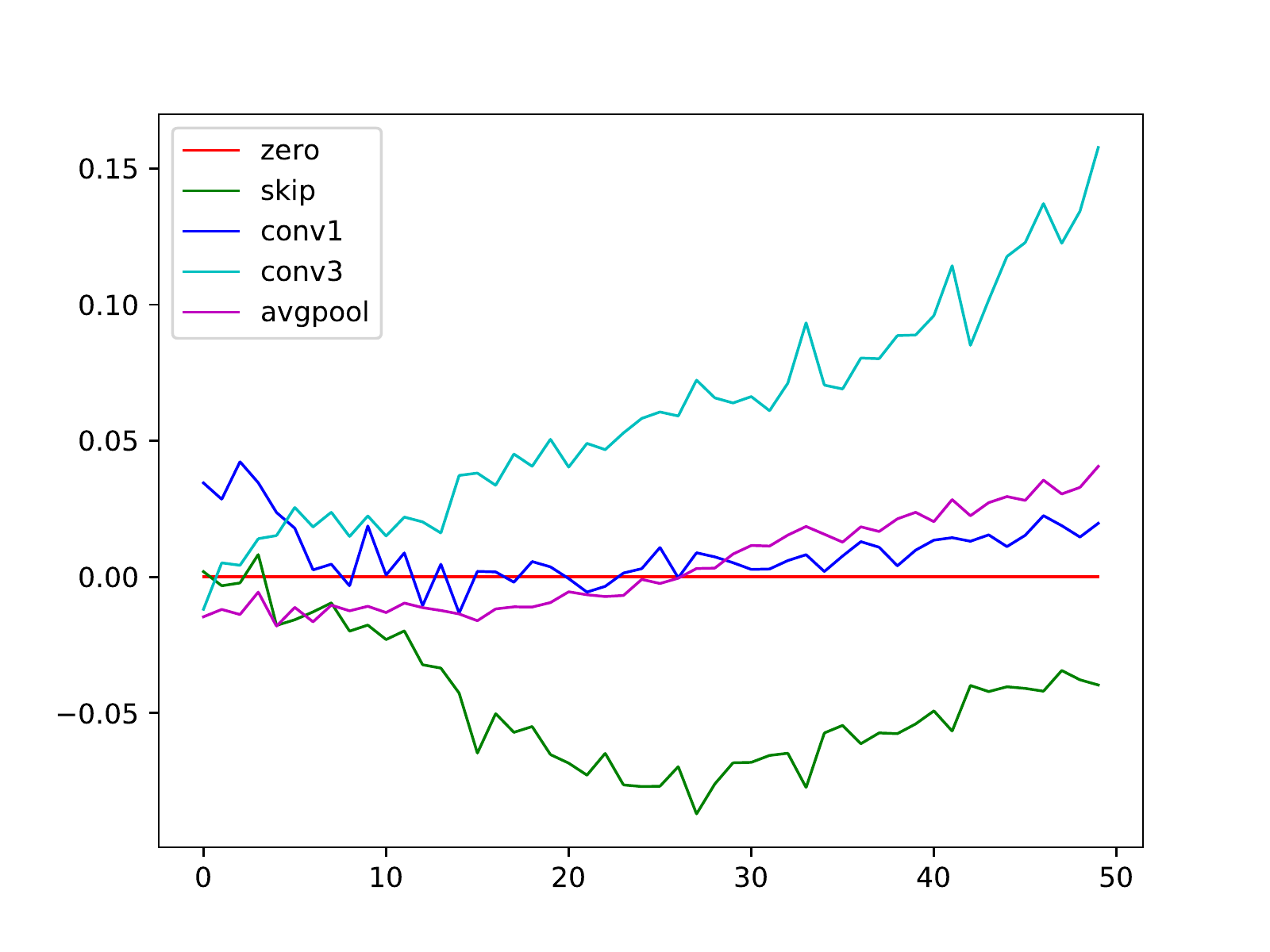}
 \caption{edge.3$\leftarrow$0}
\end{subfigure}
\hfill
 \begin{subfigure}[b]{0.3\linewidth}
 \centering
\includegraphics[width=\textwidth]{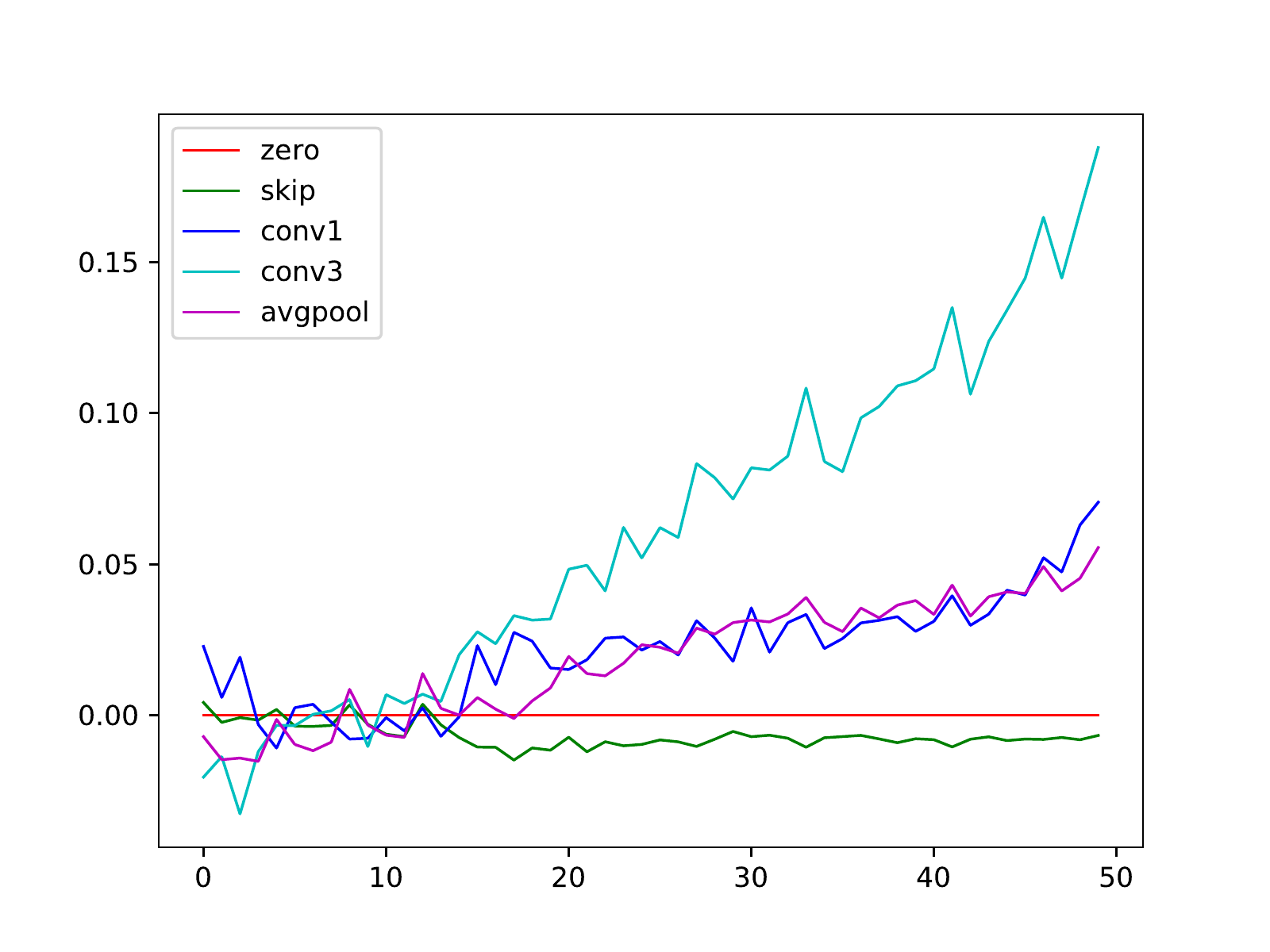}
 \caption{edge.3$\leftarrow$1}
\end{subfigure}
\hfill
 \begin{subfigure}[b]{0.3\linewidth}
 \centering
\includegraphics[width=\textwidth]{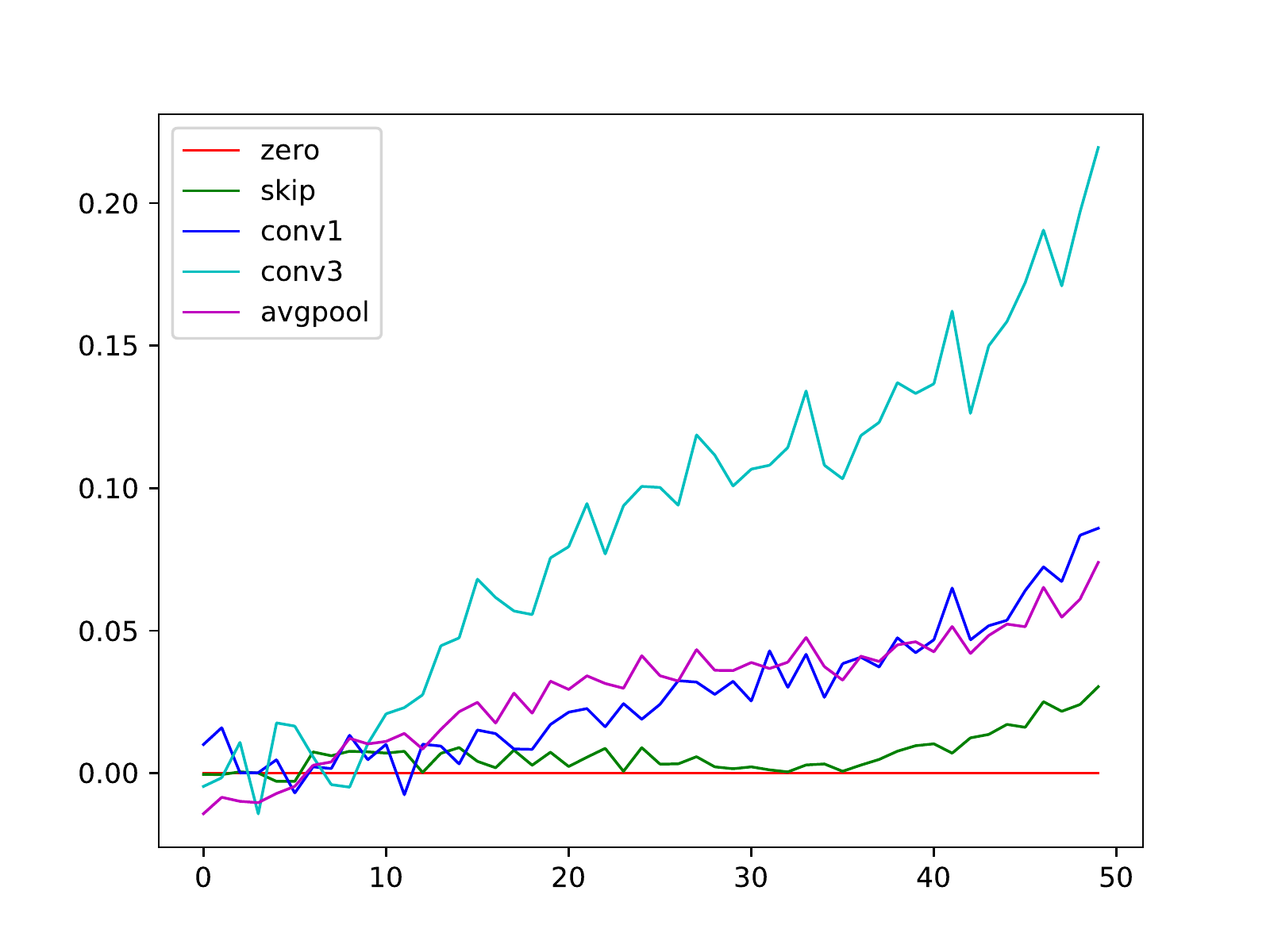}
 \caption{edge.3$\leftarrow$2}
\end{subfigure}
\hfill
\caption{DARTS, the 16th cell.}
\end{figure*}
\begin{figure*}[h]
\centering
 \begin{subfigure}[b]{0.3\linewidth}
 \centering
\includegraphics[width=\textwidth]{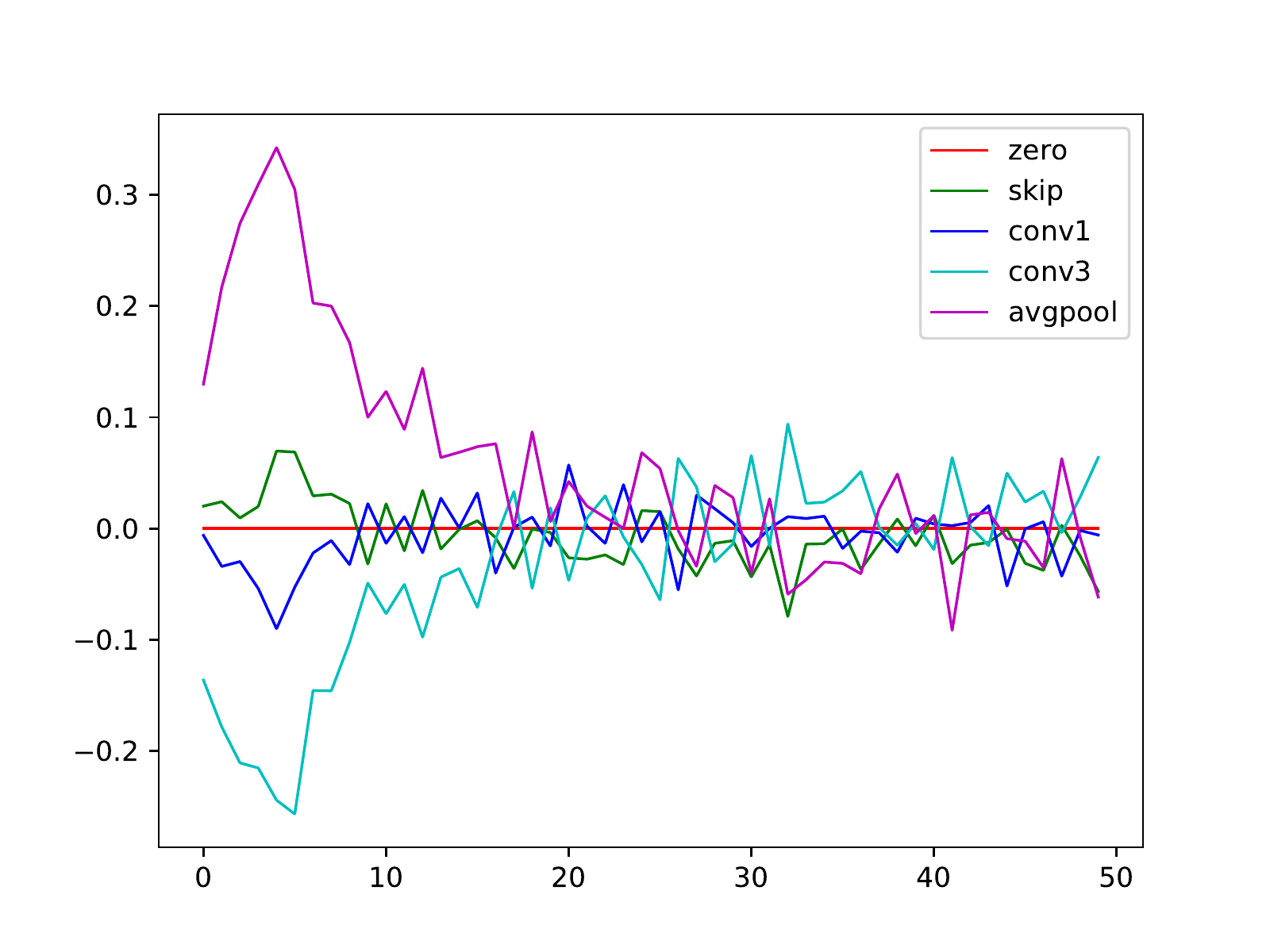}
 \caption{edge.1$\leftarrow$0}
\end{subfigure}
\hfill
 \begin{subfigure}[b]{0.3\linewidth}
 \centering
\includegraphics[width=\textwidth]{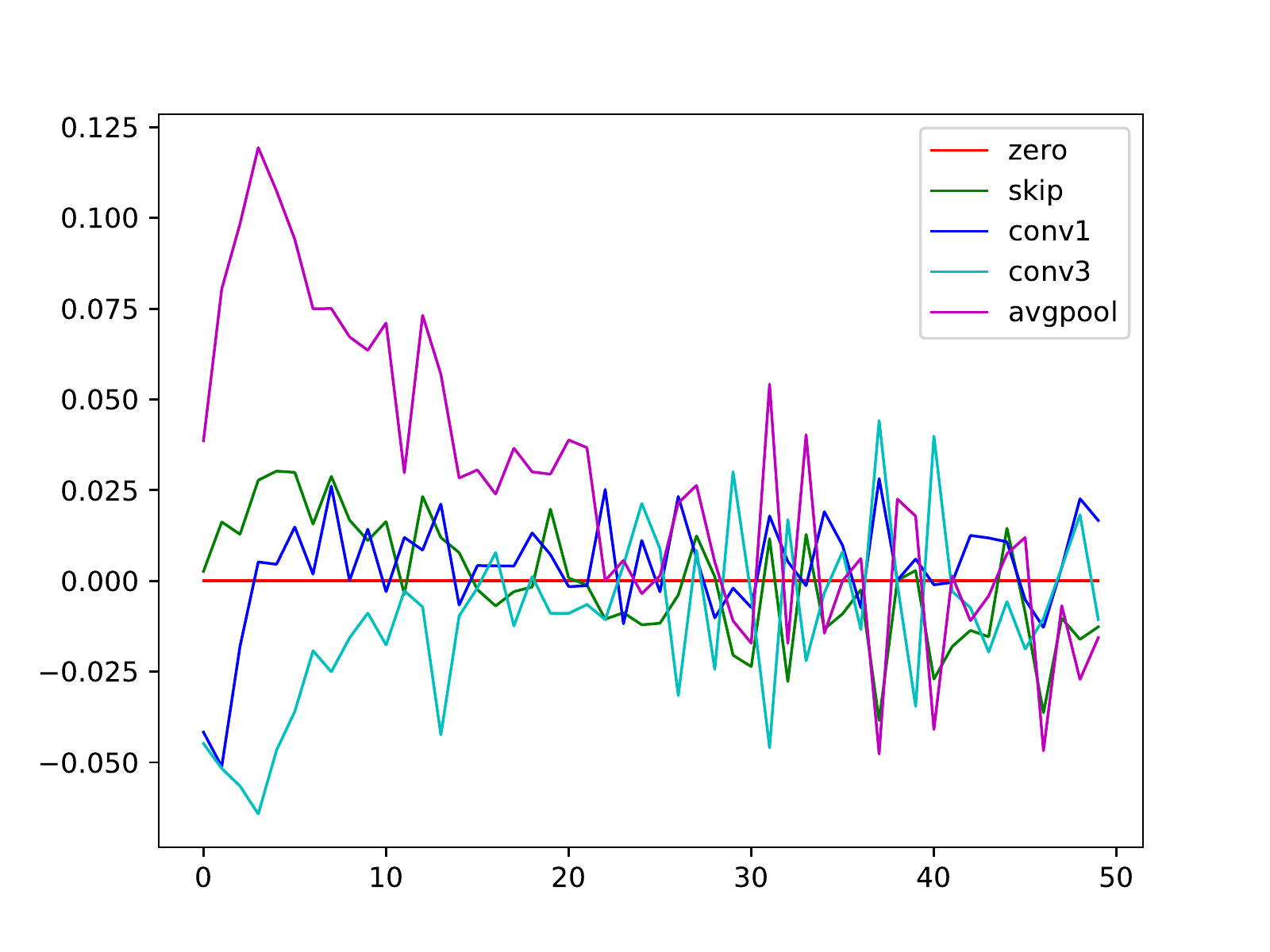}
 \caption{edge.2$\leftarrow$0}
\end{subfigure}
\hfill
 \begin{subfigure}[b]{0.3\linewidth}
 \centering
\includegraphics[width=\textwidth]{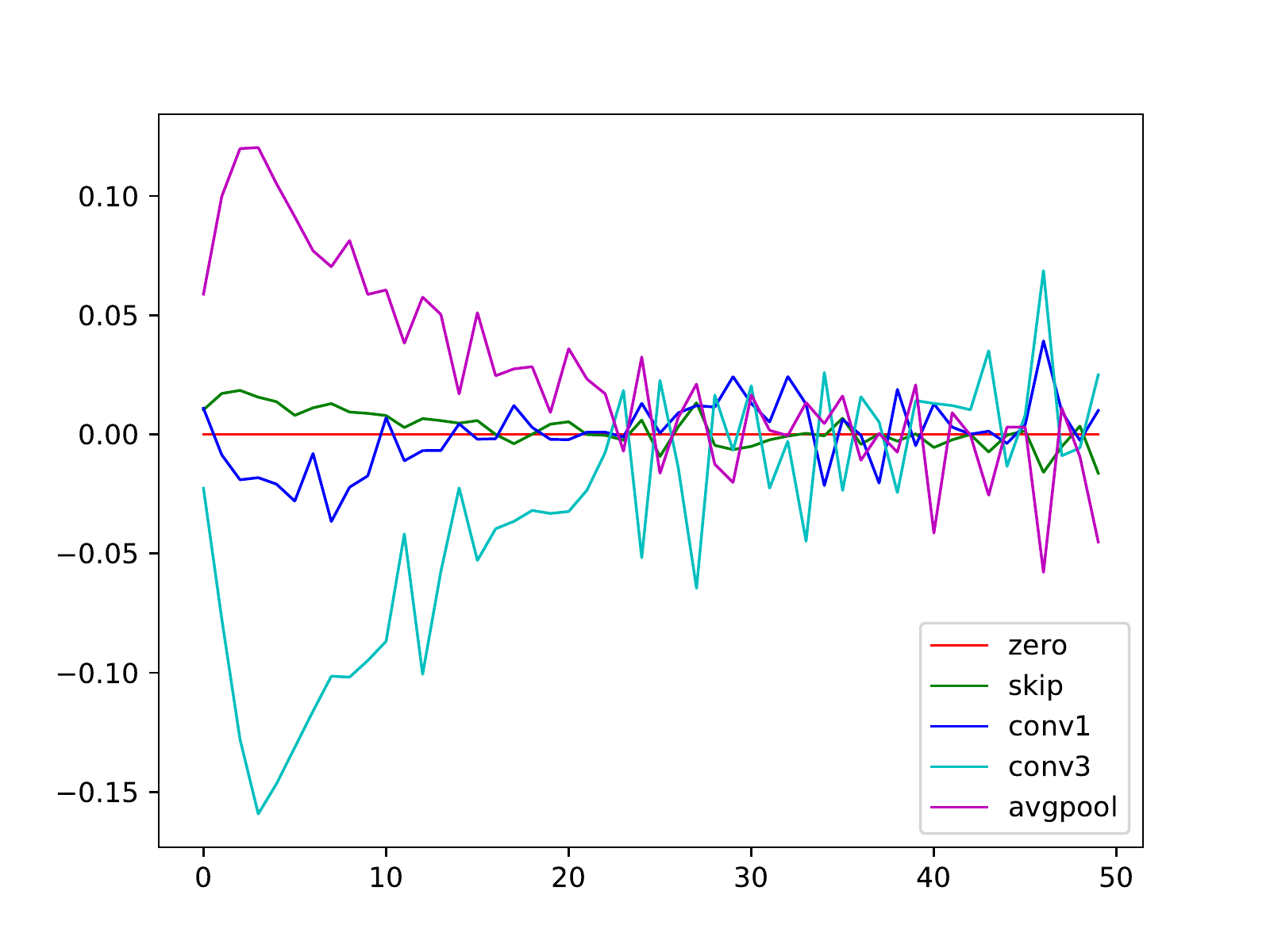}
 \caption{edge.2$\leftarrow$1}
\end{subfigure}
\hfill
\quad
 \begin{subfigure}[b]{0.3\linewidth}
 \centering
\includegraphics[width=\textwidth]{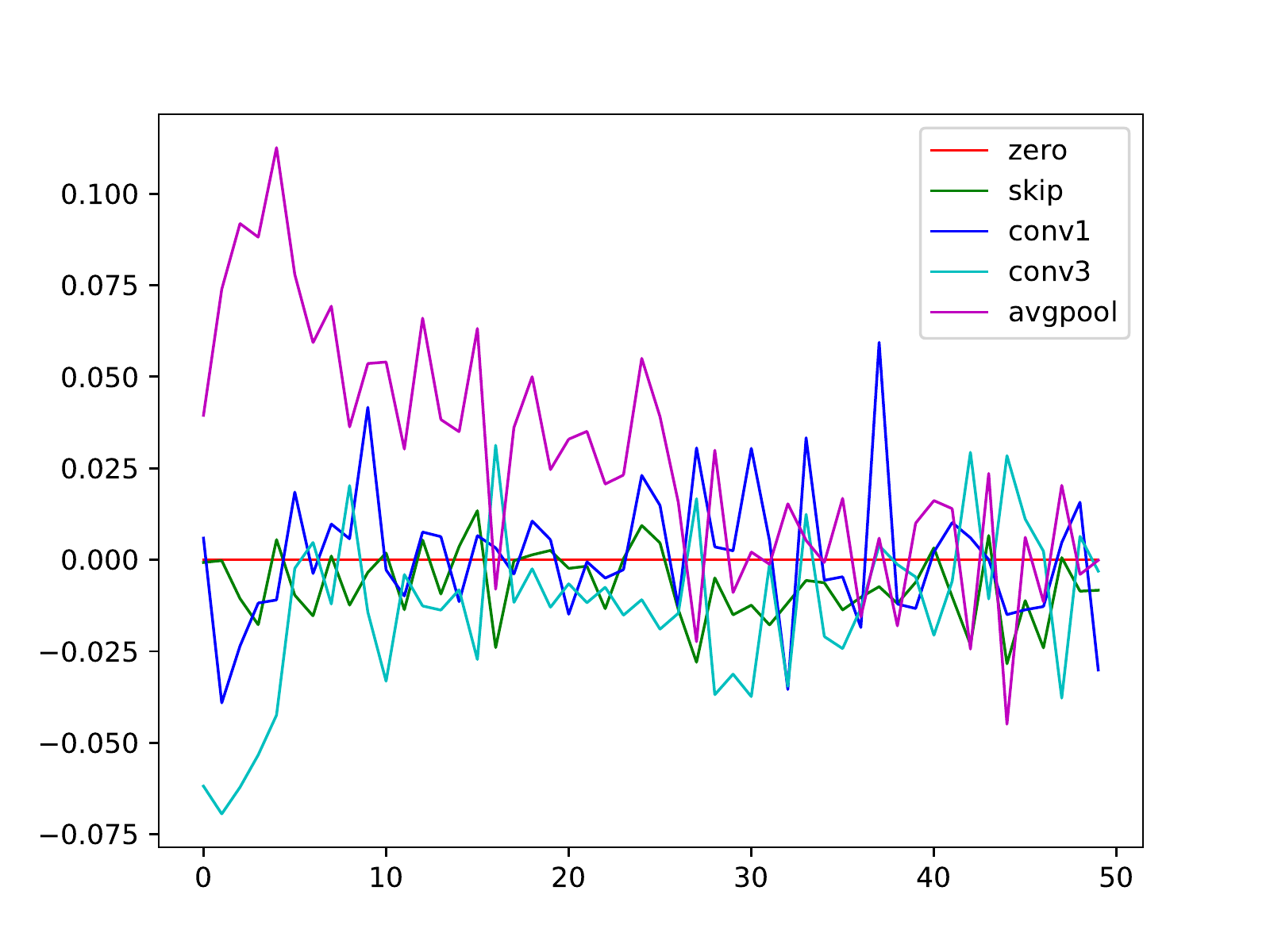}
 \caption{edge.3$\leftarrow$0}
\end{subfigure}
\hfill
 \begin{subfigure}[b]{0.3\linewidth}
 \centering
\includegraphics[width=\textwidth]{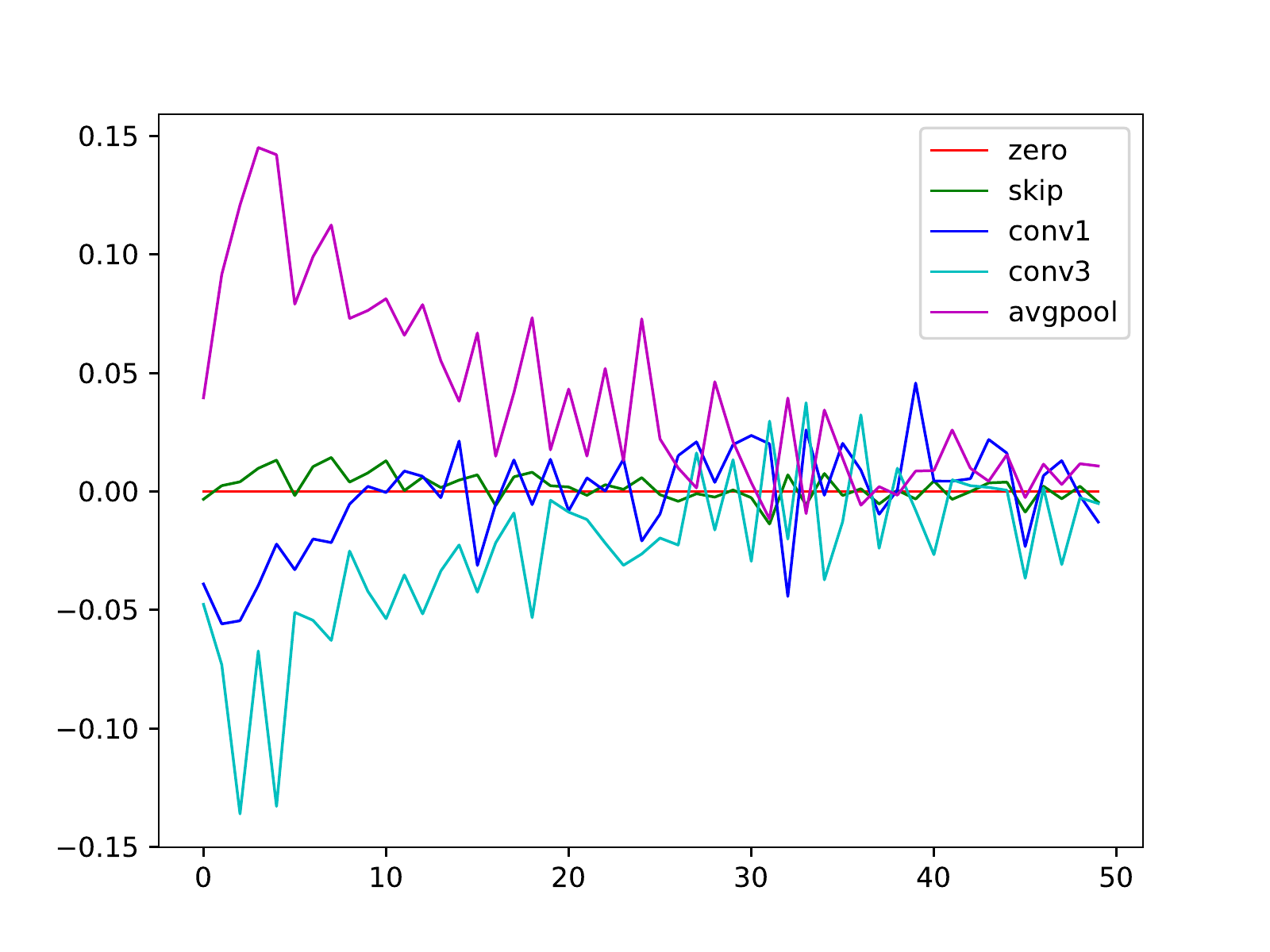}
 \caption{edge.3$\leftarrow$1}
\end{subfigure}
\hfill
 \begin{subfigure}[b]{0.3\linewidth}
 \centering
\includegraphics[width=\textwidth]{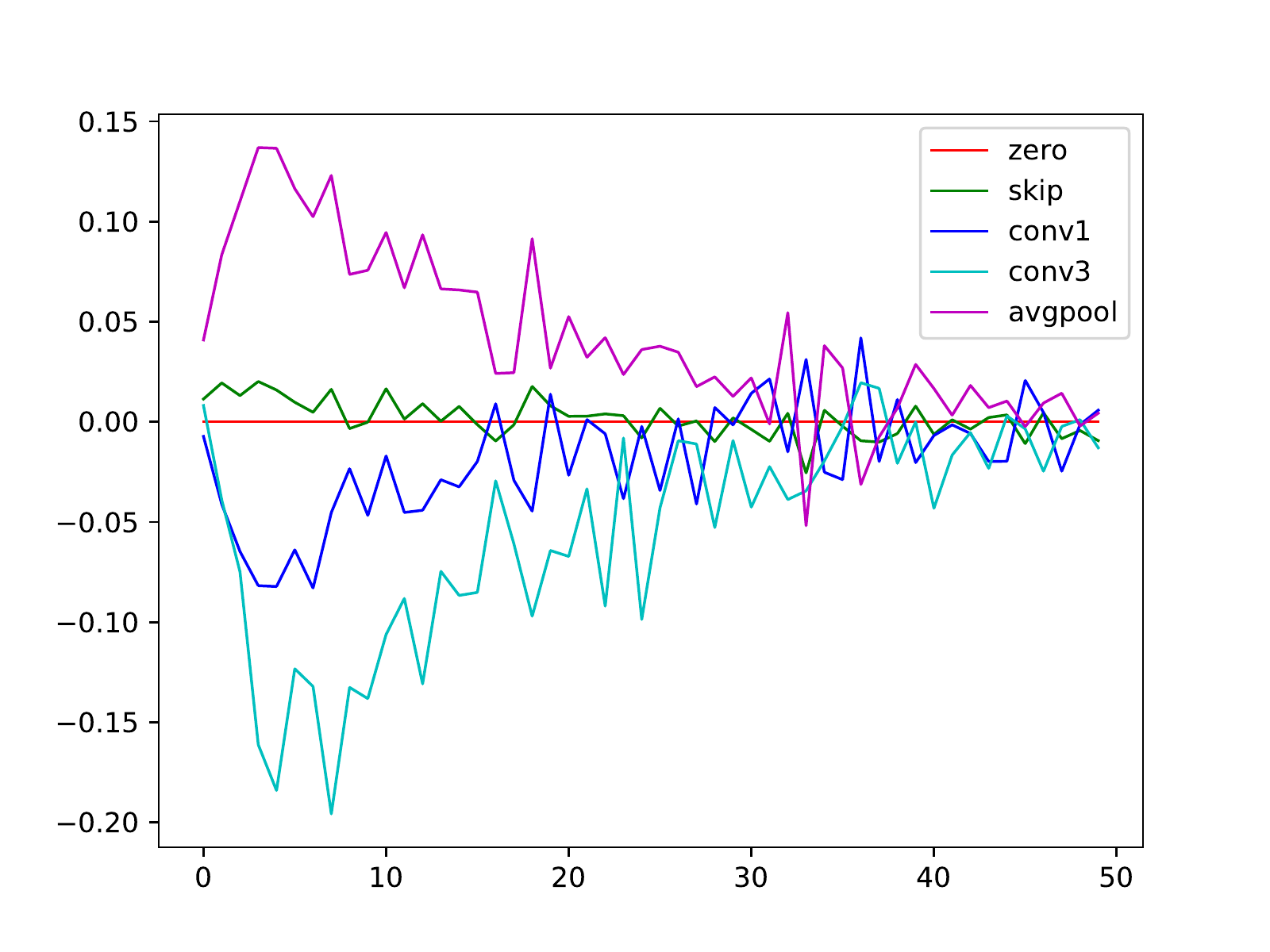}
 \caption{edge.3$\leftarrow$2}
\end{subfigure}
\hfill
\caption{Single-DARTS, the 0th cell.}
\end{figure*}

\begin{figure*}[h]
 \begin{subfigure}[b]{0.3\linewidth}
 \centering
\includegraphics[width=\textwidth]{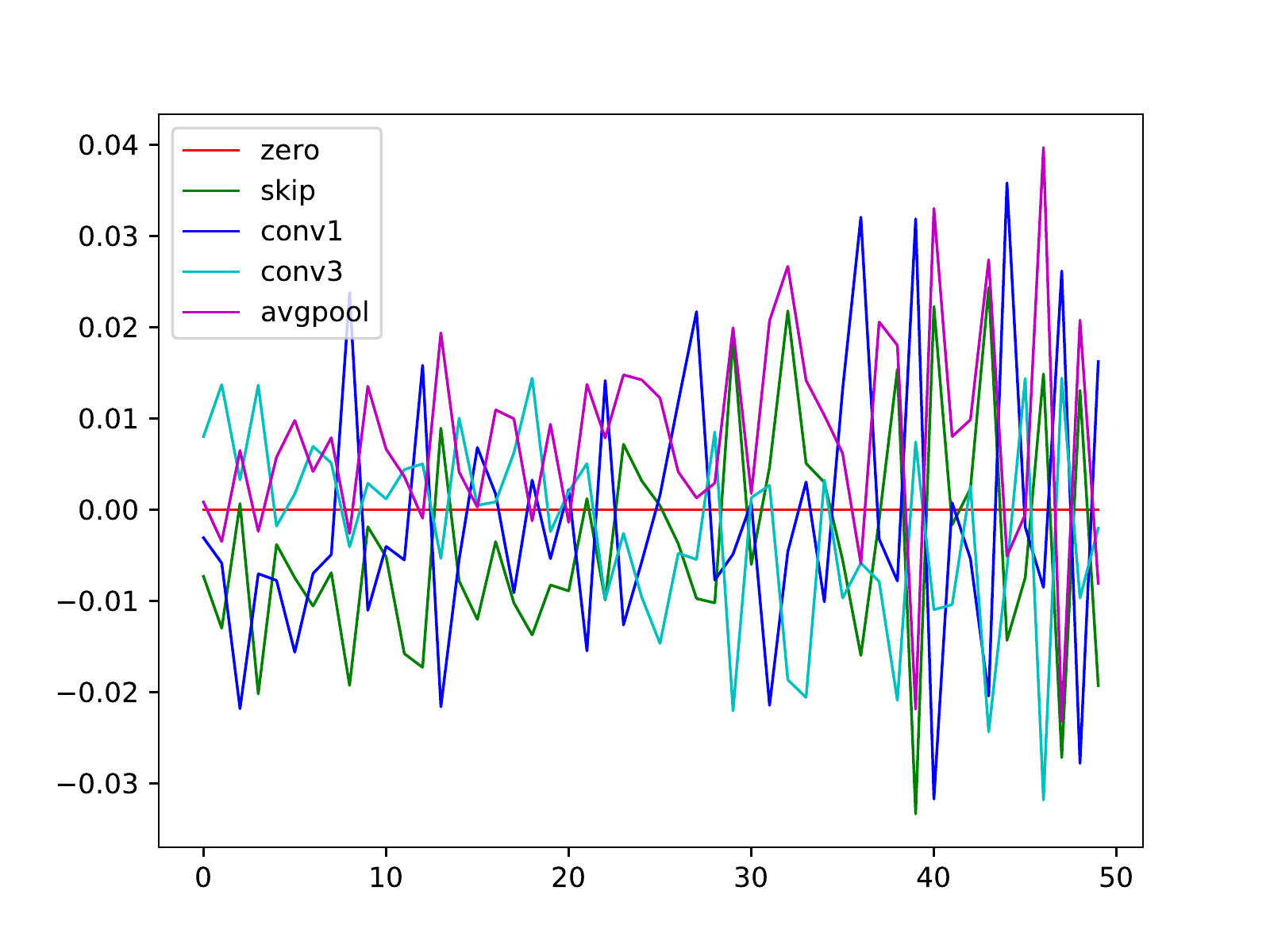}
 \caption{edge.1$\leftarrow$0}
\end{subfigure}
\hfill
 \begin{subfigure}[b]{0.3\linewidth}
 \centering
\includegraphics[width=\textwidth]{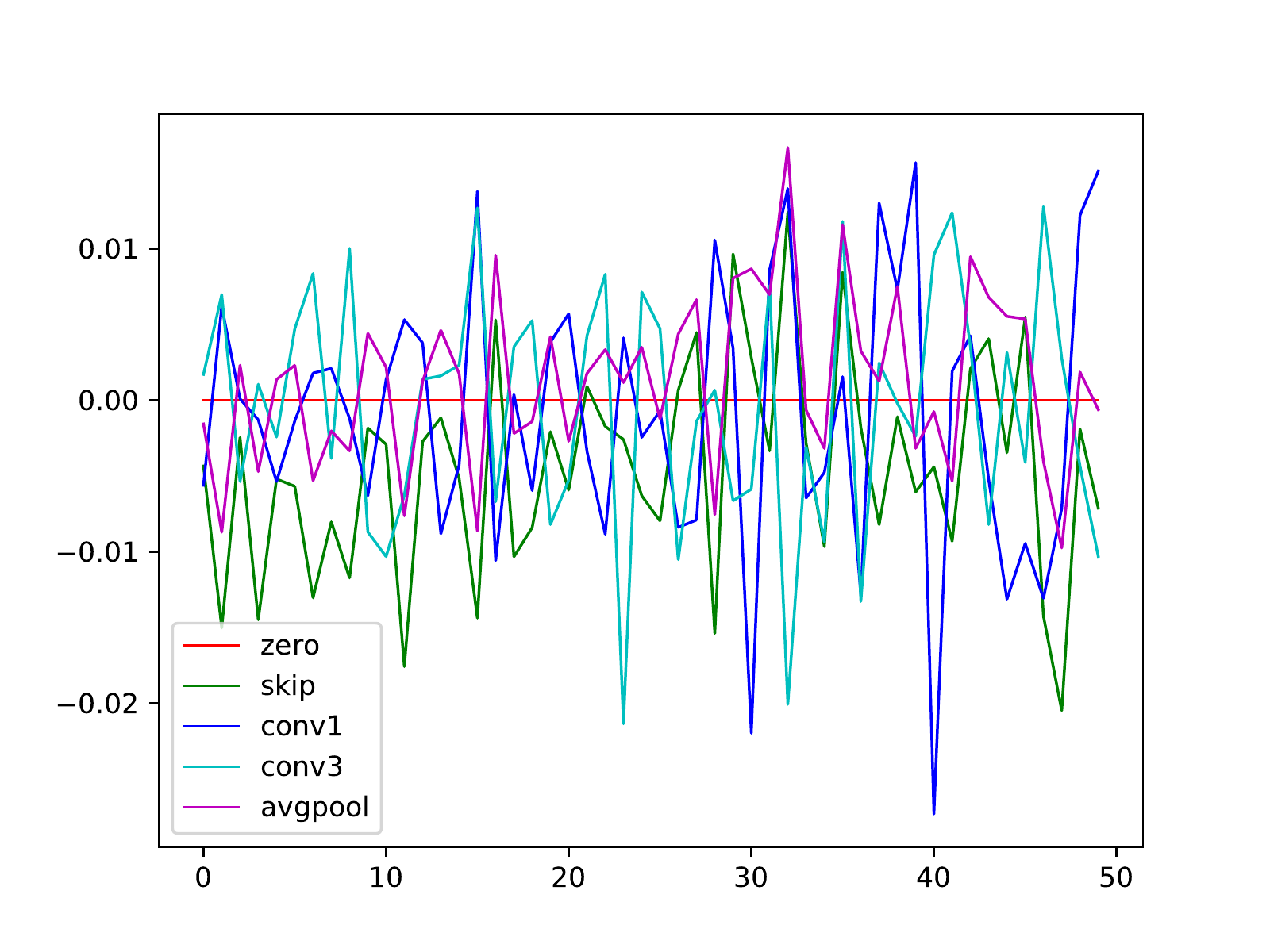}
 \caption{edge.2$\leftarrow$0}
\end{subfigure}
\hfill
 \begin{subfigure}[b]{0.3\linewidth}
 \centering
\includegraphics[width=\textwidth]{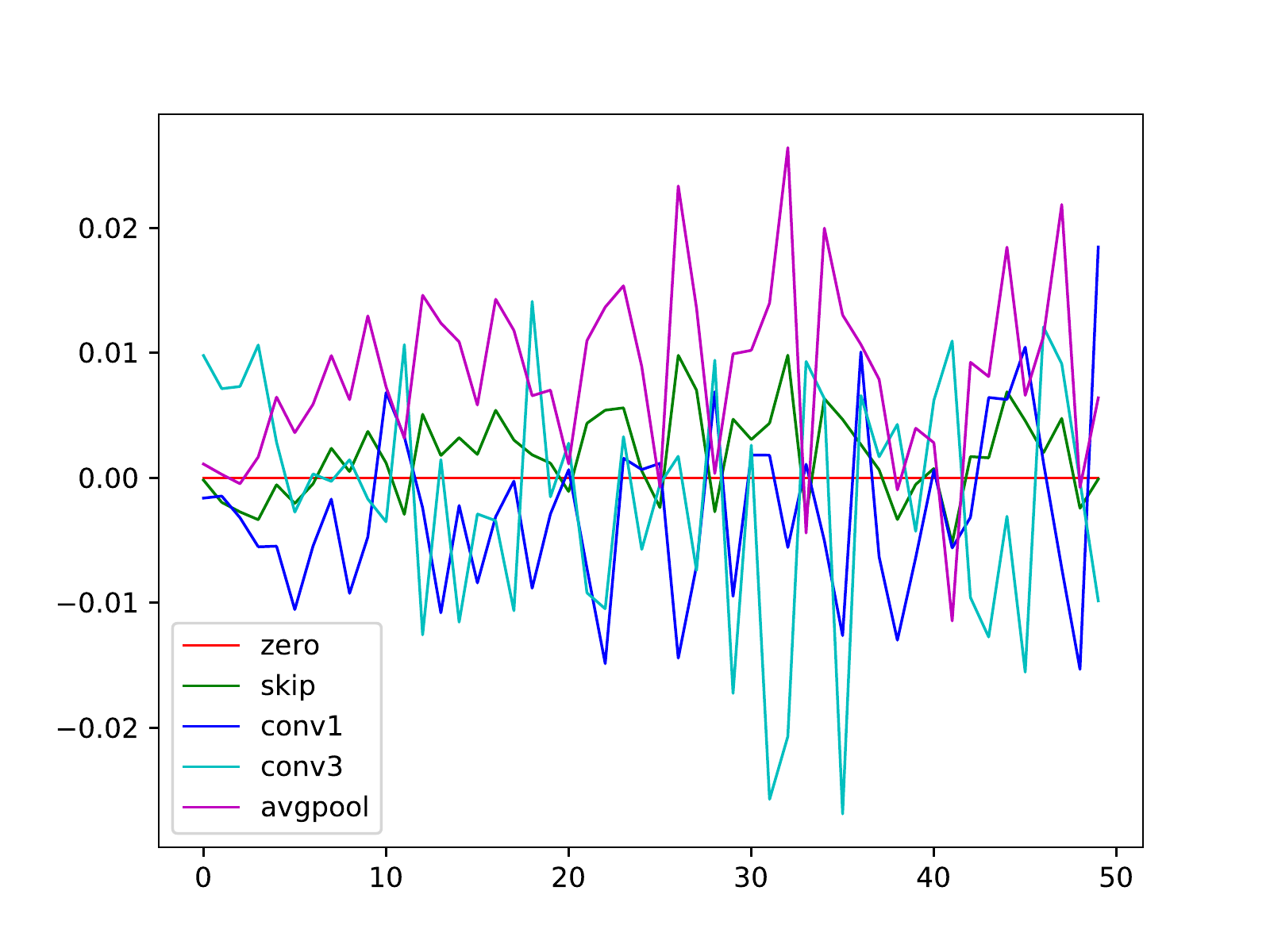}
 \caption{edge.2$\leftarrow$1}
\end{subfigure}
\hfill
\quad
 \begin{subfigure}[b]{0.3\linewidth}
 \centering
\includegraphics[width=\textwidth]{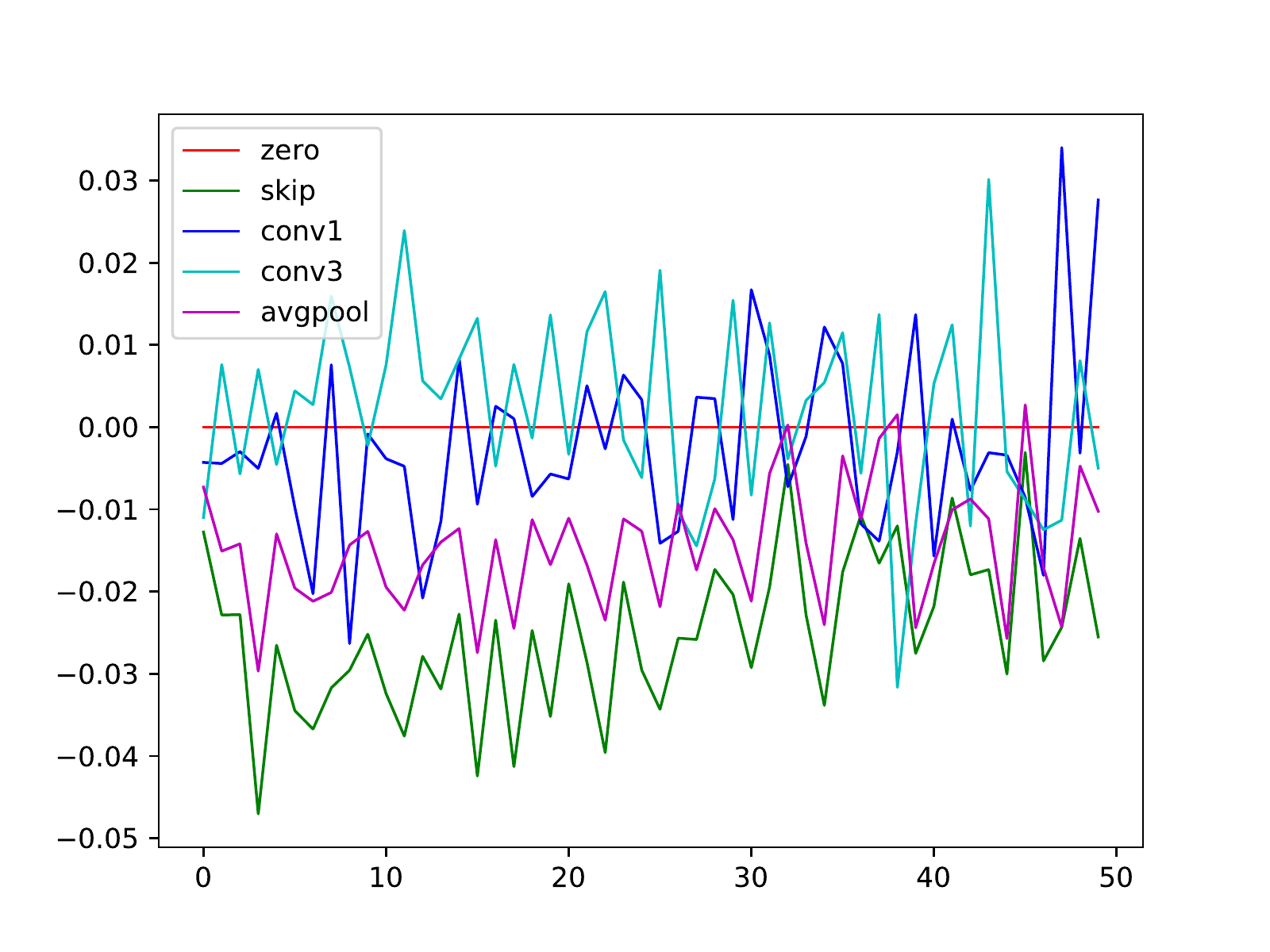}
 \caption{edge.3$\leftarrow$0}
\end{subfigure}
\hfill
 \begin{subfigure}[b]{0.3\linewidth}
 \centering
\includegraphics[width=\textwidth]{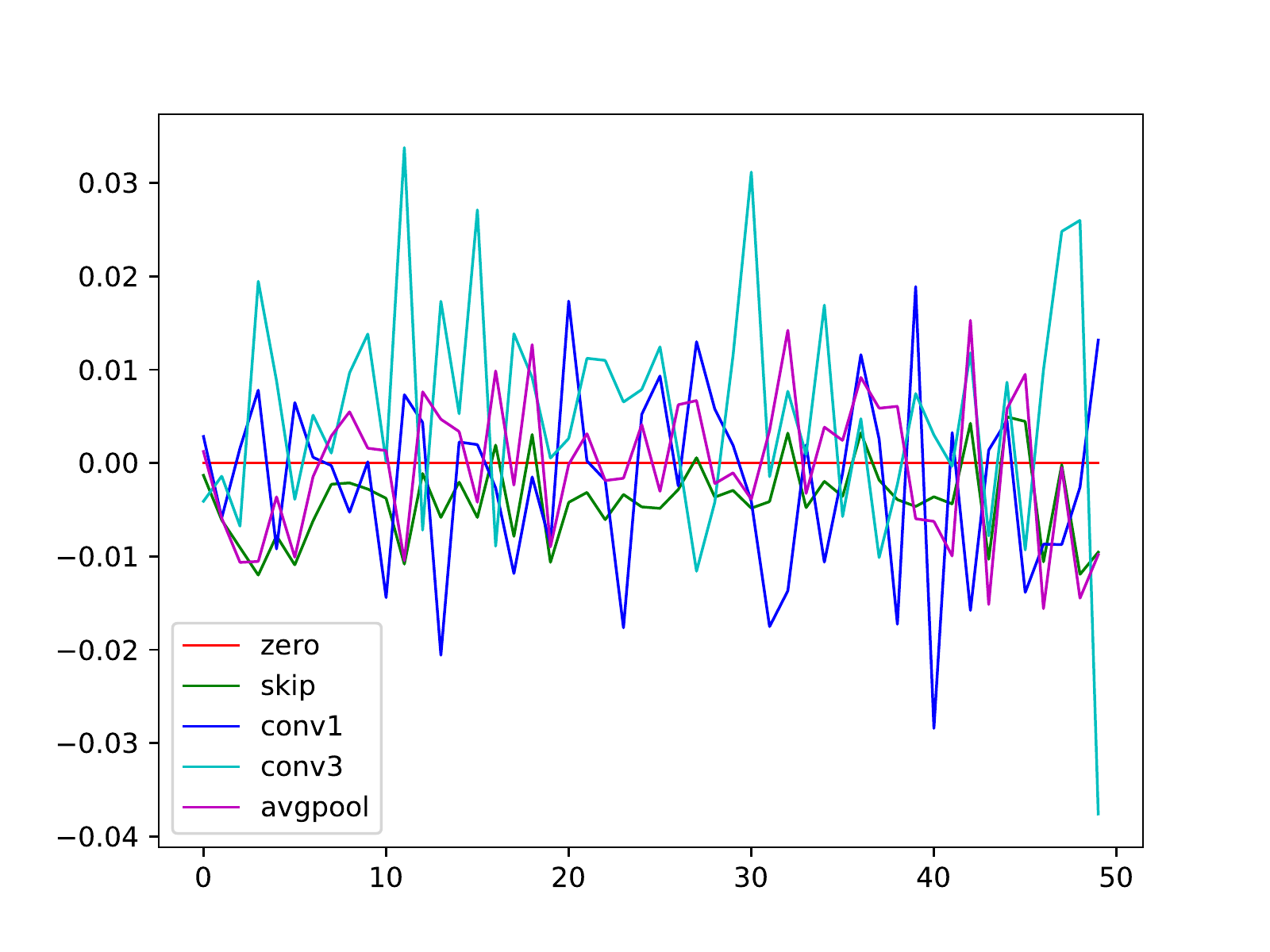}
 \caption{edge.3$\leftarrow$1}
\end{subfigure}
\hfill
 \begin{subfigure}[b]{0.3\linewidth}
 \centering
\includegraphics[width=\textwidth]{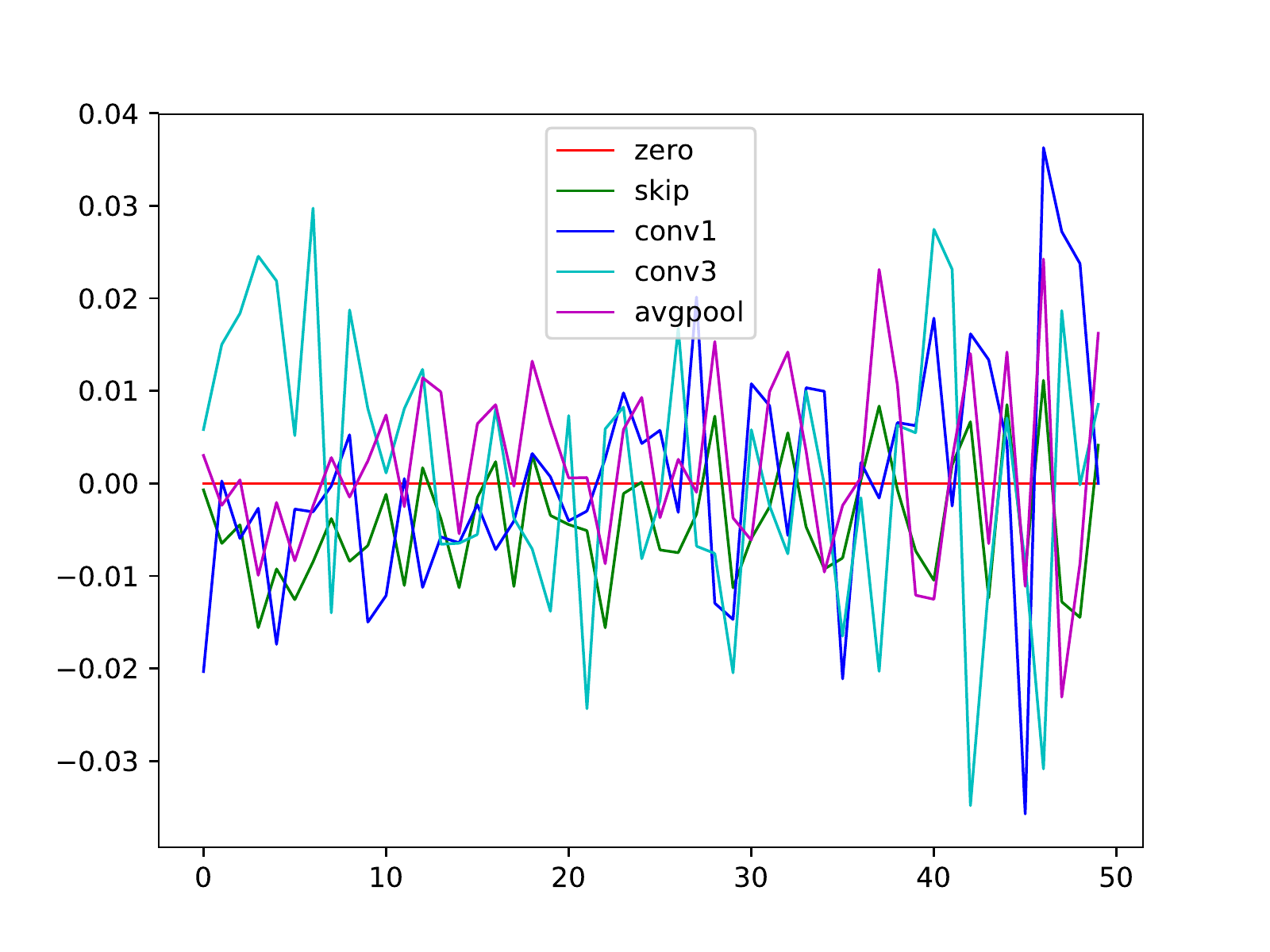}
 \caption{edge.3$\leftarrow$2}
\end{subfigure}
\hfill
\caption{Single-DARTS, the 8th cell.}
\end{figure*}

\begin{figure*}[h]
 \begin{subfigure}[b]{0.3\linewidth}
 \centering
\includegraphics[width=\textwidth]{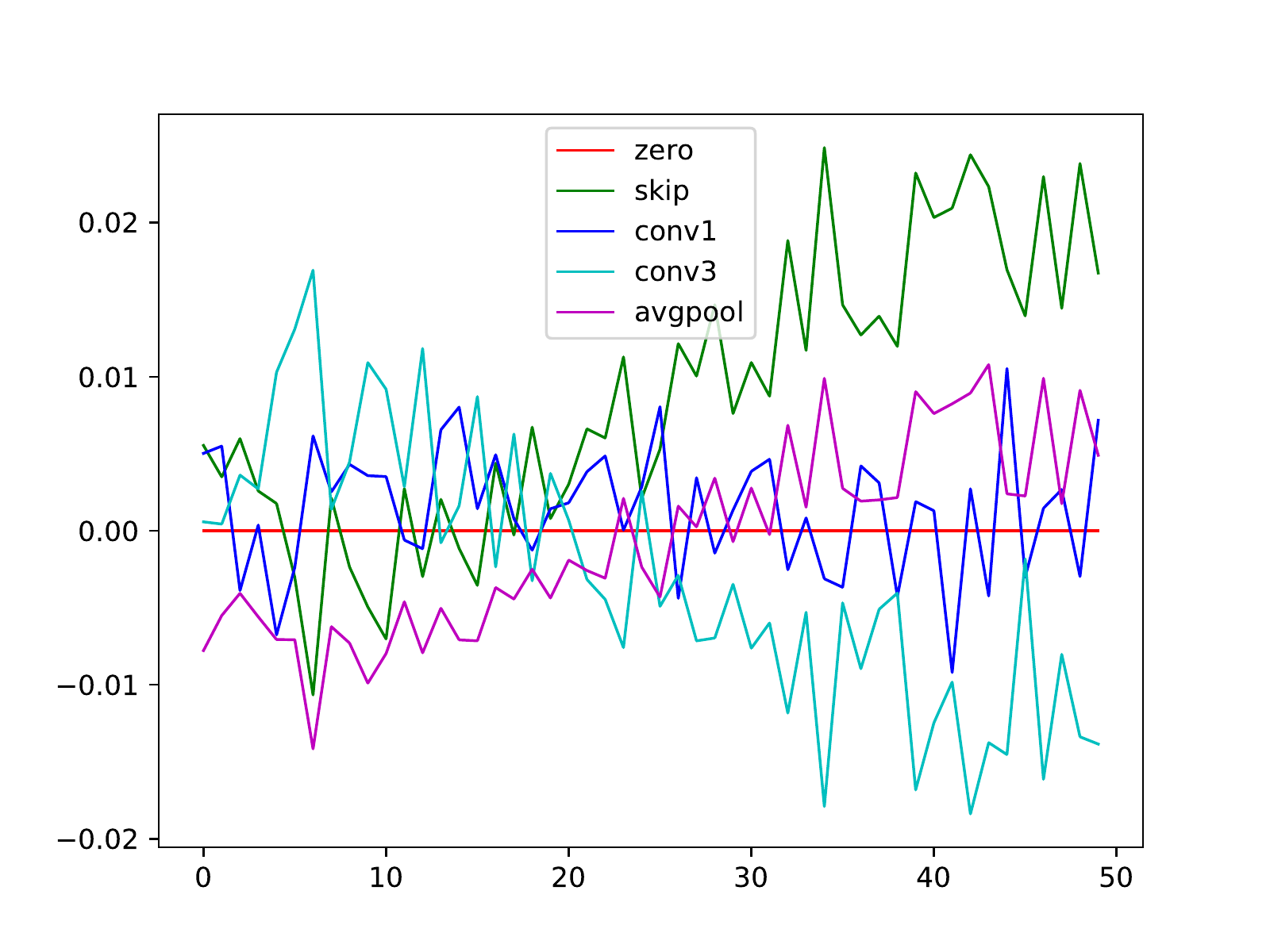}
 \caption{edge.1$\leftarrow$0}
\end{subfigure}
\hfill
 \begin{subfigure}[b]{0.3\linewidth}
 \centering
\includegraphics[width=\textwidth]{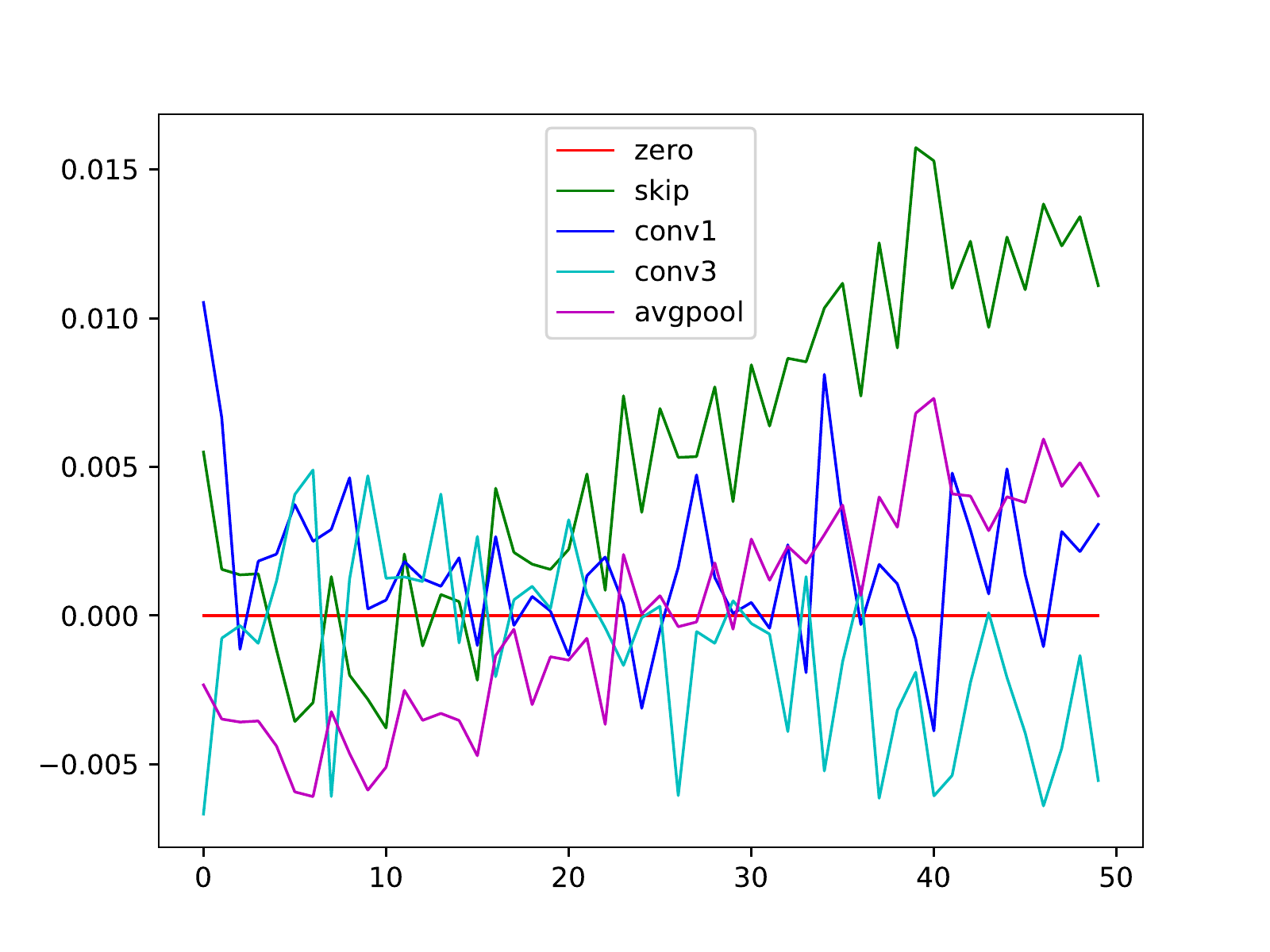}
 \caption{edge.2$\leftarrow$0}
\end{subfigure}
\hfill
 \begin{subfigure}[b]{0.3\linewidth}
 \centering
\includegraphics[width=\textwidth]{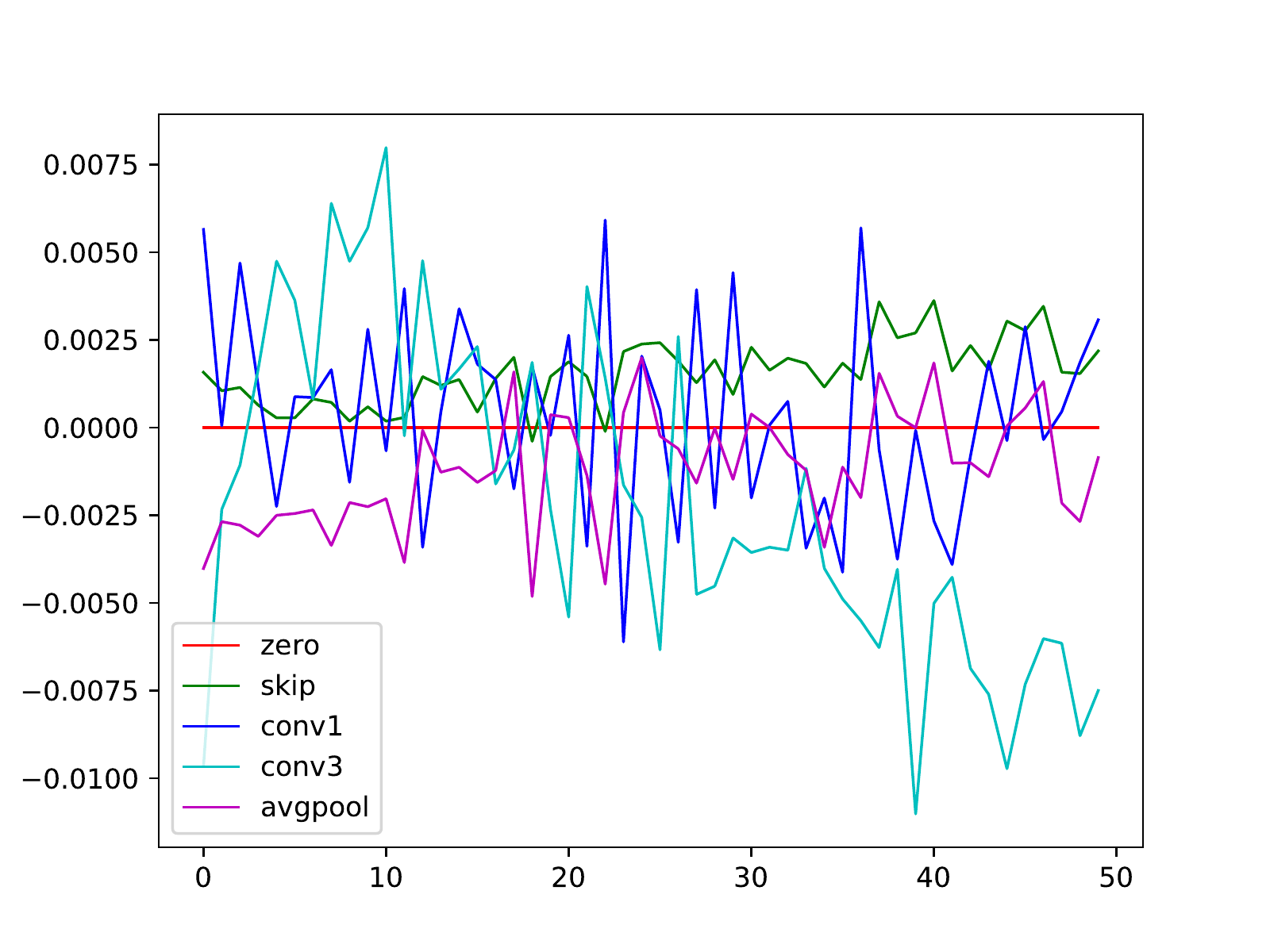}
 \caption{edge.2$\leftarrow$1}
\end{subfigure}
\hfill
\quad
\begin{subfigure}[b]{0.3\linewidth}
 \centering
\includegraphics[width=\textwidth]{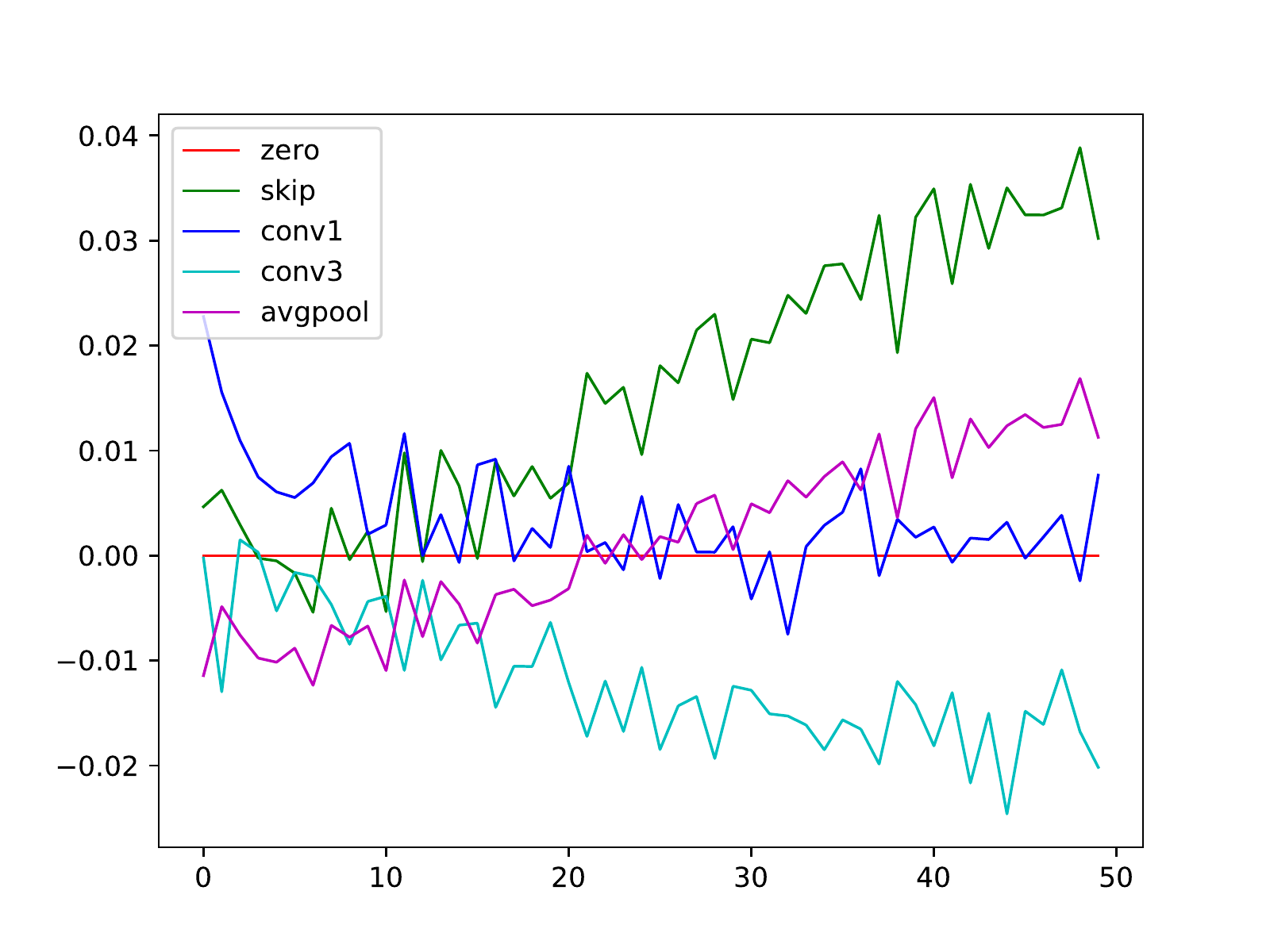}
 \caption{edge.3$\leftarrow$0}
\end{subfigure}
\hfill
 \begin{subfigure}[b]{0.3\linewidth}
 \centering
\includegraphics[width=\textwidth]{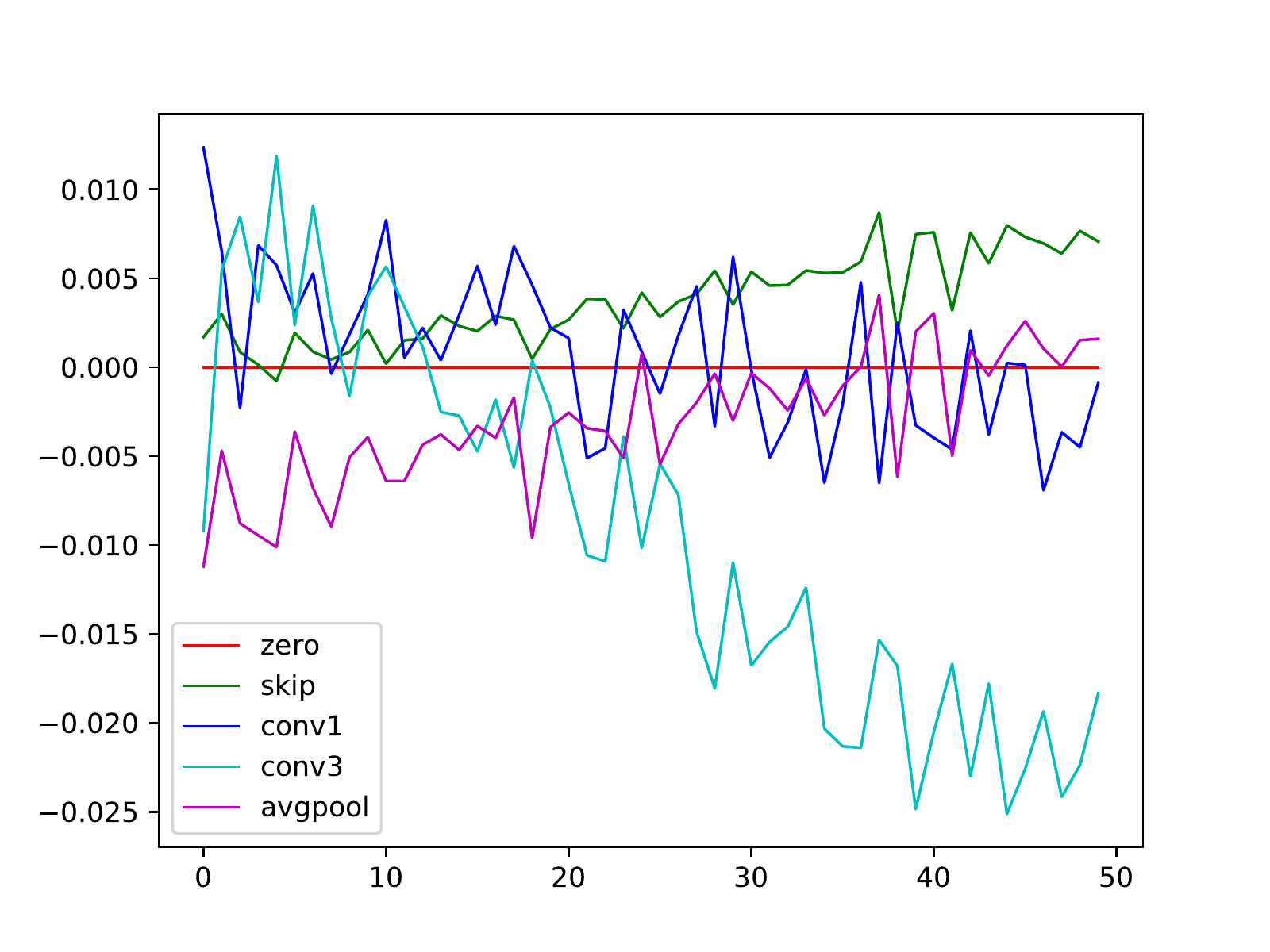}
 \caption{edge.3$\leftarrow$1}
\end{subfigure}
\hfill
 \begin{subfigure}[b]{0.3\linewidth}
 \centering
\includegraphics[width=\textwidth]{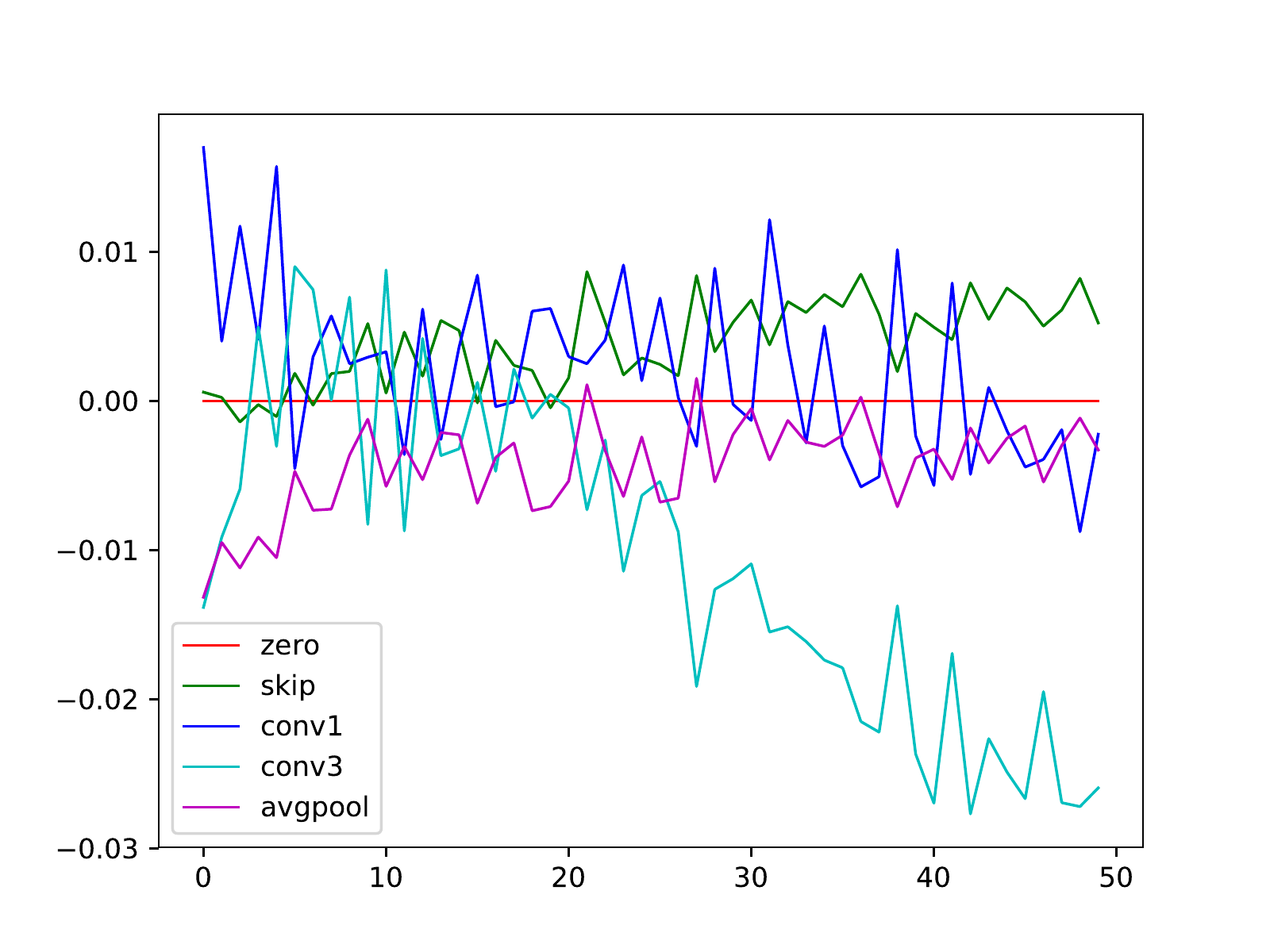}
 \caption{edge.3$\leftarrow$2}
\end{subfigure}
\hfill
\caption{Single-DARTS, the 16th cell.}
\end{figure*}

\section{Visualization of architectures}
\begin{figure*}[h]
 \begin{subfigure}[b]{0.49\linewidth}
  \parbox[][2.2cm][c]{\linewidth}{
 \centering
\includegraphics[width=\textwidth]{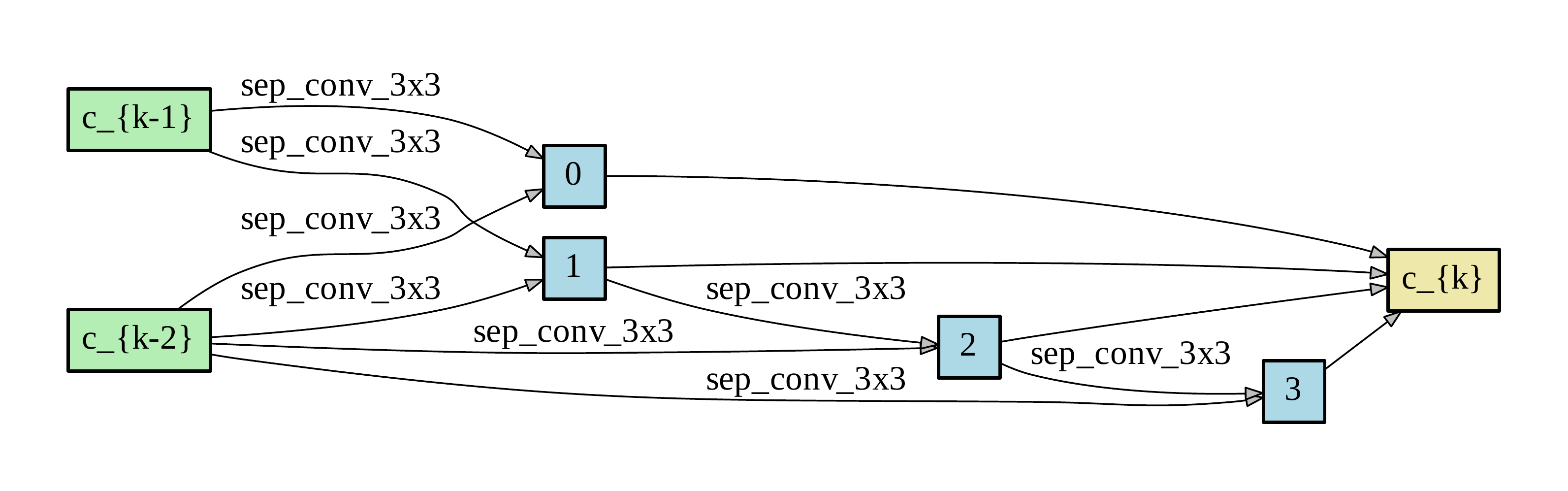}
 }
 \caption{normal}
\end{subfigure}
\hfill
 \begin{subfigure}[b]{0.49\linewidth}
 \centering
\includegraphics[width=\textwidth]{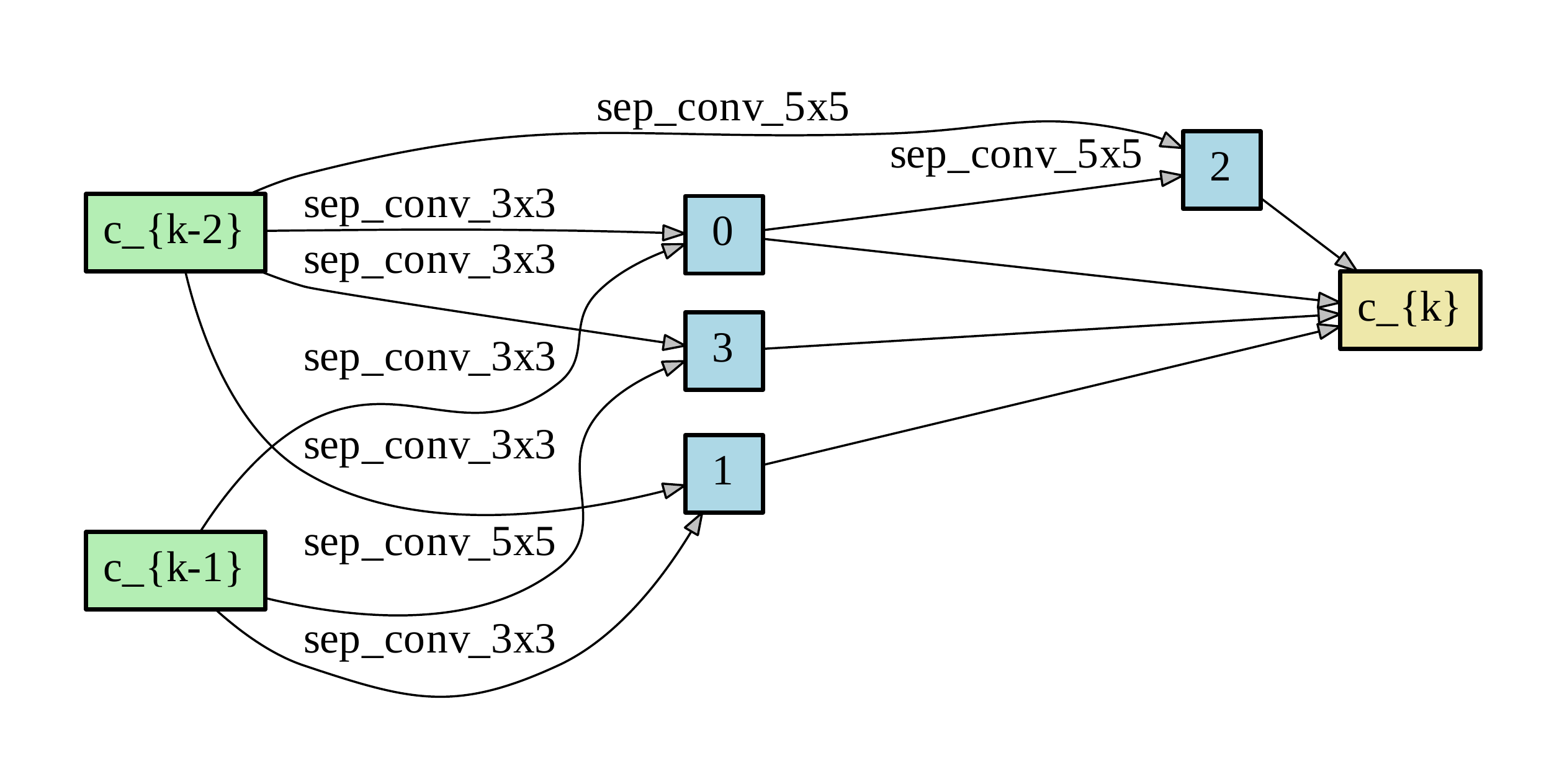}
 \caption{reduction}
\end{subfigure}
\hfill
\quad
\caption{activation=softmax, data=full, seed=0}
\end{figure*}

\begin{figure*}[h]
 \begin{subfigure}[b]{0.49\linewidth}
 \centering
\includegraphics[width=\textwidth]{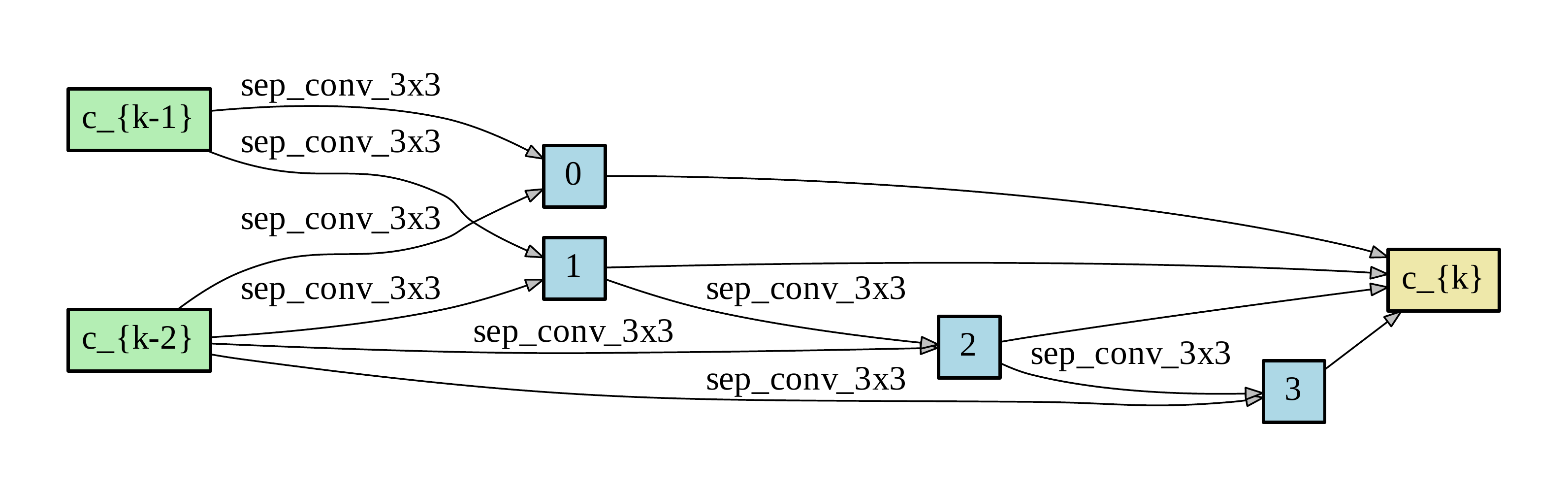}
 \caption{normal}
\end{subfigure}
\hfill
 \begin{subfigure}[b]{0.49\linewidth}
 \centering
\includegraphics[width=\textwidth]{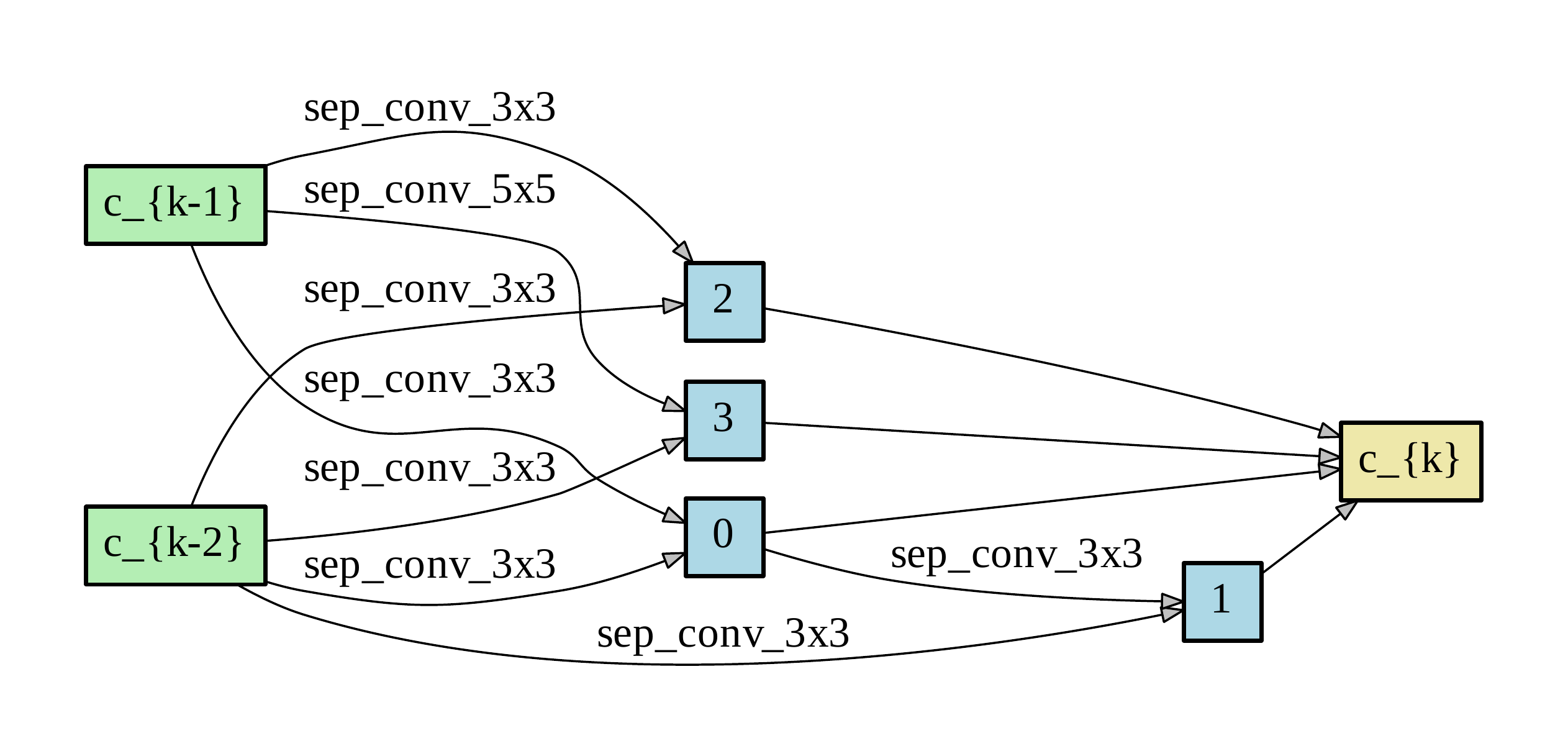}
 \caption{reduction}
\end{subfigure}
\hfill
\quad
\caption{activation=softmax, data=full, seed=1}
\end{figure*}

\begin{figure*}[h]
 \begin{subfigure}[b]{0.49\linewidth}
 \parbox[][2.0cm][c]{\linewidth}{
 \centering
\includegraphics[width=\textwidth]{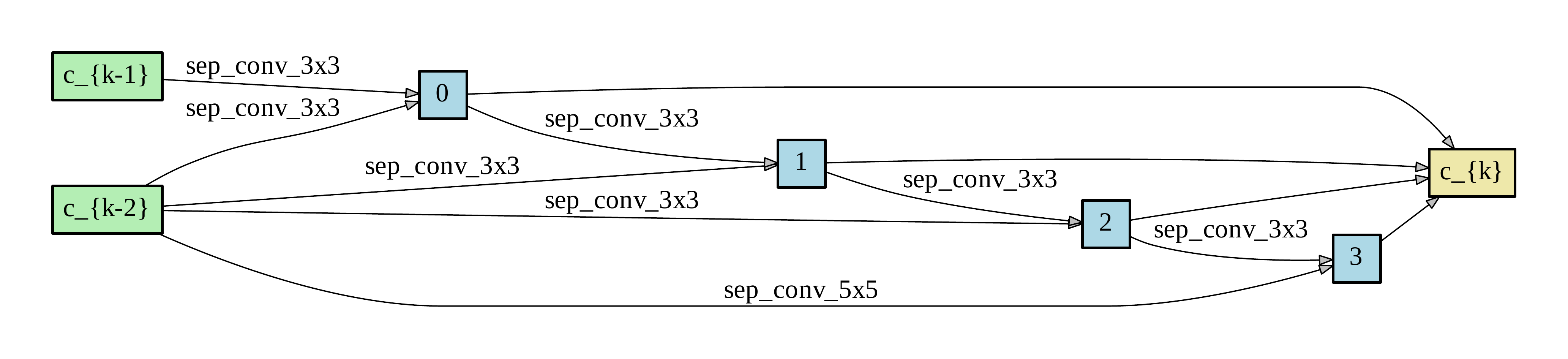}
 }
 \caption{normal}
\end{subfigure}
\hfill
 \begin{subfigure}[b]{0.49\linewidth}
 \centering
\includegraphics[width=\textwidth]{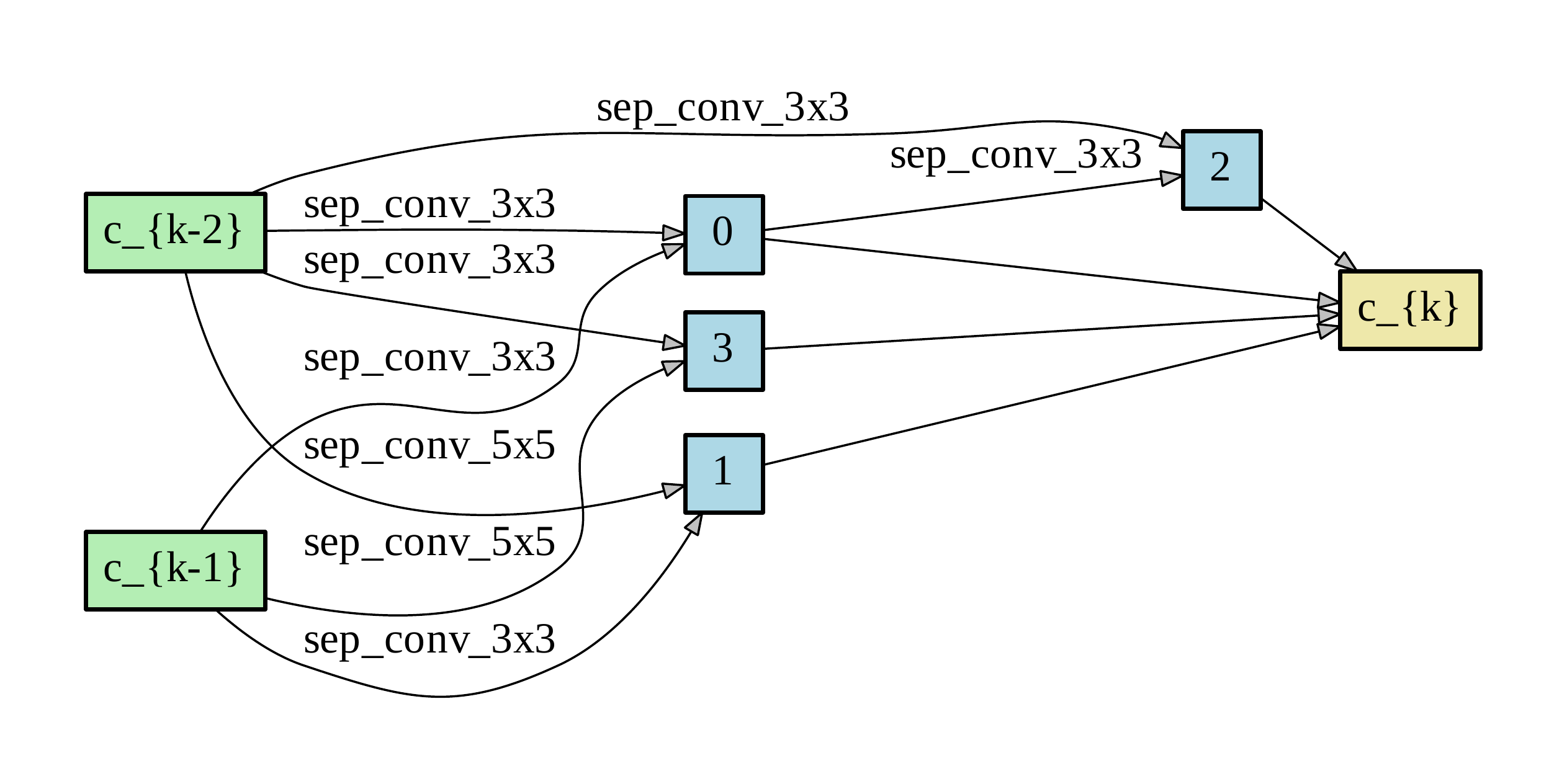}
 \caption{reduction}
\end{subfigure}
\hfill
\quad
\caption{activation=softmax, data=full, seed=2}
\end{figure*}

\begin{figure*}[h]
\centering
 \begin{subfigure}[b]{0.49\linewidth}
 \parbox[][3.5cm][c]{\linewidth}{
 \centering
\includegraphics[width=\textwidth]{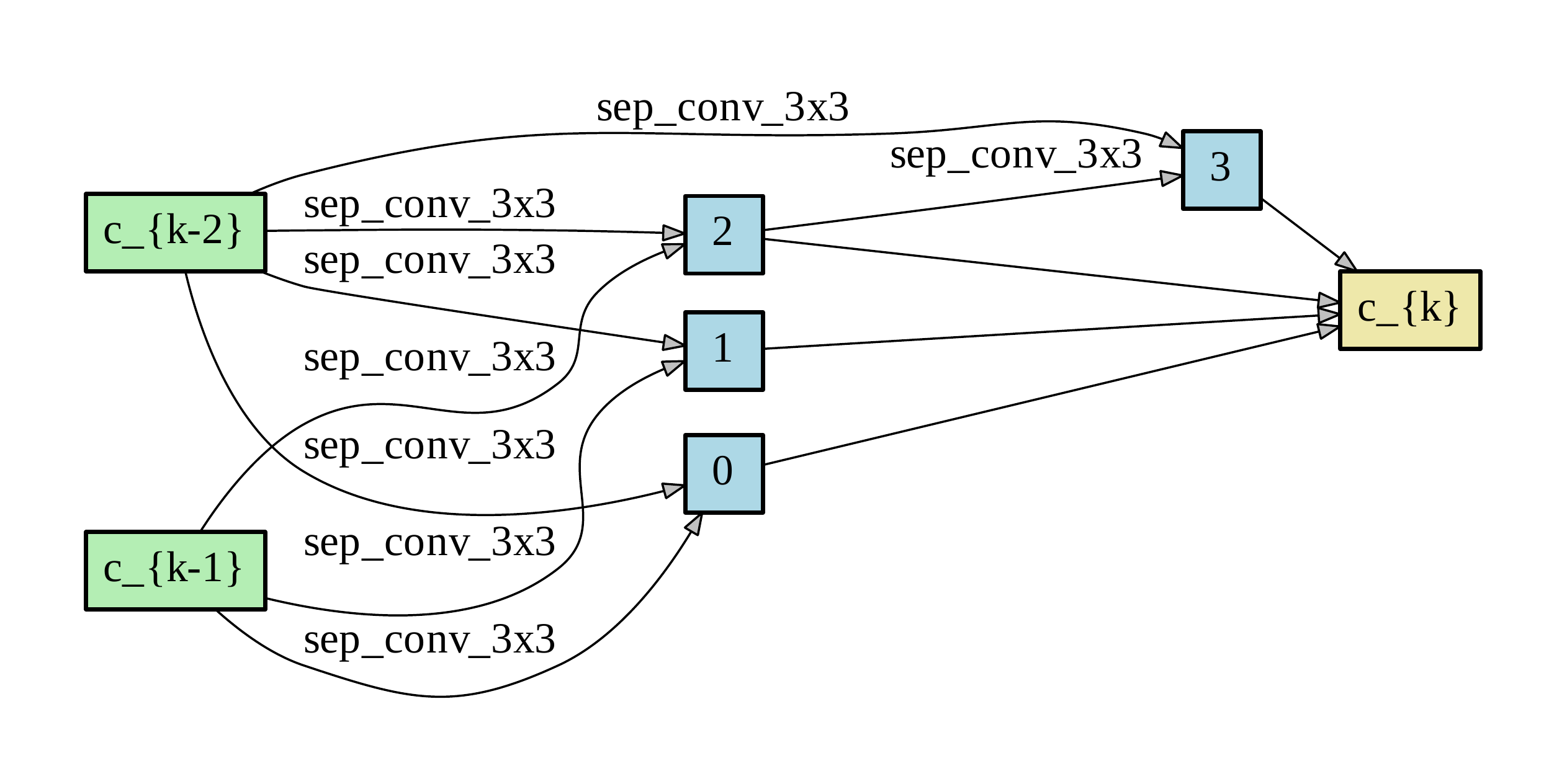}
}
 \caption{normal cell}
\end{subfigure}
\hfill
 \begin{subfigure}[b]{0.3\linewidth}
 \centering
\includegraphics[width=\textwidth]{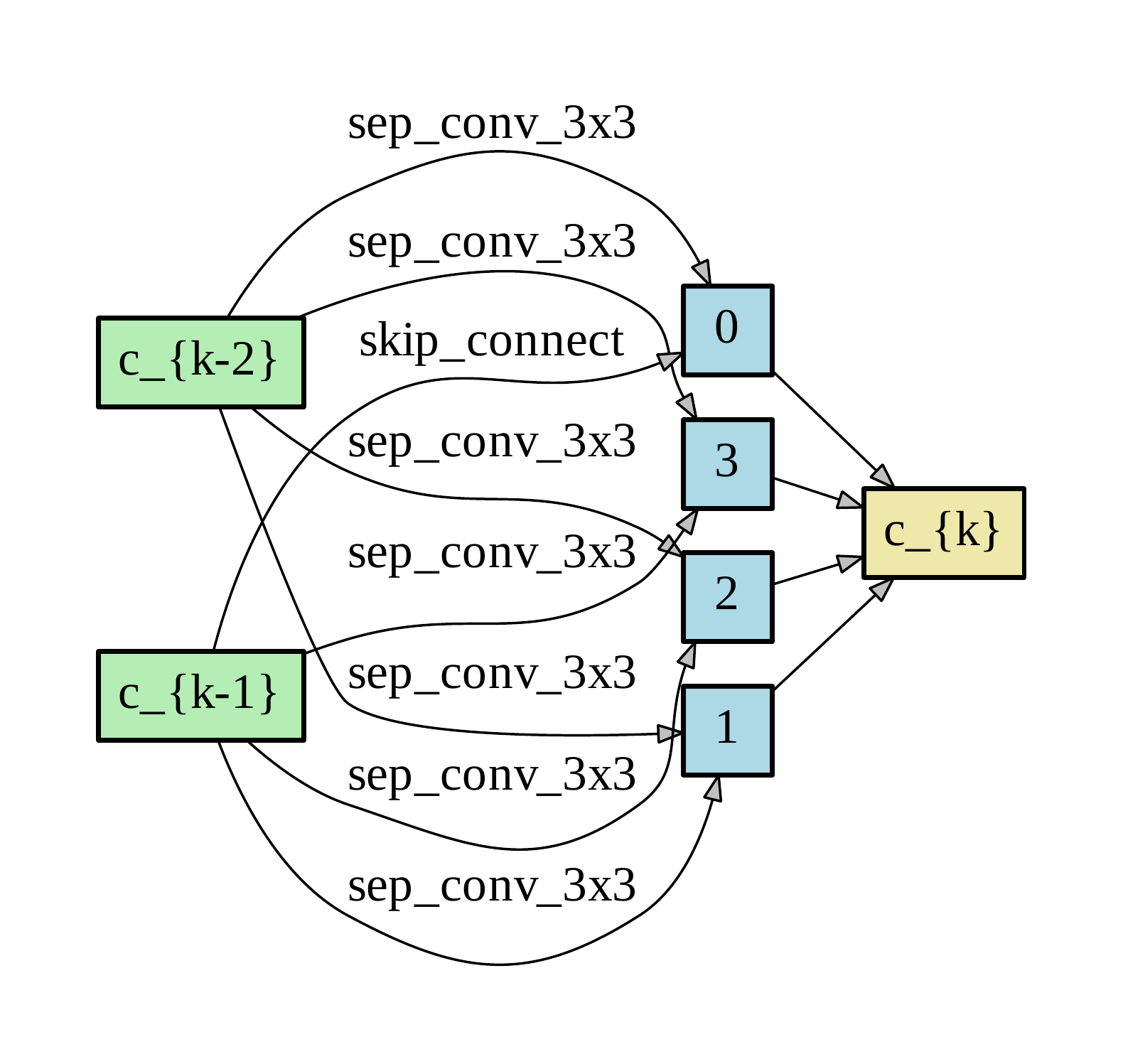}
 \caption{reduction cell}
\end{subfigure}
\hfill
\quad
\caption{activation=sigmoid*, data=full, seed=0}
\end{figure*}

\begin{figure*}[h]
 \begin{subfigure}[b]{0.48\linewidth}
 \parbox[][3.5cm][c]{\linewidth}{
 \centering
\includegraphics[width=\textwidth]{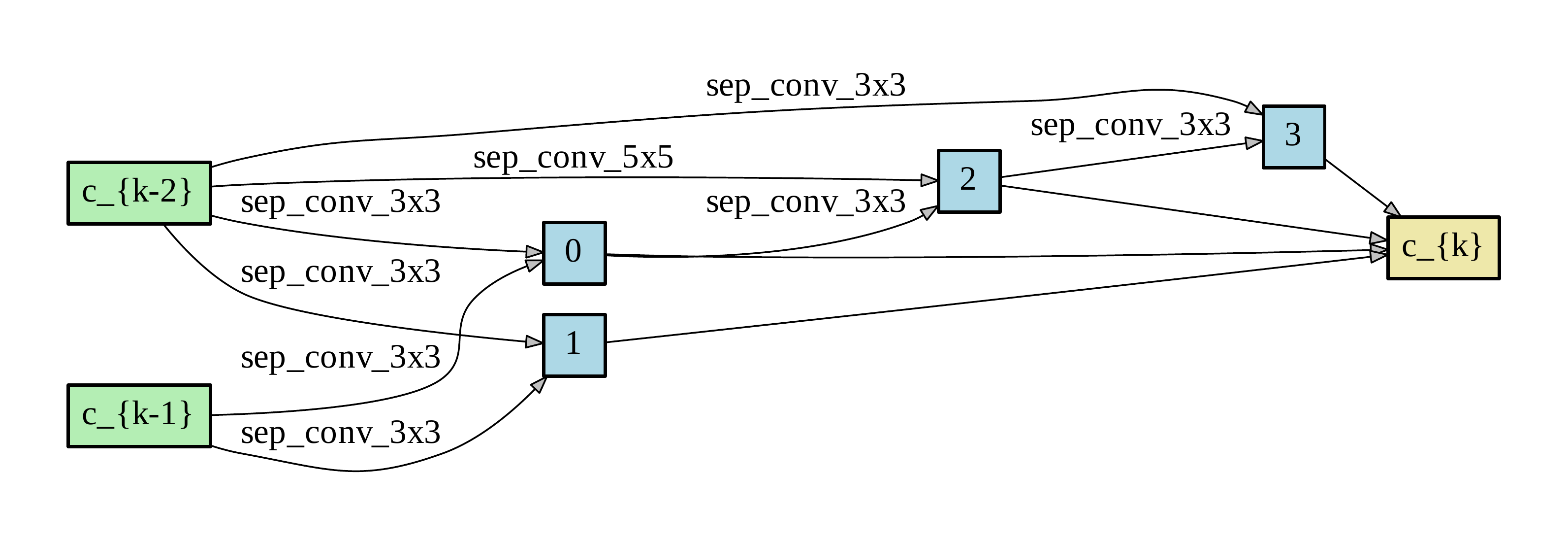}
 }
 \caption{normal}
\end{subfigure}
\hfill
 \begin{subfigure}[b]{0.3\linewidth}
 \centering
\includegraphics[width=\textwidth]{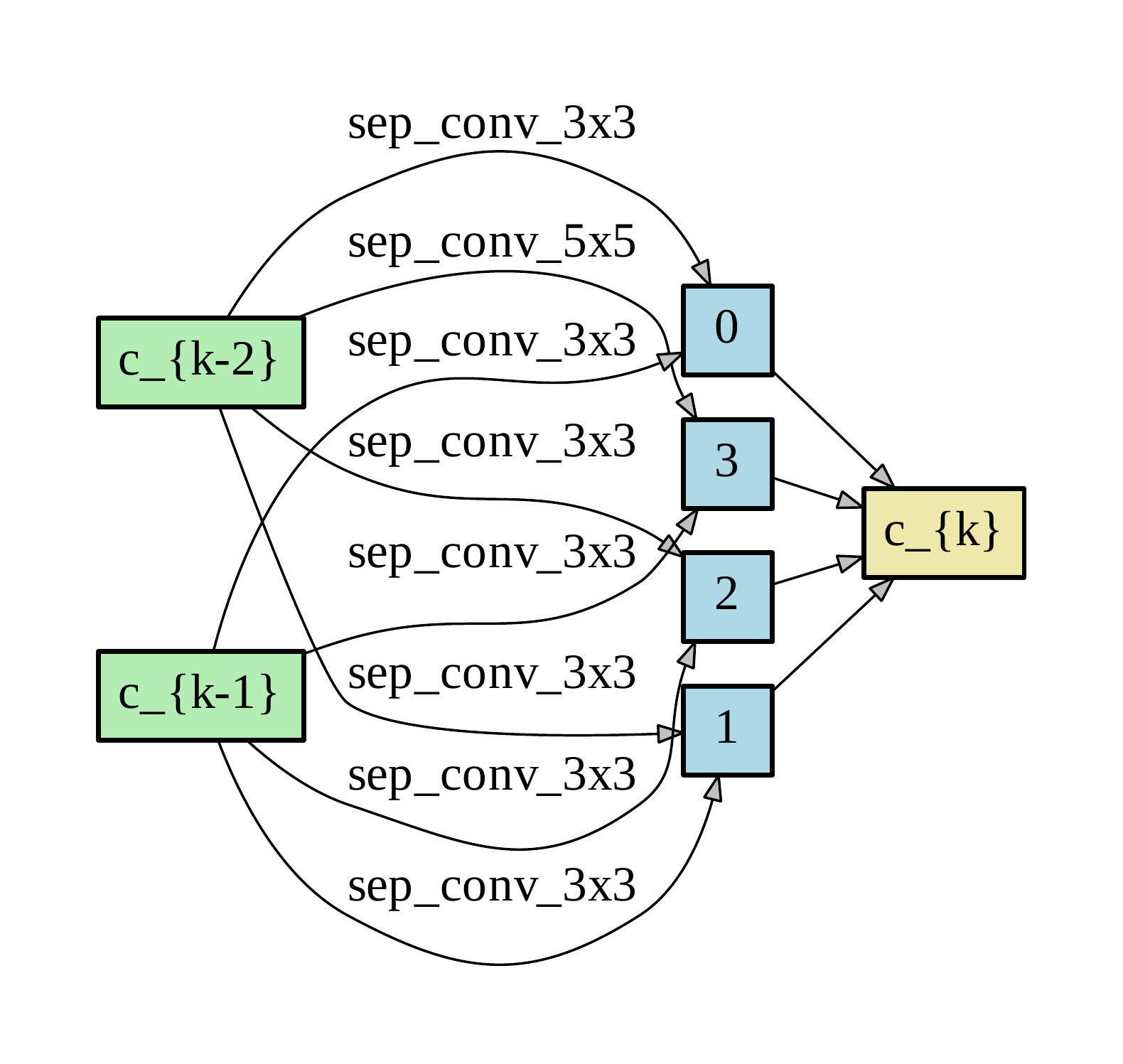}
 \caption{reduction}
\end{subfigure}
\hfill
\quad
\caption{activation=sigmoid*, data=full, seed=1}
\end{figure*}

\begin{figure*}[t]
 \begin{subfigure}[b]{0.48\linewidth}
 \parbox[][3.5cm][c]{\linewidth}{
 \centering
\includegraphics[width=\textwidth]{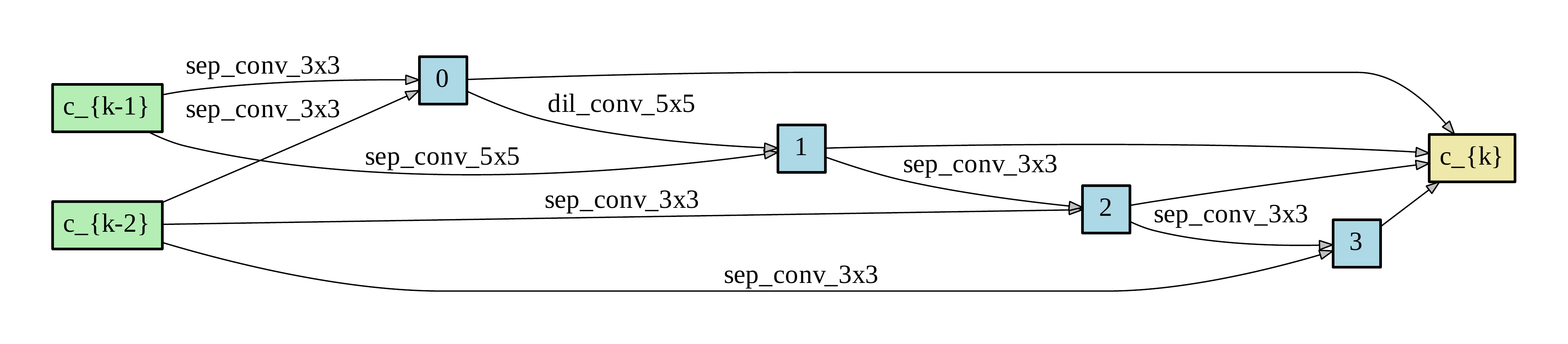}
 }
 \caption{normal}
\end{subfigure}
\hfill
 \begin{subfigure}[b]{0.3\linewidth}
 \centering
\includegraphics[width=\textwidth]{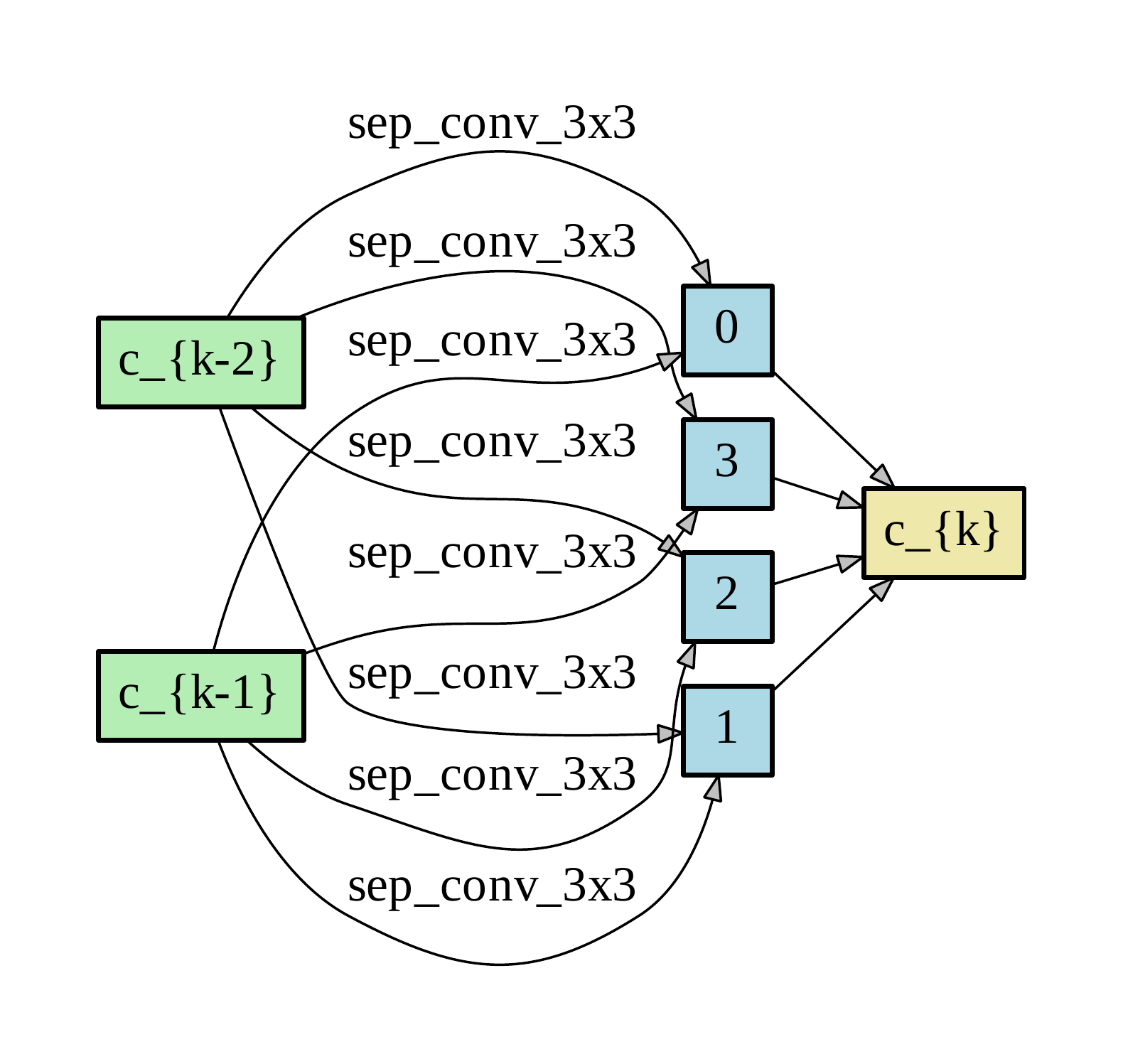}
 \caption{reduction}
\end{subfigure}
\hfill
\quad
\caption{activation=sigmoid*, data=full, seed=2}
\end{figure*}
\end{document}